\documentclass{scrartcl}

\usepackage{hyperref}
\usepackage{amsmath,amssymb}
\usepackage{amsthm}
\usepackage[utf8]{inputenc}
\usepackage{cleveref}
\usepackage{bbm}
\usepackage{color}
\usepackage[ruled]{algorithm2e}
\usepackage{graphicx}
\usepackage{caption}
\usepackage{subcaption}
\usepackage[marginal]{footmisc}

\newtheorem{theorem}{Theorem}
\newtheorem{lemma}{Lemma}
\newtheorem{remark}{Remark}

\newtheorem{proposition}{Proposition}

\newcommand{\trace}{\text{Tr}\,}
\newcommand{\R}{\mathbb{R}}

\newcommand{\Id}{\text{Id}}
\newcommand{\E}{\mathbb{E}}

\newcommand{\Aop}[1]{\mathcal{A} \left( #1 \right)}
\newcommand{\Aopaux}[1]{\tilde{\mathcal{A}} \left( #1 \right)}
\newcommand{\innerproduct}[1]{ \langle  #1 \rangle }
\newcommand{\froinnerproduct}[1]{ \langle  #1 \rangle_F }
\newcommand{\fronorm}[1]{ \Vert  #1 \Vert_F }

\newcommand{\ut}{u_t}
\newcommand{\vt}{v_t}
\newcommand{\utplushalf}{u_{t+1/2}}
\newcommand{\vtplushalf}{v_{t+1/2}}
\newcommand{\utplus}{u_{t+1}}
\newcommand{\vtplus}{v_{t+1}}
\newcommand{\utaux}{\tilde{u_t}}
\newcommand{\vtaux}{\tilde{v_t}}
\newcommand{\utplushalfaux}{\tilde{u}_{t+1/2}}
\newcommand{\vtplushalfaux}{\tilde{v}_{t+1/2}}
\newcommand{\utplusaux}{\tilde{u}_{t+1}}
\newcommand{\vtplusaux}{\tilde{v}_{t+1}}
\newcommand{\mutplushalf}{\mu_{t+1/2}}

\newcommand{\vthat}{v_{ \hat{t} }}
\newcommand{\utplushalfhat}{u_{\hat{t}+1/2}}
\newcommand{\utplushat}{u_{\hat{t}+1}}

\newcommand{\eeta}{\eta}

\newcommand{\ct}{c_{2t}}

\newcommand{\bracing}[2]{\underset{{#1}}{\underbrace{#2}}  }
\newcommand{\twonorm}[1]{\Vert  #1 \Vert}
\newcommand{\overleq}[1]{\overset{#1}{\le}}
\newcommand{\overeq}[1]{\overset{#1}{=}}

\newcommand{\ustar}{u_\star}
\newcommand{\vstar}{v_\star}
\newcommand{\Xstar}{X_\star}

\begin{document}

\title{Randomly Initialized Alternating Least Squares: Fast Convergence for Matrix Sensing}

\author{Kiryung Lee and Dominik St\"oger\thanks{Corresponding author (email: Dominik.Stoeger@ku.de).} \thanks{KL is with the Department of Electrical and Computer Engineering at the Ohio State University. DS is with the Department of Mathematics and the Mathematical Institute for Machine Learning and Data Science (MIDS) at KU Eichst\"att-Ingolstadt. KL was supported in part by NSF CAREER Award CCF 19-43201.}}

\maketitle

\begin{abstract}
We consider the problem of reconstructing rank-one matrices from random linear measurements, a task that appears in a variety of problems in signal processing, statistics, and machine learning.
In this paper, we focus on the Alternating Least Squares (ALS) method. 
While this algorithm has been studied in a number of previous works, most of them only show convergence from an initialization close to the true solution and thus require a carefully designed initialization scheme. 
However, random initialization has often been preferred by practitioners as it is model-agnostic.
In this paper, we show that ALS with random initialization converges to the true solution with $\varepsilon$-accuracy in $O(\log n + \log (1/\varepsilon)) $ iterations using only a near-optimal amount of samples, where we assume the measurement matrices to be i.i.d. Gaussian and where by $n$ we denote the ambient dimension.
Key to our proof is the observation that the trajectory of the ALS iterates only depends very mildly on certain entries of the random measurement matrices.
Numerical experiments corroborate our theoretical predictions.
\end{abstract}

\section{Introduction}

\subsection{Alternating minimization and low-rank matrix recovery problems}

Suppose we are given observations of the form
\begin{equation}\label{equ:linearsystem}
y_i = \froinnerproduct{A_{i}, \Xstar} := \trace \left( A_i^\top \Xstar \right), \quad  1\le i \le m 
\end{equation}
with known measurement matrices $\{A_i\}_{i=1}^m \subset \mathbb{R}^{n_1\times n_2}$ and our goal is to estimate an unknown low-rank matrix $\Xstar \in \mathbb{R}^{n_1\times n_2}$, i.e., $\text{rank} \left( \Xstar \right) =r \ll \min \left\{ n_1; n_2 \right\} $.
This problem is ubiquitous in many applications such as matrix completion, blind deconvolution, and phase retrieval. We refer to \cite{davenport2016overview} for a comprehensive overview.
Different approaches to this problem have been established in the literature ranging from convex methods such as nuclear norm minimization to non-convex methods based on matrix factorization such as gradient descent and alternating minimization.

The method we want to consider in this paper is the Alternating Least Squares (ALS) method. That is, we consider the non-convex loss function
\begin{equation}\label{equ:minimization_problem}
    f(U,V):= \frac{1}{2m} \sum_{i=1}^m \left( y_i - \innerproduct{A_i, UV^\top}_F \right)^2,
\end{equation}
where $ U \in \mathbb{R}^{n_1 \times r} $ and  $ V \in \mathbb{R}^{n_2 \times r}$, and we alternate between updating $U$ and $V$, i.e.,
\begin{equation}\label{intro_altmin}
\begin{split}
    U_{t+1} &= \mathop{\mathrm{argmin}}_{U \in \mathbb{R}^{n_1 \times r}} \ f\left(U,V_t\right),\\
    V_{t+1} &= \mathop{\mathrm{argmin}}_{V \in \mathbb{R}^{n_2 \times r}} \ f\left(U_{t+1},V\right).
\end{split}
\end{equation}
In each step, one needs to solve a linear least-squares problem, which can be achieved efficiently via the conjugate gradient method, see, e.g., \cite{trefethen}.

For low-rank matrix recovery, the ALS method has first been proposed in \cite{haldar2009rank}. Later, it was shown that given an initialization close to the ground truth, the ALS method converges linearly to the ground truth solution using a near-optimal amount of samples for the Matrix Sensing and Matrix Completion problem \cite{jain2013low}. Moreover, it was shown that such an initialization can be constructed via a so-called spectral initialization.

However, while the ALS method is popular among practitioners, they often use a random initialization for the ALS method instead of a spectral initialization, see, e.g., \cite{hastie_matrixcompletion}. One advantage is that random initialization is model-agnostic in contrast to spectral methods. However, despite its importance in practice, the convergence of ALS from random initialization remains poorly understood. Existing theory either shows convergence starting from spectral initialization \cite{jain2013low,sun_matrixcompletion} or with resampling, i.e., that for each iteration fresh samples are used, see, e.g., \cite{hardt_matrixcompletion1,hardt_matrixcompletion2}.

\subsection{Our contribution}

In this paper, we show that, if the $A_i$'s are i.i.d. Gaussian measurement matrices and if $\Xstar \in \mathbb{R}^{n_1 \times n_2}$ is a rank-one matrix, then ALS with random initialization converges to the ground truth in $O (\frac{\log n_2 + \log (1/\varepsilon)}{\log \log n_2}) $ iterations to $\varepsilon$-accuracy using only a near-optimal amount of measurements.
Note that the scenario that the ground truth matrix $\Xstar$ is a rank-one matrix indeed appears in many applications such as Blind Deconvolution and Phase Retrieval.
To the best of our knowledge, this is the first result in the literature that shows that the ALS iterates for low-rank matrix recovery converge to the true solution starting from random initialization (without resampling at each iteration).

In our analysis, we establish that the convergence of ALS can be separated into two distinct phases.
In the first phase, we show that, starting from an initialization that is near-orthogonal to the ground truth, the angle between the true solution and the ALS-iterates is decreasing. More precisely, we show that the cosine of this angle is growing at a geometric rate.
As soon as our signal is aligned closely enough with the ground truth signal, we enter the second phase. In this phase, our iterates converge linearly to the ground truth.
All of this is corroborated by numerical experiments, see Figure \ref{fig:iter}, which indeed confirm that there is a sharp phase transition between those two phases. 

We note that linear convergence in the second phase can essentially be deduced from the aforementioned previous work \cite{jain2013low}.
Hence, the key difficulty in proving convergence of ALS from random initialization lies in rigorously establishing the fact that the alignment of the iterates with the true signal is increasing in the first phase.
One major obstacle is that there exist many saddle points in minimization of the quadratic loss in \eqref{equ:minimization_problem}, see \cite{bhojanapalli2016global}.
In particular, it is not clear whether the iterates of ALS can avoid such saddle points.

Our analysis establishes that, with high probability, ALS does not get stuck in saddle points in the first phase. For that, we will show that, in the first phase, the ALS iterates are nearly independent of certain entries in the measurement matrices $A_i$. This allows us to make much stronger statements than what would be possible by, for example, solely relying on the loss landscape of $f$.
To establish the ``near-independence'' of the iterates to certain entries of the measurement matrices, we will construct an appropriate (virtual) auxiliary sequence. 
Our construction is inspired by the use of auxiliary sequences in \cite{chen2019gradient} to show convergence of gradient descent from a random initialization in the phase retrieval problem.
However, since the ALS method behaves quite differently than gradient descent, the resulting proofs are also quite different.

We believe that the insights and proof techniques developed in this paper will also pave the way for understanding the convergence of ALS starting from random initialization in scenarios where the rank of the underlying signal is larger than one or where more structured measurement matrices are used, for example, in the problem of Blind Deconvolution.
\begin{figure}
    \centering
    \begin{subfigure}[b]{0.49\textwidth}
        \centering
        \includegraphics[width=\textwidth]{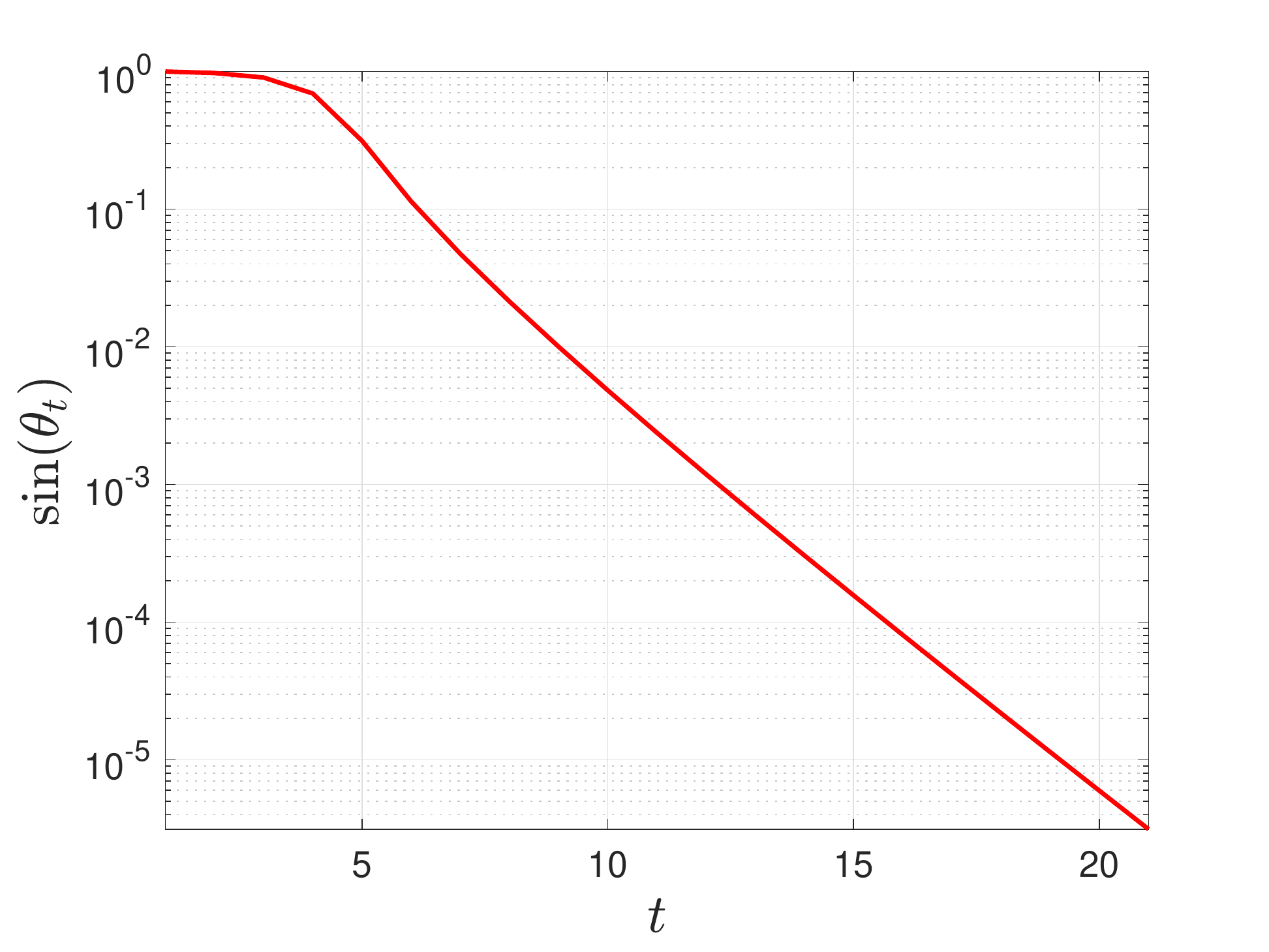}
        \caption{}
    \end{subfigure}
    \hfill
    \begin{subfigure}[b]{0.49\textwidth}
        \centering
        \includegraphics[width=\textwidth]{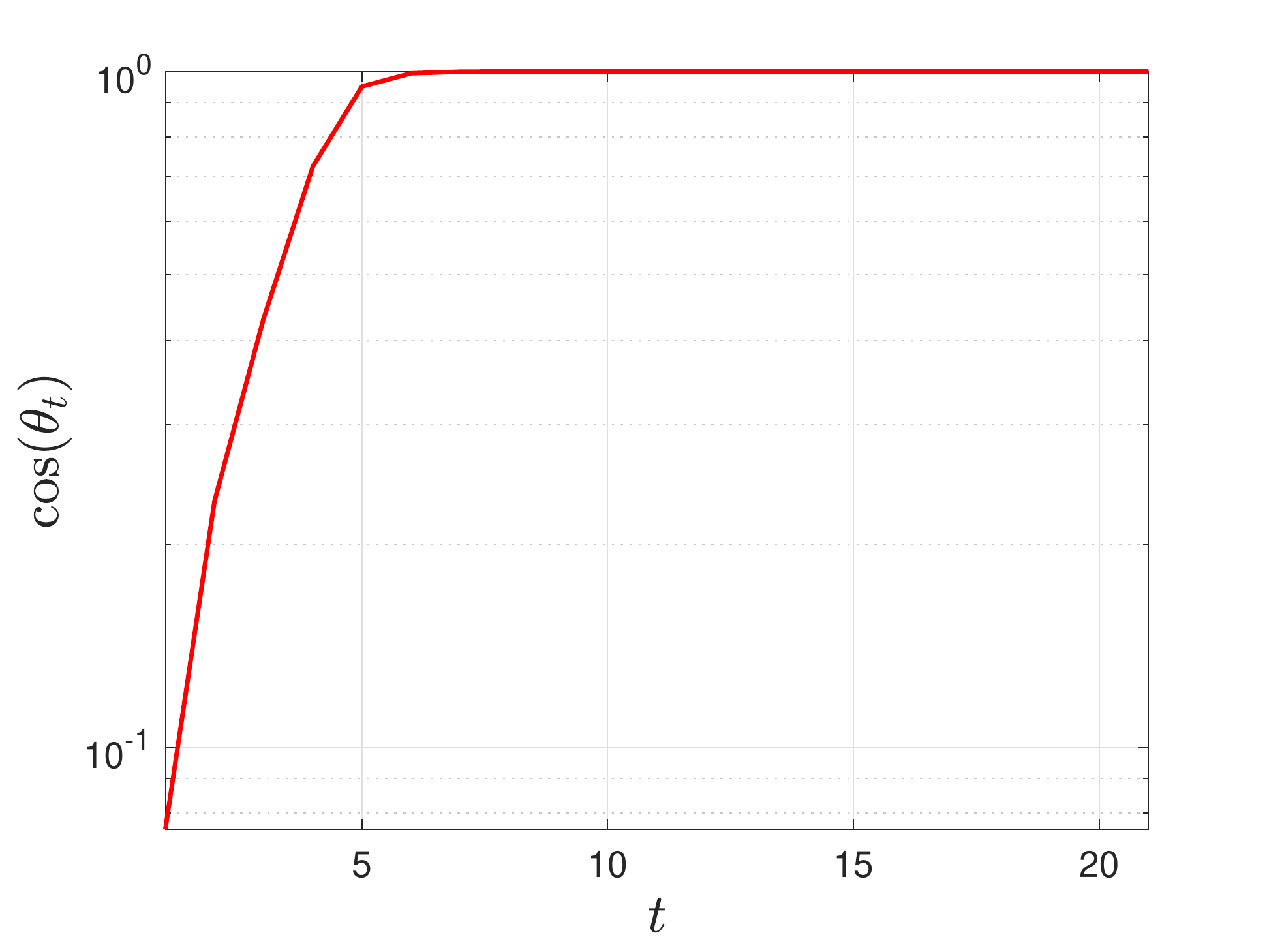}
        \caption{}
    \end{subfigure} \\
    \begin{subfigure}[b]{0.49\textwidth}
        \centering
        \includegraphics[width=\textwidth]{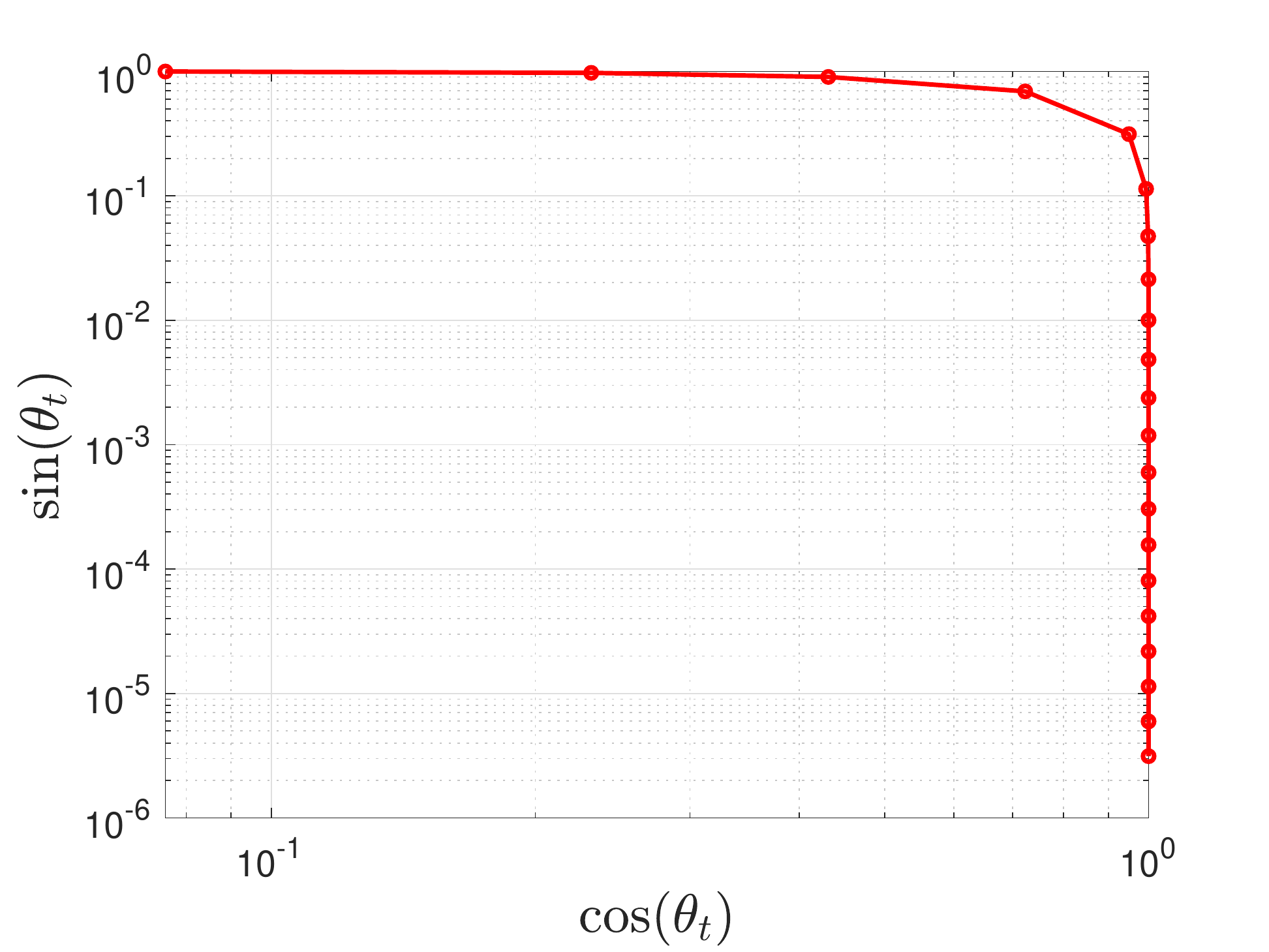}
        \caption{}
    \end{subfigure}
    \caption{Evolution of the iterates by randomly initialized ALS: The size of the ground truth matrix $X_0=\ustar \vstar^\top$ is given by $n_1 = n_2 = 256$. The number of measurements is $m = 3 (n_1+n_2)$. By $t$ we count the iterates. The estimation error is measured by the angle $\theta_t$ between $\vt$ and $\vstar$. (a) $\sin(\theta_t)$ vs number of iterations $t$; (b) $\cos(\theta_t)$ vs number of iterations $t$; (c) $\sin(\theta_t)$ vs $\cos(\theta_t)$.}
    \label{fig:iter}
\end{figure}

\section{Problem formulation} 
We consider the problem of estimating a rank-one matrix $\Xstar \in \mathbb{R}^{n_1 \times n_2}$ from $m$ random linear measurements given by equation \eqref{equ:linearsystem}.
In the following, we are going to assume that the measurement matrices $A_i \in \mathbb{R}^{n_1 \times n_2}$ are independent copies of a random matrix with i.i.d. entries following the standard normal distribution $\mathcal{N}(0,1)$. 

We  define the linear measurement operator $ \mathcal{A}: \R^{n_1 \times n_2} \rightarrow \mathbb{R}^m $ by
\begin{equation}
\label{def:calA}
\Aop{X} := \left( \frac{1}{\sqrt{m}} \froinnerproduct{A_{i}, X} \right)_{i \in [m]},
\end{equation}
where $\froinnerproduct{A_i,X} = \trace(A_i^\top X)$ denotes the Frobenius inner product between $A_i \in \mathbb{R}^{n_1 \times n_2}$ and $X \in \mathbb{R}^{n_1 \times n_2}$ and $[m] := \{1,2,\dots,m\}$.
Since $\Xstar$ is a rank-one matrix, we can assume without loss of generality that $\Xstar = \ustar \vstar^\top$ for some $ \ustar \in \mathbb{R}^{n_1}  $ and $ \vstar \in \mathbb{R}^{n_2} $.
This implies that the equation  \eqref{equ:linearsystem} can be equivalently written as
\begin{equation*}
    y= \mathcal{A} \left( \ustar \vstar^\top \right).
\end{equation*}
Moreover, note that using this notation equation \eqref{equ:minimization_problem} can be written equivalently as
\begin{equation}
\label{eq:ncp}
    \mathop{\mathrm{minimize}}_{u,v} f(u,v) := \frac{1}{2} \, \Vert y - \Aop{uv^\top} \Vert^2.
\end{equation}
We will consider a solution to \eqref{eq:ncp} by an Alternating Least Squares (ALS) method given in Algorithm~\ref{alg:als}. 

\begin{algorithm}[H]
\SetAlgoLined
\caption{Alternating Least Squares}
\label{alg:als}
\DontPrintSemicolon
\KwIn{linear measurement operator $\mathcal{A}: \mathbb{R}^{n_1 \times n_2} \rightarrow \mathbb{R}^m$, observation vector $y \in \mathbb{R}^{m}$, random initialization $v_0 \in \mathbb{R}^{n_2}$}
 \For{$t = 1,2, \ldots$}{
	$\utplushalf = \mathop{\mathrm{argmin}}_u \Vert y - \Aop{u \vt^\top}  \Vert^2$\;
	$\utplus = \utplushalf/ \Vert \utplushalf \Vert$\;
	$\vtplushalf = \mathop{\mathrm{argmin}}_v \Vert y - \Aop{ \utplus v^\top}  \Vert^2$\;
	$ \vtplus = \vtplushalf/ \Vert \vtplushalf \Vert$\;
 }
\end{algorithm}
Note that compared to \eqref{intro_altmin} there is an additional normalization step in Algorithm \ref{alg:als}. However, we have added it only for the sake of convergence analysis and this normalization step is not required for the reconstruction of $\Xstar$.

\section{Main result}
Our main result states that if the initialization vector $v_0 \in \mathbb{R}^{n_2}$ is chosen at random from the sphere with uniform distribution, then ALS converges to the true solution with high probability.
\begin{theorem}[Convergence of ALS]
\label{thm:main}
Let $\ustar \in \mathbb{R}^{n_1} \setminus \left\{ 0 \right\}$ and $ \vstar  \in \mathbb{R}^{n_2} \setminus \left\{ 0 \right\}$.  Let $\mathcal{A}: \mathbb{R}^{n_1 \times n_2} \rightarrow \mathbb{R}^m$ be the measurement operator as defined in \eqref{def:calA}, where $A_1,\dots,A_m \in \mathbb{R}^{n_1 \times n_2}$ are independent copies of a random matrix whose entries are i.i.d. following $\mathcal{N}(0,1)$. 
Let the observations in $y\in \mathbb{R}^m$ be given by $y= \mathcal{A} \left( \ustar  \vstar^\top  \right)$. 
Let $v_0 \in \mathbb{R}^{n_2}$ be a random initialization vector sampled from the unit sphere with the uniform distribution. 
Then there exists an absolute constant $C>0$ such that if the number of measurements $m$ satisfies
\begin{equation}
\label{eq:samp_complexity}
m \ge C \max(n_1,n_2) \log^{4} n_2, 
\end{equation}
then with probability at least $1 - O(\min(n_1,n_2)^{-1})$ the following holds. 
For every $\varepsilon>0$, after 
\begin{equation}\label{ineq:number_iterations}
t \ge C \left( \frac{ \log n_2}{\log \log n_2} + \frac{\log(1/\varepsilon)}{\log \log n_2} \right)
\end{equation}
iterations, the estimates $v_t$ and $u_t$ from Algorithm~\ref{alg:als} satisfy
\[
\max \left\{ \sin\left(  \angle(u_t,\ustar )\right); \sin\left(  \angle(v_t,\vstar ) \right)  \right\}\leq \varepsilon.
\]
\end{theorem}

There are a few remarks in order regarding Theorem \ref{thm:main}. We first note that the required sample-complexity \eqref{eq:samp_complexity} is optimal up to $\log$-factors. Indeed, the numbers of degrees of freedom of the unknown rank-one matrix  $\ustar \vstar^{\top} \in \mathbb{R}^{n_1 \times n_2}$  is $n_1+n_2-1$ and hence we need to have at least at the order of
$\max \left\{ n_1;n_2 \right\}$ 
measurements in order to recover the underlying ground-truth matrix (see also \cite{kech_krahmer}).

An upper bound on the number of iterations to achieve $\varepsilon$-accuracy is given by inequality \eqref{ineq:number_iterations}. As already mentioned in the introduction, our proof shows that convergence can be separated into two distinct phases.
Moreover, as it will become clear from our proof the two summands in \eqref{ineq:number_iterations} can be attributed to Phase 1 and Phase 2 as follows
\begin{equation}\label{equ:number_iterations}
  \underbrace{ C \, \frac{ \log n_2}{\log \log n_2}}_{\text{Phase 1}} + \underbrace{C \, \frac{\log(1/\varepsilon)}{\log \log n_2} }_{\text{Phase 2}}.
\end{equation}
Since the initialization vector $v_0$ is sampled from the sphere with uniform distribution, we expect that
\begin{equation*}
    \vert \innerproduct{v_0, \vstar } \vert \approx 1/\sqrt{n_2}.
\end{equation*}
Hence, we start with an initialization with is near-orthogonal to the ground truth.
However, as the \eqref{equ:number_iterations} shows we only need $O (\frac{ \log n_2}{\log \log n_2})$ iterations to obtain an iterate which is closely aligned with the ground truth.
After that, we enter the second phase. In this phase, ALS converges linearly to the ground truth as can be seen from the corresponding upper bound on the number of iterations $ O \left( \frac{ \log(1/\varepsilon)}{\log \log n_2} \right) $.

At the end, we stress again that crucially all of this is proven without the need for sample splitting, i.e., for each ALS step the same measurements are used.

\section{Related Work and Discussion}\label{Sec:related_work}
There has been a flurry of work on low-rank matrix recovery over the last fifteen years.  
For this reason, we will only provide a selective overview of the topic, highlighting the results which are most relevant to our work.
In fact, many different algorithmic approaches have been proposed for the low-rank matrix recovery problem.
The nuclear norm minimization approach \cite{recht2010guaranteed} has been studied for Matrix Completion in \cite{candes_matrixcompletion,candestao_matrixcompletion,gross_matrixcompletion}, for Phase Retrieval in \cite{candesphaselift1,mahdiphaselift,gross_phasellift}, for  Robust PCA in \cite{candes2011robust}, and for Blind Deconvolution in \cite{ahmed_blinddeconv} as well as its extension to the Blind Demixing problem in \cite{strohmer_demixing,jung_demixing}. We refer also to the overview article \cite{fuchs_overview} for further pointers to the literature.
Several other approaches, which have been proposed in the literature, are the projected gradient method \cite{jain2010guaranteed}, the iterative greedy algorithm \cite{lee2010admira}, and the Iteratively Reweighted Least Squares (IRLS) algorithm \cite{massimo_irls,mohan_irls}.

In recent years, there has been a flurry of work on non-convex approaches based on matrix factorization due to their small memory footprint and their low computational burden. These approaches can roughly be categorized as first-order methods based on gradient descent, e.g. \cite{GD_wirtinger,tu2016low}, and as methods based on alternating least squares (ALS) \cite{jain2013low}, which is also the method studied in this paper. We refer to \cite{chi2019nonconvex} for an overview of non-convex approaches based on matrix factorization.

\textbf{Non-convex gradient descent:} Non-convex methods based on gradient descent have been studied for the Matrix Sensing problem \cite{tu2016low}, for Blind Deconvolution \cite{GD_blinddeconv,gd_blinddeconv2}, its extension to Blind Demixing \cite{GD_blinddemixing} as well as for the Phase Retrieval problem \cite{GD_wirtinger,gd_phaseretrieval}.
However, all of these papers above only guarantee local convergence for gradient descent.
That is, convergence is only guaranteed if one picks initialization in a neighborhood of the true solution.
In most of these works, such an initialization is constructed via a so-called spectral initialization.

To obtain more insights into the global convergence properties of non-convex gradient descent based on matrix factorization people started to analyse the landscape of the loss function.
More precisely, this line of research tries to show that the landscape is benign in the sense that (i) all local minima are in fact global minima and (ii) saddle-points have at least one direction of strictly negative curvature. For the matrix sensing problem \cite{bhojanapalli2016global}, for the phase retrieval problem \cite{phaseretrieval_landscape}, and for the matrix completion problem \cite{matrixcompletion_landscape1,matrixcompletion_landscape2} it has been shown that the landscape of the loss function is benign. 
In \cite{saddle_avoidance} it has been shown that properties (i) and (ii) already imply convergence of gradient descent to a global minimum. However, \cite{du2017gradient} provides an example, that shows that this property does not rule out exponentially slow convergence. In particular, this means that properties (i) and (ii) do not guarantee convergence in polynomial time. 
Motivated by this, in \cite{chen2019gradient} the authors showed that in the Phase Retrieval problem with Gaussian measurement vectors gradient descent converges to the ground truth starting from random initialization by using a near-optimal amount of iterations and measurements.
In the case of symmetric low-rank matrix sensing, this was also shown in \cite{ma_algorithmic,stoger2021small}. However, these results require a random initialization which is chosen sufficiently small. 
For the asymmetric scenario, similar results \cite{ye_algorithmic, jiang_algorithmic} have only recently been obtained for the population loss case.
It remains an open problem to show an analogous result in the finite sample case.

\textbf{Alternating Least Squares:} In general, ALS has been widely used in a broad class of applications including low-rank approximation of data \cite{kroonenberg1980principal} and imaging \cite{o2007alternating}. 
In the context of low-rank matrix recovery, ALS approaches are arguably less well studied than methods based on gradient descent. There are several papers that study ALS (or some variants) for the matrix completion problem. However, these works either require fresh samples at each other iteration \cite{hardt_matrixcompletion1,hardt_matrixcompletion2,zhao_matrixCompletion} or they  show local convergence starting from a spectral initialization \cite{sun_matrixcompletion}. 

In \cite{kiryung_powerfactorization,deconv_powerfactorization,jakob_powerfactorization}, the authors propose to use alternating minimization combined with a projection step to recover a rank-one matrix with sparse entries from linear random measurements. 
However, their analysis requires an initialization close to the ground truth, which is a major bottleneck in the analysis. 
It is an interesting avenue for future work to see whether our analysis can also be extended to this algorithm.

For the phase retrieval problem, the Error Reduction (ER) algorithm has been proposed \cite{fienup1982phase,gerchberg1972practical}. While this method can be interpreted as an alternating minimization method, it is different from the ALS algorithm studied in this paper. 
Local convergence from a spectral initialization for the ER algorithm, in a setting where the measurements are Gaussian, has been first established in \cite{phase_netrapalli}, the analysis in this paper requires fresh samples for each iteration. This assumption has been removed by Waldspurger in \cite{phaseretrieval_waldspurger}, which showed local converge without sample splitting. Convergence from a random initialization has been established in \cite{phaseretrieval_zhang}, however using a (suboptimal) sample size at the order of $ n^{3/2}$.

The above discussion illustrates that our understanding of global convergence of non-convex methods in low-rank matrix recovery is still in its infancy. This paper contributes to this line of research by establishing the first convergence result from random initialization for the ALS method.

\textbf{Auxiliary sequences:} As already discussed in the introduction, in this paper, we construct a (virtual) auxiliary sequence to establish mild dependence of our ALS iterates on certain entries of the measurement matrices. For optimization tasks, such auxiliary sequences appeared before in  \cite{zhong_leaveoneout}, where the authors used a slightly different construction (leave-one-out sequences) to establish that the iterates depend only weakly on the individual measurements.
In \cite{chen_leaveoneout,li_leaveoneout}, the authors used leave-one-out sequences to show that gradient converges fast to the global optimum, when initialized in a local neighborhood, in several low-rank matrix recovery problems.
In \cite{ding_leaveoneout}, leave-one-out sequences were used to improve bounds for the required sample complexity of the nuclear norm minimization approach in matrix completion.

\section{Proof ideas and auxiliary sequences}

In this section, we illustrate the main ideas for proving Theorem~\ref{thm:main}. We will also introduce some necessary notation. Moreover, we will define a (virtual) auxiliary sequence, which will be a key ingredient in our proof.

\subsection{Notation}
Without loss of generality, we assume throughout the proof that $\Vert \ustar  \Vert= \Vert \vstar  \Vert =1 $.
Furthermore, we set $n:=\max(n_1,n_2)$. 
Moreover, the following shorthand notations will be used throughout this section. 
We consider the orthogonal decomposition of $\ut$ given by $\ut = u^{\parallel}_t + \ut^\perp$, where $u^{\parallel}_t :=  \mu_t \ustar$ with $\mu_t := \innerproduct{\ustar , u_t}$ and $ \ut^\perp := \ut - \ut^\parallel $ denote the projection of $\ut$ into the subspace spanned by $\ustar$ and its orthogonal complement. 
Consequently, $\|\ut^\parallel\|$ and $\|\ut^\perp\|$ respectively correspond to the cosine and sine of the angle between $\ut$ and $\ustar$. 
These will be used as metrics for convergence. 
Similarly, $\vt$ is decomposed as $\vt= v^{\parallel}_t + \vt^\perp$, where $v^{\parallel}_t :=  \lambda_t \vstar$ with $\lambda_t := \innerproduct{\vstar , v_t}$ and $ \vt^\perp := \vt - \vt^\parallel  $.
In an analogous fashion, we set $\mutplushalf= \innerproduct{\ustar ,\utplushalf}$. Then we have that
$\utplushalf = \utplushalf^{\parallel} + \utplushalf^{\perp}$, where $\utplushalf^{\parallel} := \mutplushalf \ustar$ and $\utplushalf^{\perp} := \mutplushalf- \utplushalf^{\parallel}$.

By $C>0$ we denote an absolute numerical constant, whose value may change from line to line.

\subsection{First-order necessary conditions}
Suppose that $\vt \in \mathbb{R}^{n_2}$ is given and that $\utplus$ is calculated via Algorithm \ref{alg:als}. Then it must hold that
\begin{equation*}
\nabla_{u} f \left( \utplushalf, \vt \right) =0.
\end{equation*}
By explicitly calculating the gradient it follows that
\begin{equation*}
\left[ \mathcal{A}^* \Aop{\utplushalf \vt^\top - \ustar \vstar^\top} \right] \vt
=0.
\end{equation*}
Note that by using $\twonorm{\vt}=1$ this expression can be rearranged as
\begin{equation}\label{eq:normal}
\utplushalf = \innerproduct{\vt,\vstar} \ustar+ \left[ \left(\Id - \mathcal{A}^* \mathcal{A} \right) \left( \utplushalf v_t^\top  - \ustar  \vstar^\top     \right)   \right] v_t.
\end{equation}
This identity will be used frequently in our analysis.

\subsection{Analysis in population loss}
To gain some intuition, we first consider the scenario where the number of samples $m$ is going to infinite, i.e., the population loss scenario.
Note that since $A_1,\dots,A_m$ are independent copies of a random matrix with i.i.d. standard Gaussian entries, it follows that in the scenario the measurement operator $\mathcal{A}$ is isotropic, i.e., $\E \left[ \mathcal{A}^* \mathcal{A} \right] = \Id $. 
Hence, it follows from Equation \eqref{eq:normal} that in this case
\begin{equation}\label{eq:normal_pop}
\utplushalf= \innerproduct{\vstar  , \vt}\ustar .
\end{equation}
This implies that a single step of Algorithm~\ref{alg:als} exactly recovers $\ustar$ up to a scale factor (under the assumption that $\innerproduct{\vstar  , \vt} \ne 0 $). 
The update on $\vtplushalf$ that follows will provide $u_{t+1} \vtplushalf^\top = \ustar \vstar^\top$. 
In other words, ALS from any nondegenerate initialization converges in a single iteration.

\subsection{Analysis in the finite-sample scenario}

At the sample level, the normal equation in \eqref{eq:normal} deviates from the population-level equation \eqref{eq:normal_pop} by the factor $\left[ \left(\Id - \mathcal{A}^* \mathcal{A} \right) \left( \utplushalf v_t^\top  - \ustar  \vstar^\top     \right)   \right] v_t$.
For this reason, we do not expect that one iteration will recover the signal as in the population loss scenario. 
Nevertheless, in the first convergence phase we aim to show that
\begin{equation}\label{eq:goal}
    \twonorm{ \utplus^\parallel} = \frac{\twonorm{ \utplushalf^\parallel}}{\twonorm{ \utplushalf} } \gg \twonorm{\vt^\parallel },
\end{equation}
meaning that the iterates become more aligned with the ground truth in each iteration.
To show this, we first decompose $\utplushalf$ into its parallel and its perpendicular part, i.e., $\utplushalf = \utplushalf^\parallel + \utplushalf^\perp  $.
We obtain that
\begin{equation}\label{ineq:intern34}
    \utplushalf^\parallel= \innerproduct{\vt,\vstar} \ustar+ \innerproduct{\ustar, \left[ \left(\Id - \mathcal{A}^* \mathcal{A} \right) \left( \utplushalf v_t^\top  - \ustar  \vstar^\top     \right)   \right] v_t} \ustar
\end{equation}
and 
\begin{equation*}
    \utplushalf^\perp = \left( \Id -\ustar \ustar^\top \right)  \left[ \left(\Id - \mathcal{A}^* \mathcal{A} \right) \left( \utplushalf v_t^\top  - \ustar  \vstar^\top     \right)   \right] v_t.
\end{equation*}
A standard approach to deal with the deviation term is to invoke the well-known Restricted Isometry Property (RIP), see Section~\ref{sec:concentration} as well as Lemma~\ref{lemma:perpdecrease}, which yields
\begin{equation}\label{eq:intern1}
\Vert \utplushalf^\perp\Vert \le  \frac{\delta}{1-\delta}\Vert \vt^{\perp}\Vert
\end{equation}
as well as
\begin{equation}\label{eq:intern2}
\big\Vert \utplushalf^\parallel - \innerproduct{\vstar ,v_t} \ustar  \big\Vert \le  \frac{\delta}{1-\delta} \Vert \vt^{\perp}\Vert
\end{equation}
for a RIP-constant $0<\delta < 1$.
While inequality \eqref{eq:intern1} will turn out to be sufficient to show \eqref{eq:goal}, inequality \eqref{eq:intern2} will not suffice.
The reason is that ideally we would like to have that
\begin{equation}\label{eq:intern3}
\twonorm{\utplushalf^{\parallel}} \approx \twonorm{\vt^{\parallel} } = \vert  \innerproduct{\vstar ,v_t}  \vert.
\end{equation}
However, this does not follow from \eqref{eq:intern2}. The reason for this is that we start from random initialization, which yields that $v_0 \in \mathbb{R}^{n_2}$ is almost orthogonal to the ground truth $\vstar$ in the sense that $\twonorm{v_0^\parallel} = \vert \innerproduct{v_0, \vstar} \vert \approx 1/\sqrt{n_2 }$ (and, consequently $\twonorm{v_0^\perp}$ is very close to $1$). 

In particular, this implies that \eqref{eq:intern2} is rather vacuous. 
Hence, we need to find other approaches to deal with the expression
\begin{equation}\label{eq:auxterm1}
\Big\vert  \innerproduct{\ustar, \left[ \left(\Id - \mathcal{A}^* \mathcal{A} \right) \left( \utplushalf v_t^\top  - \ustar  \vstar^\top     \right)   \right] v_t} \ustar \Big\vert
\end{equation}
in \eqref{ineq:intern34}. Note that we obtained inequality \eqref{eq:intern2} via the Restricted Isometry Property (RIP), which is a uniform bound, i.e., it holds for all vectors $ \utplushalf \in \mathbb{R}^{n_1} $ and $\vt \in \mathbb{R}^{n_2} $. In particular, it may be suboptimal for particular choices of $\vt $ and $\utplushalf$. For example, assume for a moment that $\vt $ and $\utplushalf$  would be independent of the measurement operator $\mathcal{A}$ (which of course is not the case). Under this assumption we could hope to derive much stronger concentration bounds than what could be obtained by a uniform estimate induced by the Restricted Isometry Property.

The key insight is that we can indeed establish that $\utplushalf$ and $\vt$ are \textit{nearly independent} of certain entries of the measurement matrices $\left\{ A_i \right\}_{i=1}^m$, which will allow us to go beyond the suboptimal estimates obtained via the Restricted Isometry Property.

More precisely, to show this \textit{near-independence}, we introduce a new set of measurement matrices $\{\tilde{A}_i\}_{i=1}^m$, which are obtained by substituting partial entries of the original measurement matrices as independent copies. 
This allows us to define a new measurement operator $\tilde{\mathcal{A}}$, which is constructed using the new measurement matrices $\{\tilde{A}_i\}_{i=1}^m$.
Then an auxiliary sequence of estimates $(\utaux,\vtaux)$ is obtained from the ALS algorithm starting from the same random initialization $v_0$, but replacing $\mathcal{A}$ with the new measurement operator $\tilde{\mathcal{A}} $. 
For a detailed and precise description of the construction of this auxiliary sequence, we refer to the next subsection. 

Next, we are going to establish that the trajectory of the auxiliary sequence will stay close to the trajectory the original sequence. Using this property, we expect that we can replace the expression \eqref{eq:auxterm1} by
\begin{equation*}
\Big\vert  \innerproduct{\ustar, \left[ \left(\Id - \mathcal{A}^* \mathcal{A} \right) \left( \utplushalfaux \vtaux^\top  - \ustar  \vstar^\top     \right)   \right] \vtaux} \ustar \Big\vert
\end{equation*}
as we expect those terms to be nearly the same. By leveraging that $\utaux $ and $\vtaux $ are independent of certain entries of $ \left\{ A_i \right\}_{i=1}^m $, we can now derive much stronger estimates for the above expression than what would be possible by solely relying on the RIP. 
These estimates allow us to show \eqref{eq:intern3}, from which we can in turn deduce \eqref{eq:goal}.
By inductively repeating these arguments we obtain that our iterates become more and more aligned with the ground truth signal until we enter the second convergence phase.

To show convergence in the second phase we then rely on well-known estimates induced by the Restricted Isometry Property of the measurement operator $\mathcal{A}$.

\subsection{Auxiliary sequences}

As our measurements follow a rotation-invariant distribution, we can assume without loss of generality that $\ustar = e_1 \in \mathbb{R}^{n_1} $ and $\vstar  = e_1 \in \mathbb{R}^{n_2}$. 
Here, with a slight abuse of notation, $e_1$ denotes the first standard basis vector such that the first entry is $1$ and the other entries are $0$. The ambient dimension will be clear from the context. 
We introduce an auxiliary measurement operator $ \tilde{\mathcal{A}} $, which is defined by
\begin{equation*}
\tilde{\mathcal{A}} \left(X\right) := \left( \frac{1}{\sqrt{m}} \froinnerproduct{\tilde{A}_{i}, X} \right)_{i \in [m]}
\end{equation*}
with the matrix $\tilde{A}_i$ given by
\begin{equation*}
(\tilde{A}_i)_{j,k}:= 
\begin{cases}
(A_i)_{j,k} & \text{if } ( j\ne 1 \text{ and } k\ne 1) \text{ or } \left(j,k\right) = \left(1,1\right), \\
(\hat{A}_i)_{j,k} & \text{else},
\end{cases}
\end{equation*}
where $(\hat{A}_i)_{j,k}$ are independent copies of $(A_i)_{j,k}$. We observe that it follows directly from the definition of the operator that 
\begin{equation*}
    y=\tilde{\mathcal{A}} \left( \ustar  \vstar^\top \right)  = \Aop{\ustar \vstar^\top}.
\end{equation*}

\noindent For our analysis we will need the following auxiliary sequences $ \left\{ \utaux \right\}_t $ and $ \left\{ \vtaux \right\}_t $. They are computed via the same algorithm as $ \left\{ \ut \right\} $ and $ \left\{ \vt \right\} $ except that the measurement operator $ \mathcal{A} $ is replaced by $ \tilde{\mathcal{A}} $.
We set $\tilde{v}_0 = v_0 $, that is, the auxiliary sequences start from the same initialization. Then for $ t \geq 0 $, the auxiliary sequences are iteratively updated by alternating least squares in the following four steps: Given $\vtaux$, $\tilde{u}_t$, the updates are computed via
\begin{align*}
\utplushalfaux&:= \mathop{\mathrm{argmin}}_{u \in \mathbb{R}^{n-1}} \Big\Vert y - \Aopaux{u \vtaux^\top}  \Big\Vert^2, 
& \utplusaux&:= \frac{\utplushalfaux}{\Vert \utplushalfaux \Vert}, \\
\vtplushalfaux&:= \mathop{\mathrm{argmin}}_{v \in \mathbb{R}^{n_2}} \Big\Vert y - \Aopaux{ \utaux v^\top}  \Big\Vert^2, & \vtplus&:= \frac{\vtplushalfaux}{\Vert \vtplushalfaux \Vert}. \end{align*}
\noindent Let $\tilde{f}: \mathbb{R}^{n_1} \times \mathbb{R}^{n_2} \rightarrow \mathbb{R}$ be defined by
\begin{equation*}
\tilde{f} \left(u,v\right) := \frac{1}{2} \Big\Vert y - \Aopaux{uv^\top} \Big\Vert^2.
\end{equation*}
Then its gradients with respect to $u$ and $v$ are respectively given by
\begin{align*}
\nabla_u \tilde{f} \left(u,v\right) &=   \tilde{\mathcal{A}}^* \left(  \Aopaux{uv^\top}-y  \right)   v, \\
\nabla_v \tilde{f} \left(u,v\right) &= \left[  \tilde{\mathcal{A}}^* \left(  \Aopaux{uv^\top}-y  \right)    \right]^\top u.
\end{align*}
We will now introduce some additional definitions, which will ease the notation in our proofs.
For each $i \in [m]$, we consider the decomposition $A_i = D_i + O_i$, where
\begin{align*}
D_i &= \ustar\ustar^\top A_i \vstar\vstar^\top + (I_{n_1} - \ustar\ustar^\top) A_i (I_{n_2} - \vstar\vstar^\top), \\
O_i &= \ustar\ustar^\top A_i (I_{n_2} - \vstar\vstar^\top) + \left( I_{n_1} - \ustar\ustar^\top \right) A_i \vstar\vstar^\top. \end{align*}
Moreover, we set 
\begin{align*}
\tilde{O}_i &= \ustar\ustar^\top \tilde{A}_i (I_{n_2} - \vstar\vstar^\top) + \left( I_{n_1} - \ustar\ustar^\top \right) \tilde{A}_i \vstar\vstar^\top.   
\end{align*}
We observe that it follows directly from these definitions that for all $ i \in [m] $
\begin{align*}
    A_i&=D_i + O_i,\\
    \tilde{A}_i&=D_i + \tilde{O}_i.
\end{align*}
This allows us define the following linear operators \begin{align*}
\mathcal{D} \left(X\right) := \left( \frac{1}{\sqrt{m}} \froinnerproduct{D_{i}, X} \right)_{i \in [m]},\\
\mathcal{O} \left(X\right) := \left( \frac{1}{\sqrt{m}} \froinnerproduct{O_{i}, X} \right)_{i \in [m]},\\
\tilde{\mathcal{O}} \left(X\right) := \left( \frac{1}{\sqrt{m}} \froinnerproduct{\tilde{O}_{i}, X} \right)_{i \in [m]}.
\end{align*}
Note that it follows immediately from these definitions that $\mathcal{A}$ and $\tilde{\mathcal{A}}$ can be decomposed as
\begin{align*}
    \mathcal{A} = \mathcal{D} + \mathcal{O}
    \quad \text{and} \quad 
    \tilde{\mathcal{A}} = \mathcal{D} +  \tilde{\mathcal{O}}.
\end{align*}

Throughout the proof we need to show that the original sequence and the true sequence stay close to each other. For that, we will establish that the inequalities
\begin{align*}
\max \left\{ \twonorm{ \utplus^\parallel - \utplusaux^\parallel };  \twonorm{ \utplus^\perp - \utplusaux^\perp }  \right\} \le c_{2t+1} \twonorm{ \utplus^\parallel  }
\end{align*}
and
\begin{align*}
\max \left\{ \twonorm{ \vtplus^\parallel - \vtplusaux^\parallel };  \twonorm{ \vtplus^\perp - \vtplusaux^\perp }  \right\} &\le c_{2t+2} \Vert \vtplus^{\parallel} \Vert
\end{align*}
hold (see Lemma \ref{lemma:closeness4}), where $c_t$ is defined as 
\begin{equation}
\label{def:ct}
c_t:= \left(1+\frac{1}{\log n_2} \right)^t-1
\end{equation}
for any natural number $t$. Note that this implies that in the first few iterations, where $\twonorm{ \utplus^\parallel  } $, respectively $ \twonorm{ \vtplus^\parallel  } $, is small, the original iterates and the iterates from the auxiliary sequence are close to each other. In particular, this shows that, in the beginning, the ALS trajectories (or the virtual trajectories) do depend only mildly on $\{ O_i \}_{i=1}^m$, respectively $\{ \tilde{O}_i \}_{i=1}^m$.

As already noted in Section \ref{Sec:related_work}, in \cite{chen2019gradient} an auxiliary sequence with similar properties has been constructed for the analysis of gradient descent for the phase retrieval problem.
However, as the algorithms under consideration are quite different, the proofs which show that the auxiliary sequences stay close too each other are quite different.
As it turns out, a key difficulty in our proof lies in showing that the auxiliary sequence and the original sequence are still close after the normalization step (see Lemma \ref{lemma:closeness4} and its proof in Appendix \ref{app:proof_normalization}).

\section{Proof of Theorem~\ref{thm:main}}

In this section, we will provide the details for the proof of Theorem \ref{thm:main}.
We first list several concentration inequalities, which will be used throughout the proof. They are consequences of the Restricted Isometry Property (RIP) of the measurement operator $\mathcal{A}$ and also of the near-independence of auxiliary sequences from the measurement matrices. 
Then the main proof arguments will be built upon these results. 

\subsection{Concentration inequalities} 
\label{sec:concentration}
We proceed with the proof of Theorem~\ref{thm:main} under a set of events on $\mathcal{A}$ and $\mathcal{\tilde{A}}$, which hold with high probability. 
These events are stated in Lemmas~\ref{lemma:randombound1}, \ref{lemma:indepbound1}, and \ref{lemma:indepbound2}, whose proofs are deferred to the appendix.
First note that the linear operator $\mathcal{A}$ satisfies the restricted isometry property. 
\begin{lemma}[{A special case of \cite[Theorem~2.3]{candes2011tight}}]\label{lemma:RIP}
Let $\mathcal{A}$ be the linear operator defined in \eqref{def:calA}. 
There exists a numerical constant $C_0$ such that if 
\begin{equation*}
	m \ge C_0 \delta^{-2} \max(n_1,n_2),
\end{equation*}
then with probability at least $1 - O\left( \exp (-c m) \right)$
\begin{equation}\label{eq:rip}
(1-\delta) \fronorm{Z}^2 
\le \twonorm{\mathcal{A} \left(Z\right)}^2 
\le (1+\delta) \fronorm{Z}^2
\end{equation}
holds for all matrices $Z \in \R^{n_1 \times n_2}$ with rank at most $4$.
\end{lemma}

The following results, whose proof is deferred to Appendix \ref{app:rip_consequences}, are direct consequences of the restricted isometry property and will be used throughout the remainder of the proof. 
\begin{lemma}\label{lemma:RIPlemma}
Suppose that $\mathcal{A}$ satisfies the restricted isometry property in \eqref{eq:rip} with constant $\delta>0$. 
Then for all $ u \in \R^{n_1}, v \in \R^{n_2}$, we have
\begin{align}
\big\Vert \mathcal{O}^* \mathcal{D} \left(uv^\top\right)  \big\Vert 
&\le \delta \twonorm{u} \twonorm{v}, \label{eq:res_rip1} \\
\big\Vert \mathcal{D}^* \mathcal{O} \left(uv^\top\right)  \big\Vert 
&\le \delta \twonorm{u} \twonorm{v}, \label{eq:res_rip2}
\end{align}
and
\begin{align}
\Vert \left( \mathcal{O}^* \mathcal{O} - \mathcal{P_O} \right) \left(uv^\top \right)  \Vert  & \le \delta \fronorm{uv^\top}, \label{eq:res_rip3} 
\end{align}
where the orthogonal projection $\mathcal{P_O}: \mathbb{R}^{n_1 \times n_2} \rightarrow \mathbb{R}^{n_1 \times n_2}$ 
is defined as 
\begin{equation*}
\mathcal{P_O} \left( Z \right) = \ustar \ustar^\top Z \left(I_{n_2} - \vstar \vstar^\top \right)+\left(I_{n_1} - \ustar \ustar^\top \right) Z \vstar \vstar^\top .     
\end{equation*}
Moreover, if $\innerproduct{u_1v_1^\top,u_2 v^\top_2} =0$ holds, then we have that
\begin{equation}
\label{eq:rop}
\big\vert \innerproduct{ \mathcal{A} \left( u_1v_1^\top \right), \mathcal{A} \left( u_2 v_2^\top \right) }	\big\vert \le  \delta \fronorm{ u_1v_1^\top  } \fronorm{ u_2 v_2^\top }.
\end{equation}
\end{lemma}
\noindent By construction, $\mathcal{\tilde{A}}$ and $\mathcal{\tilde{O}}$ satisfy the same properties in Lemmas~\ref{lemma:RIP} and \ref{lemma:RIPlemma}. 

Next, recall that $\ustar=e_1$. 
We will also use the following standard concentration result, whose proof can be found in Appendix~\ref{app:concentration1}.
\begin{lemma}\label{lemma:randombound1}
With probability at least $ 1-\mathcal{O} \left(  \exp \left(-c \min \left\{ m; \min(n_1,n_2) \right\} \right) \right)$ it holds that
\begin{equation}\label{ineq:randombound1a}
\frac{1}{m} \Big\Vert  \left[  \sum_{i=1}^{m} \left(A_i\right)_{1,1}  O_i    \right] \ustar \Big\Vert \le 4 \sqrt{\frac{n_1}{m}}
\end{equation}
and
\begin{equation}\label{ineq:randombound1b}
\frac{1}{m} \Big\Vert  \left[  \sum_{i=1}^{m} \left(A_i\right)_{1,1}  \tilde{O_i}    \right] \ustar \Big\Vert \le 4 \sqrt{\frac{n_1}{m}}.
\end{equation}
\end{lemma}	
Finally, by construction the auxiliary sequences are independent from the off-diagonal blocks of the measurement matrices. Therefore we obtain the following lemmas, which are proved in Appendices~\ref{app:concentration2} and \ref{app:concentration3}.
\begin{lemma}\label{lemma:indepbound1}
Let $T \in \mathbb{N}$ and let $\eeta>0$. With probability at least $1-\eeta^{-1}-\mathcal{O} \left( \exp \left(-cm \right) \right)$, it holds for all $t\in [T]$ that 
\begin{equation}\label{ineq:indepbound1a}
\frac{1}{m} \Big\Vert  \left[ \sum_{i=1}^{m} \left(A_i\right)_{1,1}  O_i \right] \vtaux^{\perp} \Big\Vert  \lesssim   \sqrt{ \frac{ \log T + \log \eeta}{m} } \cdot \Vert  \vtaux^{\perp}  \Vert
\end{equation}
and
\begin{equation}\label{ineq:indepbound1b}
\frac{1}{m} \Big\Vert  \left[  \sum_{i=1}^{m} \left(A_i\right)_{1,1}  \tilde{O_i}    \right] \vt^{\perp} \Big\Vert  \lesssim  \sqrt{ \frac{ \log T + \log \eeta}{m} } \cdot \Vert \vtaux^{\perp}  \Vert.
\end{equation}
\end{lemma}

\begin{lemma}\label{lemma:indepbound2}
Let $T \in \mathbb{N}$ and let $\eeta>0$. 
With probability at least $1 - \eeta^{-1}$ it holds for all $ t \in [T] $ simultaneously that
\begin{equation}\label{ineq:indepbound2a}
 \frac{1}{m} \Big\vert \sum_{i=1}^m \innerproduct{ O_i^\top e_i, \vtaux } \innerproduct{ D_i, \utplushalfaux^{\perp} (\vtaux^\perp)^\top  } \Big\vert
  \lesssim \sqrt{\frac{\log T + \log \eeta}{m}}  \cdot  \Big\Vert \mathcal{A} \left( \utplushalfaux^{\perp} \left( \vtaux^{\perp}\right)^\top \right) \Big\Vert
\end{equation}
and
\begin{equation}\label{ineq:indepbound2b}
\frac{1}{m} \Big\vert  \sum_{i=1}^m  \innerproduct{ \tilde{O_i}^\top e_i, \vtaux } \innerproduct{ D_i, \utplushalf^{\perp} (\vt^\perp)^\top }\Big\vert
\lesssim  \sqrt{\frac{\log T + \log \eeta}{m}} \cdot \Big\Vert \mathcal{A} \left( \utplushalf^{\perp} \left( \vt^{\perp}\right)^\top \right) \Big\Vert.  
\end{equation}
\end{lemma}

The inequalities in Lemmas~\ref{lemma:randombound1}, \ref{lemma:indepbound1}, and \ref{lemma:indepbound2} together with the RIP of the measurement operators $\mathcal{A}$ and $\tilde{\mathcal{A}}$ imply the following inequalities in Lemma \ref{lemma:stochasticbound}, Lemma \ref{lemma:help1}, and Lemma \ref{lemma:nearindependencebounds}. 
The proofs are also deferred to Appendices~\ref{app:concentration4}, \ref{app:concentration5}, and \ref{app:concentration6}

\begin{lemma}\label{lemma:stochasticbound}
Suppose that \cref{ineq:randombound1a,ineq:randombound1b,ineq:indepbound1a,ineq:indepbound1b} hold. 
Furthermore, suppose that both $\mathcal{A}$ and $\tilde{\mathcal{A}}$ satisfy the RIP with constant $\delta >0$.
Then it holds that
\begin{equation*}
\Big\Vert  \left[  \left(    \mathcal{A}^* \mathcal{A}  - \tilde{ \mathcal{A}}^* \tilde{\mathcal{A}}   \right) \left(    \ustar \vstar^\top \right)   \right] \vtaux \Big\Vert
\lesssim \left(  \sqrt{\frac{\log T + \log \eeta}{m}}+  \sqrt{  \frac{n_1}{m}}  \Vert \vtaux^{\parallel} \Vert \right) + \delta \Vert \vt -\vtaux \Vert.
\end{equation*}	
\end{lemma}

\begin{lemma}\label{lemma:help1}
Suppose that \cref{ineq:randombound1a,ineq:randombound1b,ineq:indepbound1a,ineq:indepbound1b,ineq:indepbound2a,ineq:indepbound2b} hold. 
Moreover, suppose that the measurement operators $\mathcal{A}$ and $\tilde{\mathcal{A}}$ satisfy RIP with constant $\delta>0$ and that we have $\Vert \utplushalf \Vert \le 2$ as well as $\Vert \utplushalfaux \Vert \le 2$.
Then it holds that
\begin{align*}
&\Big\Vert \left[  \left(    \mathcal{A}^* \mathcal{A}  - \tilde{ \mathcal{A}}^* \tilde{\mathcal{A}}   \right) \left(  \utplushalfaux \vtaux^\top  \right)   \right] \vtaux  \Big\Vert\\
\lesssim &  \sqrt{\frac{\log T + \log \eeta}{m}} + \delta \Vert  \vt^{\parallel} \Vert + \left( \delta + \sqrt{\frac{n_1}{m}} \right) \Vert  \tilde{\vt}^{\parallel} \Vert  + \delta \Vert \vtaux -  \vt  \Vert  + \delta  \Vert \utplushalfaux - \utplushalf  \Vert .
\end{align*}
\end{lemma}

\begin{lemma}\label{lemma:nearindependencebounds}
Suppose that \cref{ineq:indepbound1a,ineq:indepbound1b,ineq:indepbound2a,ineq:indepbound2b} hold. 
Furthermore, suppose that the measurement operator $\mathcal{A}$ satisfies RIP with constant $\delta >0$ and that $\twonorm{\utplushalfaux} \le 2$. 
Then there exists an absolute constant $C>0$ for which it holds that 
\begin{align}\label{ineq:intern5}
 \Big\vert \innerproduct{ \mathcal{A}  \left(  \ustar  \left( v^{\perp}_t  \right)^\top \right), \mathcal{A} \left(  \ustar  \vstar^\top    \right)     } \Big\vert \le \delta \Vert  \vt^\perp -\vtaux^\perp \Vert + C \sqrt{\frac{\log T + \log \eeta}{m}}
\end{align}
and 
\begin{equation}\label{ineq:intern6}
\begin{aligned}
& \Big\vert \innerproduct{\mathcal{A} \left( \ustar  \left( v^{\perp}_t \right)^\top \right)  , \mathcal{A} \left(   \utplushalf^{\perp} \left( v^{\perp}_t \right)^\top  \right) } \Big\vert \\
& \le \delta \Vert \utplushalf^\perp - \utplushalfaux^\perp \Vert + 2\delta \Vert \vt^\perp -\vtaux^\perp \Vert +C \sqrt{\frac{\log T + \log \eeta}{m}}.
\end{aligned}
\end{equation}
\end{lemma}

\begin{remark}\label{remark:conc_ineq}
The inequalities in Lemmas \ref{lemma:randombound1} to \ref{lemma:nearindependencebounds}
will be used to analyze the update of $u_{t}$ to $u_{t+1}$ given $v_t$ by the normal equation in \eqref{eq:normal} (by the least-squares minimization step and by the normalization step in Algorithm \ref{alg:als}). 
To analyze the ALS update from $v_t $ to $v_{t+1} $ given $ u_{t+1}$ we will need analogous inequalities in order to be able to analyze these updates.
Due to symmetry of the problem, the statements and proofs of these analogous results can be obtained in an analogous way. 
For this reason, to keep the presentation concise we omit the statements and proofs of analogous versions of these lemmas.
\end{remark}

\subsection{Phase 1: From random initialization to a local neighborhood of the ground truth}

Since the initialization vector $v_0 \in \mathbb{R}^{n_2}$ is chosen from the sphere with uniform distribution, with probability at least $1- \mathcal{O} \left( n_2^{-1}\right)$, the random initialization $v_0 \in \mathbb{R}^{n_2}$ satisfies
\begin{equation}\label{ineq:init}
\big\Vert v^{\parallel}_0 \big\Vert = \vert \innerproduct{v_0,\vstar } \vert \ge \frac{1}{2\sqrt{n_2 \log n_2}}.
\end{equation}
Then the following proposition illustrates the convergence properties of the ALS iterates $ \left\{ \ut \right\}_t $ and $  \left\{ \vt \right\}_t $ to a neighborhood of $\vstar$ in Phase 1. 

\begin{proposition}\label{prop:stage1}
There exists a numerical constant $c>0$ for which the following holds. 
Suppose that 
\begin{itemize}
    \item[i)] $\mathcal{A}$ and $\mathcal{\tilde{A}}$ satisfy RIP with constant $\delta = \frac{c}{4\log n_2}$.
    \item[ii)] $m \ge \delta^{-2} \left( n_1 + n_2 \right) \log n_2 \log |T|$, where $T = \left\lceil \frac{\log n_2}{4\log\log n_2} \right\rceil$. 
    \item[iii)] \cref{ineq:randombound1a,ineq:randombound1b} hold. 
    \item[iv)] \cref{ineq:indepbound1a,ineq:indepbound1b,ineq:indepbound2a,ineq:indepbound2b} hold for all $t \in [T]$ with $\eeta = n_2 $. 
    \item[v)] $v_0$ satisfies \eqref{ineq:init}. 
    \item[vi)] Analogous inequalities of iii) and iv) hold for updating $v_{t}$ to $v_{t+1}$ given $u_{t+1}$ (see Remark \ref{remark:conc_ineq}) with $ \eeta=n_2 $. 
\end{itemize}
Then for every $t \in [T]$ it holds that
\begin{equation}
\Vert \vt^{\parallel} \Vert 
\ge \left( \log n_2 \right)^{2t} \cdot \Vert v_0^{\parallel} \Vert \ge \frac{\left( \log n_2 \right)^{2t}}{2\sqrt{n_2 \log n_2}} \label{ineq:induction1}
\end{equation}
and
\begin{equation}
\max \left\{ \twonorm{ \vt^\parallel - \vtaux^\parallel };  \twonorm{ \vt^\perp - \vtaux^\perp }  \right\} 
\le \ct \Vert \vt^{\parallel} \Vert, \label{ineq:induction2} 
\end{equation}
where $c_t$ is defined in \eqref{def:ct} until we have that
\begin{equation}\label{ineq:end_stage1}
\min\left\{ \twonorm{\vt^\parallel}, \twonorm{\ut^\parallel} \right\} 
\ge \frac{c}{\log n_2}. 
\end{equation}
\end{proposition}

\begin{proof}[Proof of Proposition \ref{prop:stage1}]
It suffices to only consider the case when the initialization vector $v_0 \in \mathbb{R}^{n_2}$ does not satisfy \eqref{ineq:end_stage1}. Otherwise there is nothing to prove.

We are going to show by induction that \eqref{ineq:induction1} and \eqref{ineq:induction2} hold until condition \eqref{ineq:end_stage1} is fulfilled. In particular, note that by our choice of $T$ this immediately implies that \eqref{ineq:end_stage1} holds for some $t\le T$.
For the base case, observe that for $t=0$ the two inequalities in \eqref{ineq:induction1} and \eqref{ineq:induction2} are satisfied since we have $v_0 = \tilde{v}_0$ by definition and since we assume that inequality \eqref{ineq:init} holds.

For the induction step, suppose that the statements hold for some natural number $t$ with $t \le T$.
Then we will show that the statements also hold for $t+1$ whenever \eqref{ineq:end_stage1} is not yet satisfied. 
To this end, we first show that the estimation error and the norm of the next least-squares update $\utplushalf$ are upper-bounded as shown in the following lemma. It is proved in Appendix~\ref{app:perpdecrease}.

\begin{lemma}\label{lemma:perpdecrease}
Suppose that $\mathcal{A}$ satisfies RIP for $0 < \delta < 1$ and $\Vert v_t \Vert=1$. Then it holds that
\begin{equation}\label{equ:bound1}
\big\Vert \utplushalf - \innerproduct{\vstar ,v_t} \ustar  \big\Vert \le  \frac{\delta}{1-\delta} \Vert \vt^{\perp}\Vert.
\end{equation}
In particular, it follows that
\begin{equation}\label{equ:bound2}
\Vert \utplushalf^\perp\Vert \le  \frac{\delta}{1-\delta}\Vert \vt^{\perp}\Vert.
\end{equation}
Moreover, for $\delta \le \frac{1}{2}  $, we have that
\begin{equation}\label{equ:bound3}
\Vert \utplushalf \Vert \le 2.
\end{equation}
\end{lemma}	

\noindent Analogously, since $\tilde{\mathcal{A}}$ also satisfies the RIP with the same constant $\delta$ and since $\twonorm{\vtaux}=1$ holds, we also have
\begin{align}
\big\Vert \utplushalfaux - \innerproduct{\vstar ,\vtaux} \ustar  \big\Vert &\le  \frac{\delta}{1-\delta} \Vert \vtaux^{\perp}\Vert, \label{equ:bound4} \\
\Vert \utplushalfaux^\perp\Vert &\le  \frac{\delta}{1-\delta}\Vert \vtaux^{\perp}\Vert,\label{equ:bound5} \\
\Vert \utplushalfaux \Vert &\le 2. \label{equ:bound6}
\end{align}

\noindent Given the upper estimates in \cref{equ:bound3,equ:bound6}, the next lemma, proven in Appendix~\ref{app:closeness3}, shows that the distances between the least-square updates of the original and auxiliary sequences stay close each other. 

\begin{lemma}\label{lemma:closeness3}
Under the hypothesis of Proposition~\ref{prop:stage1}, suppose that \cref{equ:bound3,equ:bound6} hold. 
Let $t\in \mathbb{N}$ and assume furthermore that the inequalities \eqref{ineq:induction1} and \eqref{ineq:induction2} hold.
Then there exists an absolute constant $C_1>0$ for which the followings hold:
\begin{align}
\twonorm{  	\utplushalf^{\parallel} - 	\utplushalfaux^{\parallel}}	
&\le \left( \ct + C_1 \delta(1+\ct) \right) \twonorm{ \vt^\parallel }  ,\label{ineq:closeness4}\\
\twonorm{  	\utplushalf^{\perp} - 	\utplushalfaux^{\perp}} 	&\le C_1 \delta \left(1+\ct\right) \twonorm{\vt^\parallel}.  \label{ineq:closeness5}
\end{align}
\end{lemma}	

\noindent The upper estimates in \eqref{ineq:closeness4} and \eqref{ineq:closeness5} imply that $\twonorm{\utplushalf^\parallel}$ is close to $\twonorm{ \vt^{\parallel} }$, which is stated in the following lemma, see Appendix~\ref{app:uparallelbound}. 

\begin{lemma}\label{lemma:uparallelbound}
Under the hypothesis of Proposition~\ref{prop:stage1}, suppose that \cref{equ:bound3,ineq:closeness4,ineq:closeness5} hold. 
Moreover, let $t\in \mathbb{N}$ and assume that the inequalities \eqref{ineq:induction1} and \eqref{ineq:induction2} hold.
Then there exists an absolute constant $C_2>0$ for which the followings hold:
\begin{equation}\label{ineq:uparallelbound}
  \left( 1- C_2 \delta (1+\ct) \right)  \big\Vert   \utplushalf^{\parallel} \big\Vert \le \Vert \vt^{\parallel} \Vert \le   \left( 1+ C_2 \delta \left(1+\ct\right) \right)  \big\Vert   \utplushalf^{\parallel} \big\Vert.
\end{equation}	
\end{lemma}
\begin{remark}
Later on, we will in fact only use the upper bound on $ \Vert \utplushalf^{\parallel} \Vert  $ in inequality \eqref{ineq:uparallelbound}.
As there is no additional effort required in proving the lower bound as well, we also decided to include it in this manuscript.
\end{remark}

\noindent Moreover, since $T = \left\lceil \frac{\log n_2}{4 \log\log n_2} \right\rceil$, it follows that $c_{2t}$ is bounded from above by an absolute constant for all $t\le T$, which is formally stated in the following lemma. 

\begin{lemma}\label{lemma:ctbound} 
Then for all $t \le 2T \leq \left\lceil \frac{\log n_2}{2 \log\log n_2} \right\rceil + 1$ it holds that $c_t$ defined in \eqref{def:ct} satisfies $c_t \leq C$ for an absolute constant $C$. 
\end{lemma}
\begin{proof}
For all $t \leq 2T$ we have 
\begin{align*}
c_t + 1 
&= \exp \left( t \log \left( 1+ \frac{1}{\log n_2}  \right)  \right) \\
& \overleq{(a)} \exp \left(t/\log n_2\right) \\
& \le \exp \left(2T/\log n_2\right) \\
& \le \exp \left( \frac{1}{2\log\log n_2} + \frac{2}{\log n_2} \right)\\
& \le C
, 
\end{align*}	
where $(a)$ follows from the elementary inequality $\log \left(1+x\right)\le x$ for $x>0$. 
\end{proof}

\noindent Hence, for sufficiently small $c>0$, \eqref{ineq:uparallelbound} implies that
\begin{equation}\label{ineq:vtplushalfnorm}
\frac{1}{2}\Vert \vt^{\parallel} \Vert \le  \big\Vert   \utplushalf^{\parallel} \big\Vert. 
\end{equation}
The next lemma, proved in Appendix~\ref{app:proof_normalization}, shows that the original and auxiliary sequences stay close in $\ell_2$-distance under the conditions derived above. 

\begin{lemma}\label{lemma:closeness4}
Under the hypothesis of Proposition~\ref{prop:stage1}, suppose that \cref{ineq:closeness4,ineq:closeness5,ineq:uparallelbound} hold. 
Moreover, suppose that $c>0$ is chosen small enough (smaller than an absolute constant depending only on $C_1,C_2,C_3$).
Then it follows that
\begin{equation}\label{ineq:parallelauxcloseness}
\max \left\{ \twonorm{ \utplus^\parallel - \utplusaux^\parallel };  \twonorm{ \utplus^\perp - \utplusaux^\perp }  \right\} \le c_{2t+1} \cdot \twonorm{ \utplus^\parallel  }.
\end{equation}
\end{lemma}

\noindent We further proceed with the following lemma, which shows how the estimation error propagates with the normalization. The proof is provided in Appendix~\ref{app:convergence}.

\begin{lemma}\label{lemma:convergence}
Suppose that $\twonorm{\vt}=1$ and that for fixed $ t \in \mathbb{N} $ and real numbers $0<\beta < \alpha <1$ it holds that
\begin{align}
\Vert \utplushalf^{\parallel} \Vert^2&\ge \alpha \Vert  v^{\parallel}_t \Vert^2, \label{ineq:intern99_1} \\
\Vert \utplushalf^{\perp} \Vert^2 &\le  \beta    \Vert  v^{\perp}_t \Vert^2.  \label{ineq:intern99_2}
\end{align}
Then, whenever $\vt^{\parallel} \ne 0$, it holds that
\begin{equation}\label{ineq:convergence1}
\Vert \utplus^{\parallel} \Vert^2 
\ge  \frac{ \alpha \Vert \vt^{\parallel} \Vert^2}{ \beta + \left(\alpha - \beta \right) \Vert \vt^{\parallel} \Vert^2}
\ge \frac{  \Vert \vt^{\parallel} \Vert^2 }{ \frac{\beta}{\alpha} +\twonorm{\vt^\parallel}^2  } 
\end{equation}
and, moreover,
\begin{equation}\label{ineq:convergence2}
\Vert \utplus^{\perp} \Vert^2 \le  \frac{\beta}{\alpha \twonorm{\vt^\parallel}^2} \cdot  \twonorm{\vt^\perp}^2.
\end{equation}
\end{lemma}

\noindent Note that due to \eqref{equ:bound2} with $\delta < \frac{1}{2}$ and due to \eqref{ineq:vtplushalfnorm} the assumptions in Lemma~\ref{lemma:convergence} are satisfied with $\alpha = \frac{1}{4}$ and $\beta = 4\delta^2$. 
Therefore, with $\delta = \frac{c}{4\log n_2}$ and $\twonorm{\vt^\parallel} \le \frac{c}{\log n_2}$ we obtain that 
\begin{align}\label{equ:aux37}
\twonorm{\utplus^\parallel}^2 
& \ge 
\frac{\twonorm{\vt^\parallel}^2}{16\delta^2 + \twonorm{\vt^\parallel}^2} 
\ge \frac{2\log n_2}{c} \twonorm{\vt^\parallel}^2
\ge \left(  \frac{ 2\log n_2}{c} \right)^{2t+1} \twonorm{v_0^\parallel}. 
\end{align}
Since we have shown \eqref{ineq:parallelauxcloseness} and \eqref{equ:aux37} this finishes the induction step for $\utplus$.
With exactly the same reasoning we can then prove the inequalities
\begin{align}
\Vert \vtplus^{\parallel} \Vert 
&\ge \left( \frac{2 \log n_2}{c} \right)^{2t+2} \Vert v_0^{\parallel} \Vert, \label{eq:induction1_with_c} \\
\max \left\{ \twonorm{ \vtplus^\parallel - \vtplusaux^\parallel };  \twonorm{ \vtplus^\perp - \vtplusaux^\perp }  \right\} &\le c_{2t+2} \Vert \vtplus^{\parallel} \Vert. \nonumber
\end{align}
This shows inequalities \eqref{ineq:induction1} and \eqref{ineq:induction2} for $t+1$. 
Note that by choosing $c < \frac{1}{2}$ inequality \eqref{eq:induction1_with_c} implies \eqref{ineq:induction1}.
This completes the induction step. 
\end{proof}

\subsection{Phase 2: Linear convergence by RIP}

We enter the second phase as soon as the iterates are sufficiently aligned with the ground truth solution, that is when condition \eqref{ineq:end_stage1} is satisfied.
Once we enter the second phase, our iterates converge linearly to the ground truth as it is shown by the next proposition, which describes the second phase.
\begin{proposition}\label{prop:stage2}
There exists a numerical constant $c'>0$ for which the following holds. 
Suppose that $\mathcal{A}$ satisfies RIP with constant $\delta = \frac{c'}{8\log n_2}$ and either $\Vert v_{ \hat{t}}^{\parallel} \Vert > \frac{c'}{\log n_2}$ or $\Vert u_{ \hat{t}}^{\parallel} \Vert > \frac{c'}{\log n_2}$ for some $\hat{t} \in \mathbb{N}$. 
Then it holds that for all $t > \hat{t}$
\begin{align}
\twonorm{\utplus^\perp}	
\le  \frac{1}{2} \left( \frac{1}{2\log n_2} \right)^{2(t-\hat{t})} \twonorm{ \vthat^\perp }
\quad \text{and} \quad
\twonorm{v_{t+1}^\perp}	
\le \frac{1}{2} \left( \frac{1}{2\log n_2} \right)^{2(t-\hat{t})+1}  \twonorm{ \vthat^\perp }. \label{ineq:aux71}
\end{align}
\end{proposition}

\begin{proof}
Due to the symmetry of the argument, we may assume without loss of generality that 
\begin{equation}\label{ineq:aux51}
\Vert \vthat^{\parallel} \Vert > \frac{c'}{\log n_2} = 8\delta.
\end{equation}
Next we show that
\begin{equation}\label{ineq:aux53}
	\twonorm{ \utplushat^\perp } \le \frac{1}{3} \, \twonorm{\vthat^\perp}.
\end{equation}
By choosing the absolute constant $c'$ small enough, we may assume that $\delta < \frac{1}{2}$. 
Then by Lemma \ref{lemma:perpdecrease} and the RIP of $\mathcal{A}$ we have
\begin{equation*}
	\big\Vert \utplushalfhat - \innerproduct{\vstar ,\vthat} \ustar  \big\Vert \le  \frac{\delta}{1-\delta} \Vert \vthat^{\perp}\Vert \le 2\delta  \Vert \vthat^{\perp}\Vert.
\end{equation*}
This implies 
\begin{equation*}
    \twonorm{ \utplushalfhat^\perp }
    \le 2\delta \twonorm{ \vthat^\perp }
\end{equation*}
as well as 
\begin{equation}\label{eq:stage2_vthatnorm1}
    \Big\vert \twonorm{\utplushalfhat^\parallel} - \twonorm{\vthat^\parallel} \Big\vert 
    = \Big\vert \twonorm{\utplushalfhat^\parallel} - \twonorm{\innerproduct{\vstar ,\vthat} \ustar} \Big\vert
    \le \big\Vert \utplushalfhat - \innerproduct{\vstar ,\vthat} \ustar  \big\Vert
    \le 2\delta \twonorm{ \vthat^\perp } \le 2 \delta,
\end{equation}
where in the last inequality we used that $\twonorm{ \vthat^\perp } \le \twonorm{ \vthat } =1 $.
In particular, the inequality in \eqref{eq:stage2_vthatnorm1} implies that
\begin{align*}
	\twonorm{\utplushalfhat^\parallel} 
	\ge  \twonorm{\vthat^\parallel} - 2\delta  
	\ge  \frac{3}{4} \, \twonorm{\vthat^\parallel},
\end{align*}
where the last inequality follows from \eqref{ineq:aux51}.
Hence, setting $\alpha =  \frac{9}{16}$ and $ \beta = 4\delta^2 $, Lemma \ref{lemma:convergence} and \eqref{ineq:aux51} yield \begin{align}\label{ineq:aux75}
\twonorm{\utplushat^\perp}^2 
& \le \frac{\beta  }{\alpha \twonorm{\vthat^\parallel}^2} \cdot \twonorm{\vthat^\perp}^2 
\le \frac{1}{9} \, \twonorm{\vthat^\perp}^2.
\end{align}
This shows \eqref{ineq:aux53}. 
Next, one can show by induction that for $t > \hat{t}$
\begin{align*}
\twonorm{\utplus^\perp}	\le \left( \frac{c'}{2\sqrt{2} \log n_2} \right) \twonorm{\vt^\perp}
\quad \text{and} \quad 
\twonorm{v_{t+1}^\perp}	\le \left( \frac{c'}{2\sqrt{2} \log n_2} \right) \twonorm{u_{t+1}^\perp}.
\end{align*}
The proof of these inequalities is analogous to the proof of \eqref{ineq:aux53} except that in \eqref{ineq:aux75} we can use the estimate $ \twonorm{\vt^\parallel}^2 \ge \frac{8}{9} $ due to $\twonorm{\vt^\perp}^2 \le \twonorm{ u_{t-1}^\perp }^2 \le \frac{1}{9}$ instead of the weaker estimate $ \twonorm{ \vt^\parallel} \ge \frac{c'}{ \log n_2 }$.
Finally, one can choose $c'$ small so that \eqref{ineq:aux71} is satisfied. 
\end{proof}

\subsection{Finishing the proof of Theorem~\ref{thm:main}}
We deduce from \eqref{ineq:induction1} in Proposition~\ref{prop:stage1} that Phase 1 is completed after 
\begin{equation}\label{inequ:taux}
	\hat{t} \lesssim \frac{ \log n_2}{\log \log n_2}
\end{equation}
iterations.
Next, one observes immediately by a direct calculation that $ \sin\left(  \angle(u_t,\ustar )\right) = \twonorm{\ut^\perp}  $ and $  \sin\left(  \angle(\vt,\vstar )\right) = \twonorm{\vt^\perp}   $.
Moreover, one obtains from inequalities in \eqref{ineq:aux71} of Proposition~\ref{prop:stage2} that after 
\begin{equation*}
	t - \hat{t} \lesssim  \frac{\log(1/\varepsilon)}{\log \log n_2}
\end{equation*}
iterations it holds that $ \max \left\{ \twonorm{\ut^\perp} ; \twonorm{\vt^\perp}  \right\}  \le \varepsilon $.
Together with \eqref{inequ:taux} this finishes the proof of Theorem~\ref{thm:main}.

\section{Numerical experiments}

We present a set of Monte Carlo simulations to compare the theoretical bound in Theorem~\ref{thm:main} to the empirical performance of ALS from random initialization. According to the assumptions of Theorem~\ref{thm:main}, the measurement matrices were generated as independent copies of a random matrix with i.i.d. standard Gaussian entries. Observations were obtained without noise. 
In the first experiment, we compare the performance of ALS methods respectively from random initialization and from spectral initialization. 
Figure~\ref{fig:pt} plots the phase transition of the reconstruction error in this experiment. 
We vary the matrix size from 8 to 256 while the oversampling factor $m/(n_1+n_2)$ is between 1 and 3. 
As shown in Figures~\ref{fig:pt_si} and \ref{fig:pt_ri}, ALS from spectral initialization has larger success regime so that the reconstruction is achieved from fewer observations. 
In these plots, we displayed the median of the normalized reconstruction error over 100 random trials. 
Figure~\ref{fig:pt_ri} shows that compared to ALS from spectral initialization, the phase transition for ALS from random initialization occurs at a higher oversampling factor. The amount of excess observations scales as a poly-log of the matrix size, which coincides with the result in Theorem~\ref{thm:main}. 

\begin{figure}
    \centering
    \begin{subfigure}[b]{0.49\textwidth}
        \centering
        \includegraphics[width=\textwidth]{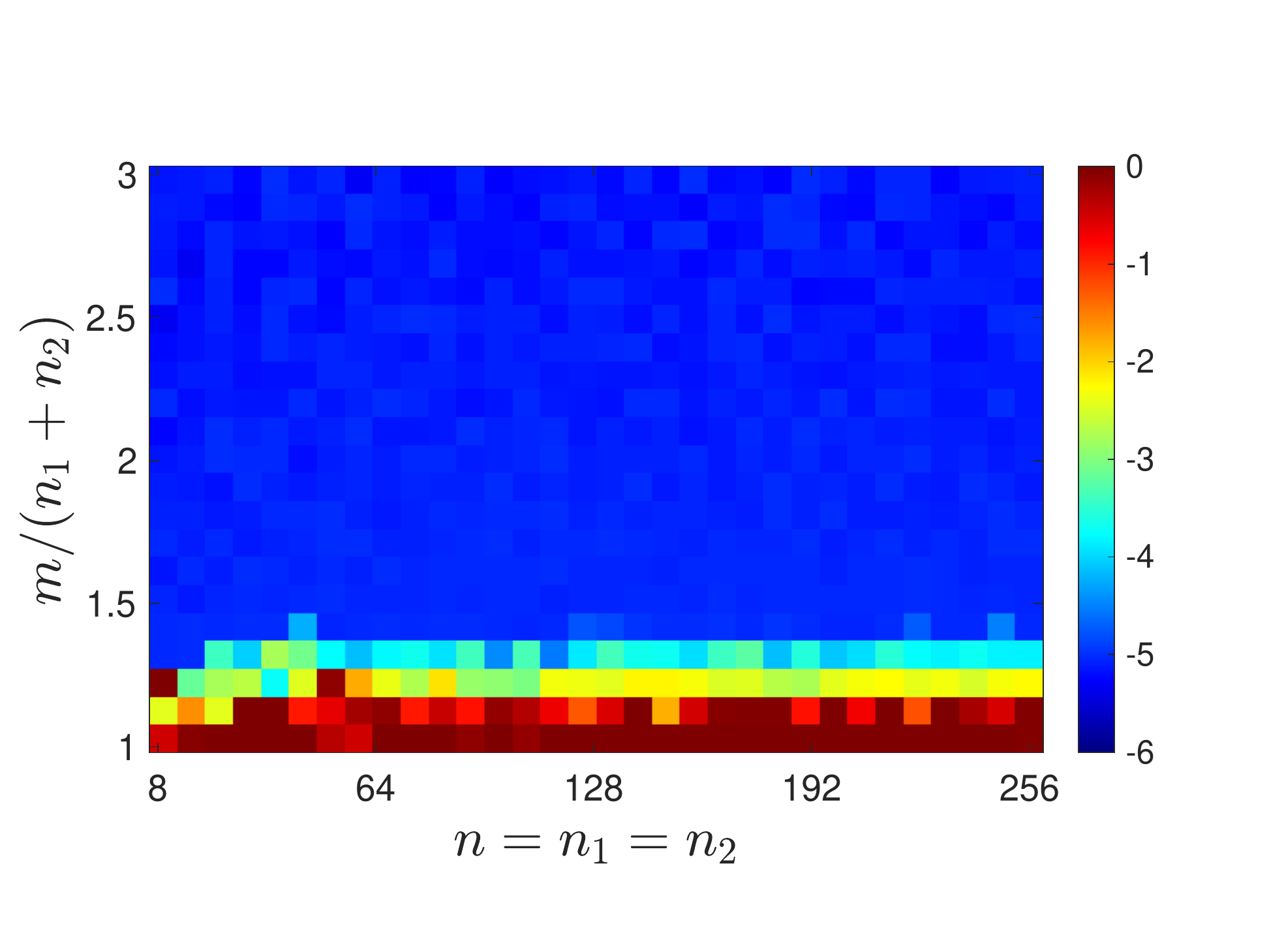}
        \caption{ALS from spectral initialization}
        \label{fig:pt_si}
    \end{subfigure}
    \hfill
    \begin{subfigure}[b]{0.49\textwidth}
        \centering
        \includegraphics[width=\textwidth]{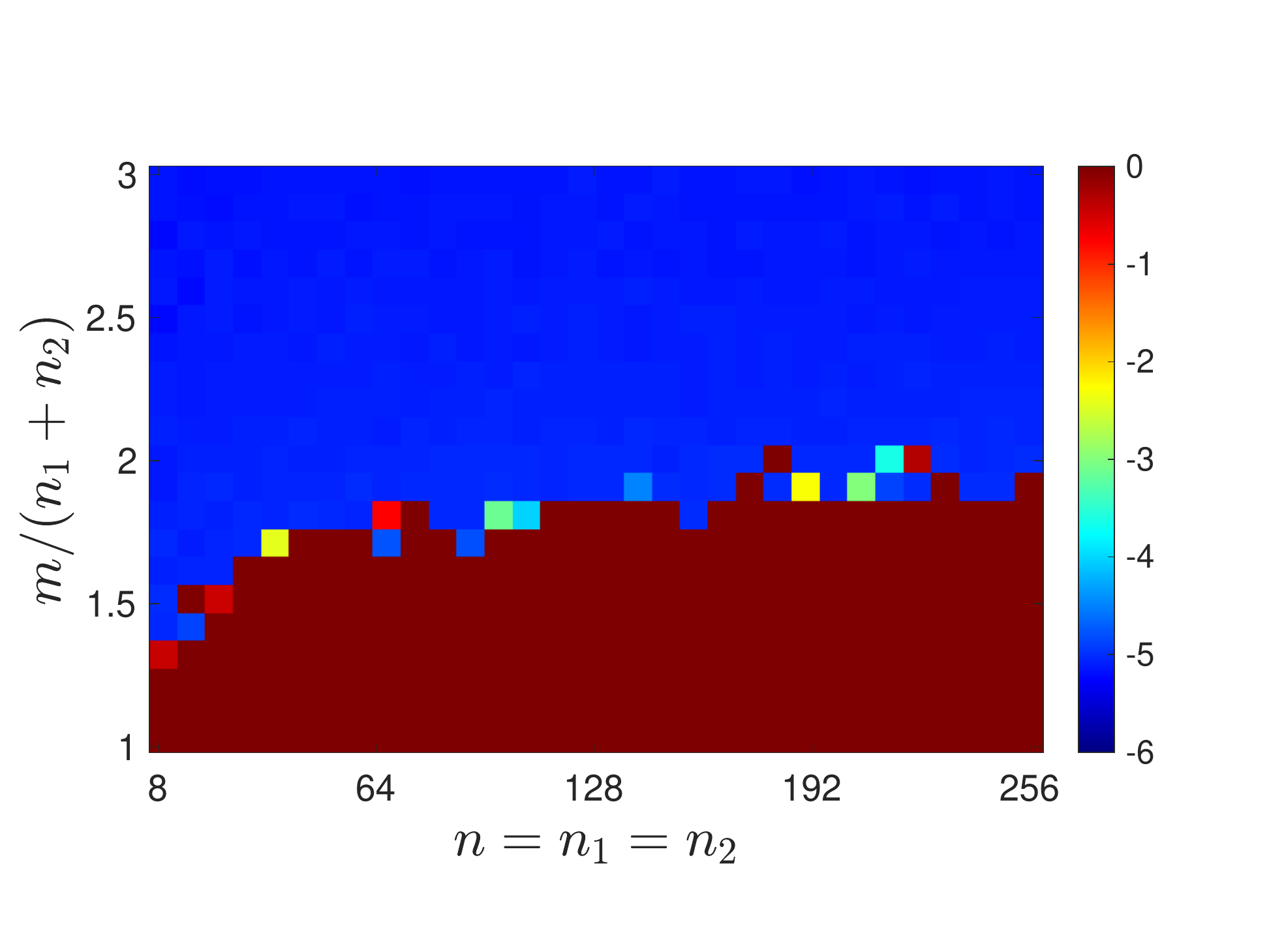}
        \caption{ALS from random initialization}
        \label{fig:pt_ri}
    \end{subfigure}
    \caption{Phase transition of reconstruction error}
    \label{fig:pt}
\end{figure}

Although the main result in Theorem~\ref{thm:main} is restricted to the rank-1 case, empirically, ALS from random initialization continues to work at a small oversampling factor when the rank of the unknown matrix becomes larger. 
We conducted the same experiment in Figure~\ref{fig:iter} in the rank-$r$ case, which is plotted in Figure~\ref{fig:iter_rankr}. 
One can observe that the same phase transition in Theorem~\ref{thm:main} occurs in the rank-5 case.  

\begin{figure}
    \centering
    \begin{subfigure}[b]{0.49\textwidth}
        \centering
        \includegraphics[width=\textwidth]{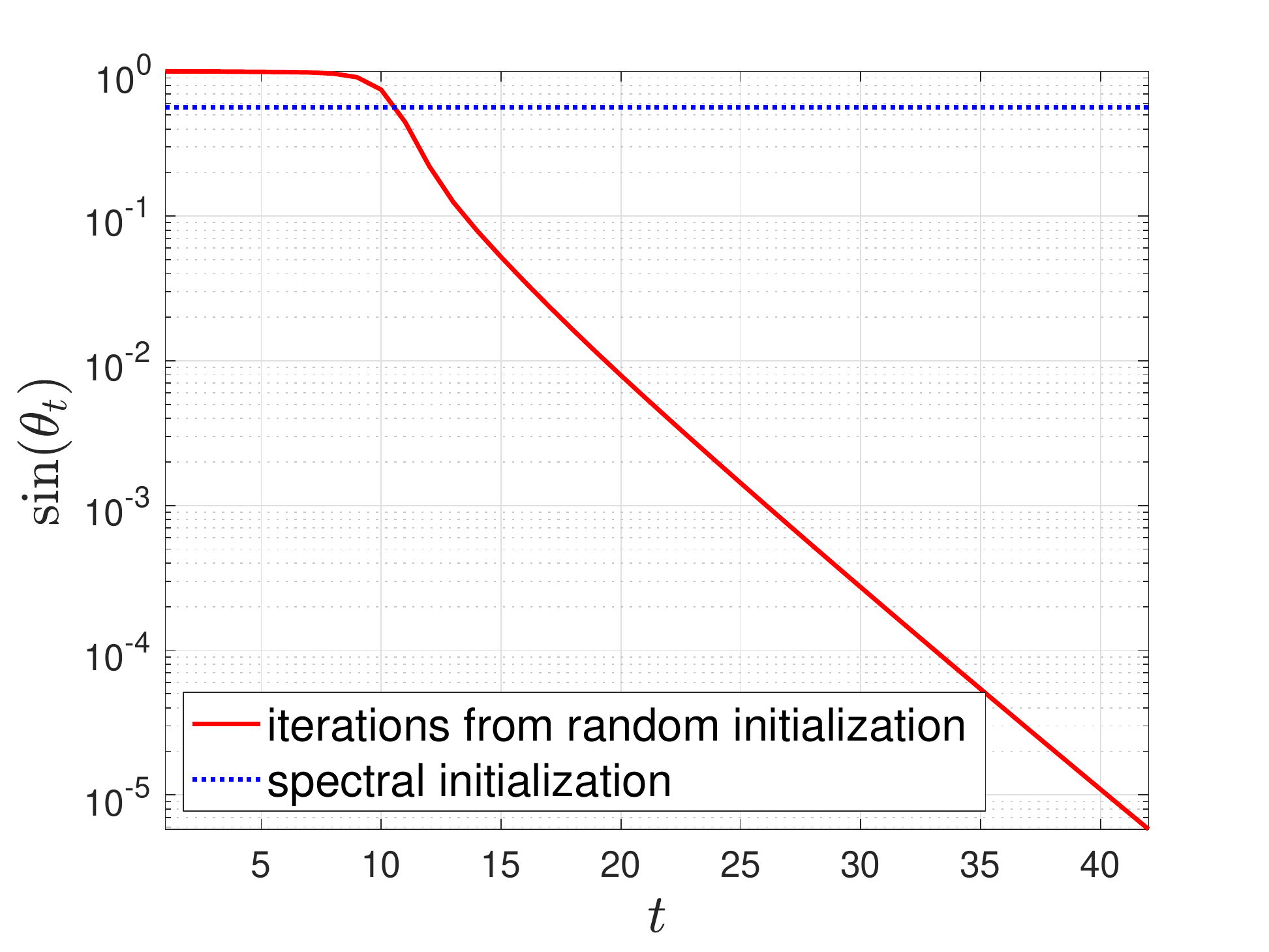}
        \caption{}
    \end{subfigure}
    \hfill
    \begin{subfigure}[b]{0.49\textwidth}
        \centering
        \includegraphics[width=\textwidth]{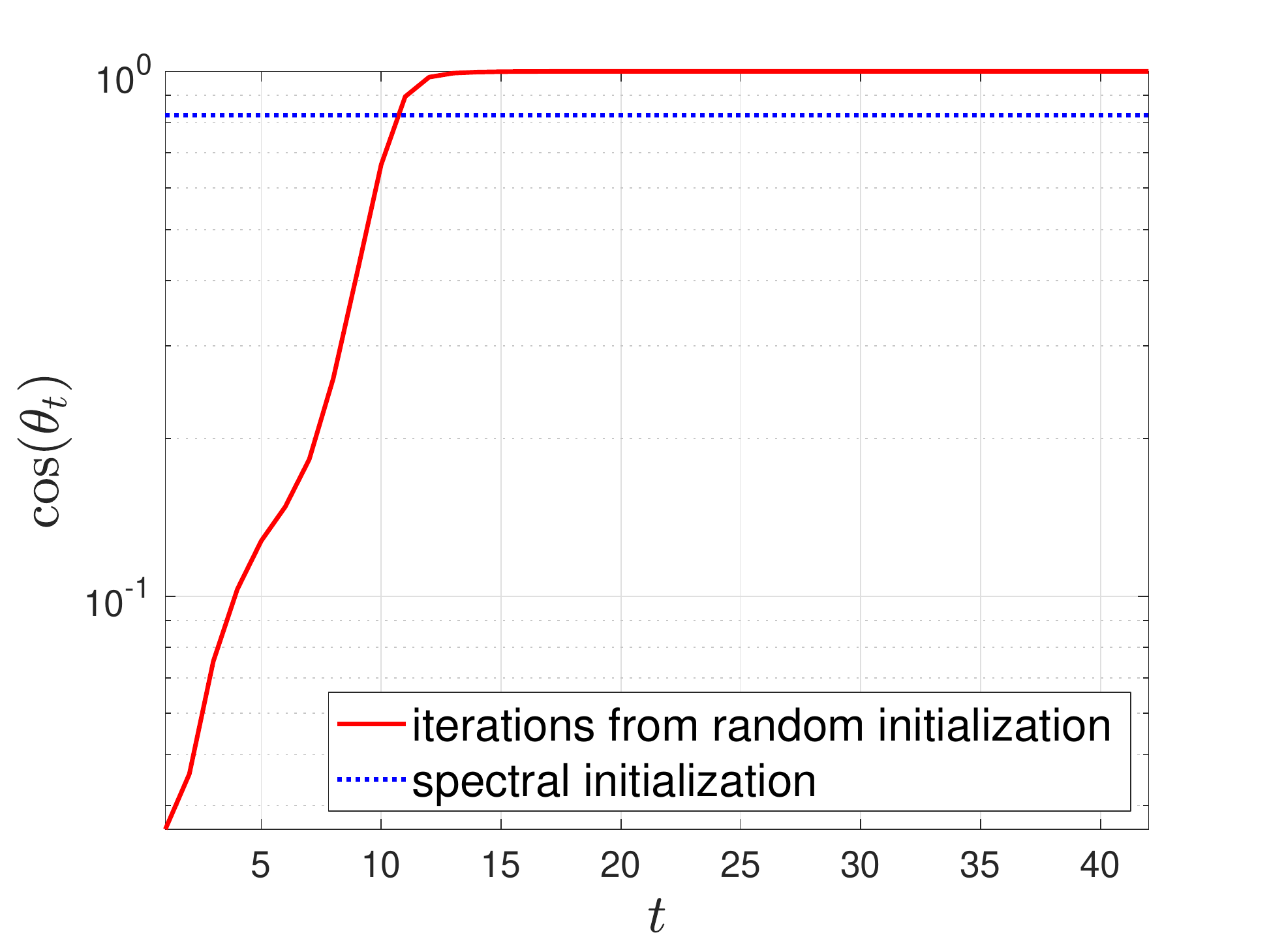}
        \caption{}
    \end{subfigure} \\
    \begin{subfigure}[b]{0.49\textwidth}
        \centering
        \includegraphics[width=\textwidth]{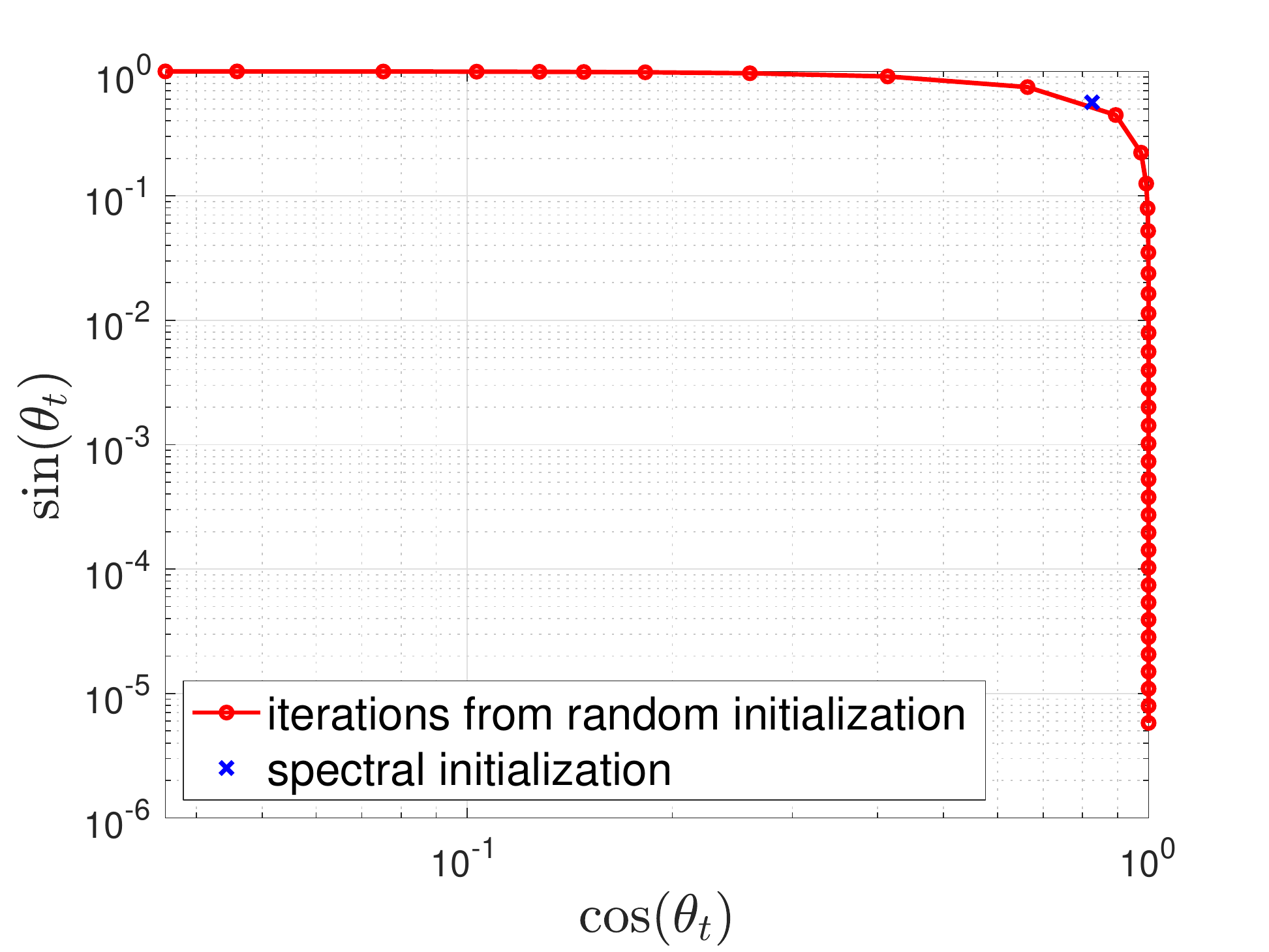}
        \caption{}
    \end{subfigure}
    \caption{Evolution of the estimation by randomly initialized ALS over iteration (rank-5 case): $n_1 = n_2 = 256$, $r = 5$, $m = 2 r (n_1+n_2-r)$. The principal angle between subspaces spanned by $\hat{U}$ and $U_t$ is denoted by $\theta_t$. (a) $\sin \theta_t$ vs $t$; (b) $\cos \theta_t$ vs $t$; (c) $\sin \theta_t$ vs $\cos \theta_t$.}
    \label{fig:iter_rankr}
\end{figure}

\section{Discussion}

We have shown that ALS from random initialization converges to the rank-one ground-truth matrix in the low-rank matrix sensing setting (with high probability). 
In our analysis, we observed that the trajectory of the iteration can be separated into two distinct phases: in the first one, the iterates converge from random initialization to a local neighborhood in $ O \left( \log n/ \log \log n  \right) $ iterations. In the second phase, the iterates converge linearly to the ground truth. 
This is aligned with our numerical experiments, where a sharp phase transition is visible. 

We expect that the convergence analysis in this paper will shed light on the convergence of ALS starting from random initialization in more general settings. 
For example, empirically, ALS from random initialization was shown to be successful if the ground truth has a rank higher than one. It would be interesting to see whether our analysis can be extended to this setting.
Moreover, it would be interesting to examine the scenario when the measurement matrices are more structured such as in the Matrix Completion problem.

Moreover, our result requires a sample size at least in the order of $ n \log^4 n $, whereas, for example, approaches based on convex relaxation such as nuclear-norm minimization only need in the order of $n$ samples. It would be interesting to examine whether it is possible to remove the additional $\log$-factors in our result.

\bibliographystyle{IEEEtran}
\bibliography{literature}

\begin{thebibliography}{10}
\providecommand{\url}[1]{#1}
\csname url@samestyle\endcsname
\providecommand{\newblock}{\relax}
\providecommand{\bibinfo}[2]{#2}
\providecommand{\BIBentrySTDinterwordspacing}{\spaceskip=0pt\relax}
\providecommand{\BIBentryALTinterwordstretchfactor}{4}
\providecommand{\BIBentryALTinterwordspacing}{\spaceskip=\fontdimen2\font plus
\BIBentryALTinterwordstretchfactor\fontdimen3\font minus
  \fontdimen4\font\relax}
\providecommand{\BIBforeignlanguage}[2]{{%
\expandafter\ifx\csname l@#1\endcsname\relax
\typeout{** WARNING: IEEEtran.bst: No hyphenation pattern has been}%
\typeout{** loaded for the language `#1'. Using the pattern for}%
\typeout{** the default language instead.}%
\else
\language=\csname l@#1\endcsname
\fi
#2}}
\providecommand{\BIBdecl}{\relax}
\BIBdecl

\bibitem{davenport2016overview}
M.~A. Davenport and J.~Romberg, ``An overview of low-rank matrix recovery from
  incomplete observations,'' \emph{IEEE Journal of Selected Topics in Signal
  Processing}, vol.~10, no.~4, pp. 608--622, 2016.

\bibitem{trefethen}
L.~N. Trefethen and D.~I. Bau, \emph{\BIBforeignlanguage{English}{Numerical
  linear algebra}}.\hskip 1em plus 0.5em minus 0.4em\relax Philadelphia, PA:
  SIAM, Society for Industrial {and} Applied Mathematics, 1997.

\bibitem{haldar2009rank}
J.~P. Haldar and D.~Hernando, ``Rank-constrained solutions to linear matrix
  equations using powerfactorization,'' \emph{IEEE Signal Process. Lett.},
  vol.~16, no.~7, pp. 584--587, 2009.

\bibitem{jain2013low}
P.~Jain, P.~Netrapalli, and S.~Sanghavi, ``Low-rank matrix completion using
  alternating minimization,'' in \emph{Proceedings of the forty-fifth annual
  ACM symposium on Theory of computing}, 2013, pp. 665--674.

\bibitem{hastie_matrixcompletion}
T.~Hastie, R.~Mazumder, J.~D. Lee, and R.~Zadeh,
  ``\BIBforeignlanguage{English}{Matrix completion and low-rank {SVD} via fast
  alternating least squares},'' \emph{\BIBforeignlanguage{English}{J. Mach.
  Learn. Res.}}, vol.~16, pp. 3367--3402, 2015.

\bibitem{sun_matrixcompletion}
R.~Sun and Z.-Q. Luo, ``\BIBforeignlanguage{English}{Guaranteed matrix
  completion via non-convex factorization},''
  \emph{\BIBforeignlanguage{English}{IEEE Trans. Inf. Theory}}, vol.~62,
  no.~11, pp. 6535--6579, 2016.

\bibitem{hardt_matrixcompletion1}
M.~Hardt, ``Understanding alternating minimization for matrix completion,'' in
  \emph{2014 IEEE 55th Annual Symposium on Foundations of Computer
  Science}.\hskip 1em plus 0.5em minus 0.4em\relax IEEE, 2014, pp. 651--660.

\bibitem{hardt_matrixcompletion2}
\BIBentryALTinterwordspacing
M.~Hardt and M.~Wootters, ``Fast matrix completion without the condition
  number,'' in \emph{Proceedings of The 27th Conference on Learning Theory},
  ser. Proceedings of Machine Learning Research, M.~F. Balcan, V.~Feldman, and
  C.~Szepesvári, Eds., vol.~35.\hskip 1em plus 0.5em minus 0.4em\relax
  Barcelona, Spain: PMLR, 13--15 Jun 2014, pp. 638--678. [Online]. Available:
  \url{https://proceedings.mlr.press/v35/hardt14a.html}
\BIBentrySTDinterwordspacing

\bibitem{bhojanapalli2016global}
S.~Bhojanapalli, B.~Neyshabur, and N.~Srebro, ``Global optimality of local
  search for low rank matrix recovery,'' \emph{Advances in Neural Information
  Processing Systems}, vol.~29, 2016.

\bibitem{chen2019gradient}
Y.~Chen, Y.~Chi, J.~Fan, and C.~Ma, ``Gradient descent with random
  initialization: Fast global convergence for nonconvex phase retrieval,''
  \emph{Math. Program.}, vol. 176, no.~1, pp. 5--37, 2019.

\bibitem{kech_krahmer}
M.~Kech and F.~Krahmer, ``\BIBforeignlanguage{English}{Optimal injectivity
  conditions for bilinear inverse problems with applications to identifiability
  of deconvolution problems},'' \emph{\BIBforeignlanguage{English}{SIAM J.
  Appl. Algebra Geom.}}, vol.~1, no.~1, pp. 20--37, 2017.

\bibitem{recht2010guaranteed}
B.~Recht, M.~Fazel, and P.~A. Parrilo, ``Guaranteed minimum-rank solutions of
  linear matrix equations via nuclear norm minimization,'' \emph{SIAM Rev.},
  vol.~52, no.~3, pp. 471--501, 2010.

\bibitem{candes_matrixcompletion}
E.~J. Cand{\`e}s and B.~Recht, ``\BIBforeignlanguage{English}{Exact matrix
  completion via convex optimization},''
  \emph{\BIBforeignlanguage{English}{Found. Comput. Math.}}, vol.~9, no.~6, pp.
  717--772, 2009.

\bibitem{candestao_matrixcompletion}
E.~J. Cand{\`e}s and T.~Tao, ``\BIBforeignlanguage{English}{The power of convex
  relaxation: near-optimal matrix completion},''
  \emph{\BIBforeignlanguage{English}{IEEE Trans. Inf. Theory}}, vol.~56, no.~5,
  pp. 2053--2080, 2010.

\bibitem{gross_matrixcompletion}
D.~Gross, ``\BIBforeignlanguage{English}{Recovering low-rank matrices from few
  coefficients in any basis},'' \emph{\BIBforeignlanguage{English}{IEEE Trans.
  Inf. Theory}}, vol.~57, no.~3, pp. 1548--1566, 2011.

\bibitem{candesphaselift1}
E.~J. Cand{\`e}s, T.~Strohmer, and V.~Voroninski,
  ``\BIBforeignlanguage{English}{Phaselift: exact and stable signal recovery
  from magnitude measurements via convex programming},''
  \emph{\BIBforeignlanguage{English}{Commun. Pure Appl. Math.}}, vol.~66,
  no.~8, pp. 1241--1274, 2013.

\bibitem{mahdiphaselift}
E.~J. Cand{\`e}s, X.~Li, and M.~Soltanolkotabi,
  ``\BIBforeignlanguage{English}{Phase retrieval from coded diffraction
  patterns},'' \emph{\BIBforeignlanguage{English}{Appl. Comput. Harmon.
  Anal.}}, vol.~39, no.~2, pp. 277--299, 2015.

\bibitem{gross_phasellift}
D.~Gross, F.~Krahmer, and R.~Kueng, ``\BIBforeignlanguage{English}{A partial
  derandomization of phaselift using spherical designs},''
  \emph{\BIBforeignlanguage{English}{J. Fourier Anal. Appl.}}, vol.~21, no.~2,
  pp. 229--266, 2015.

\bibitem{candes2011robust}
E.~J. Cand{\`e}s, X.~Li, Y.~Ma, and J.~Wright, ``Robust principal component
  analysis?'' \emph{Journal of the ACM (JACM)}, vol.~58, no.~3, pp. 1--37,
  2011.

\bibitem{ahmed_blinddeconv}
A.~Ahmed, B.~Recht, and J.~Romberg, ``\BIBforeignlanguage{English}{Blind
  deconvolution using convex programming},''
  \emph{\BIBforeignlanguage{English}{IEEE Trans. Inf. Theory}}, vol.~60, no.~3,
  pp. 1711--1732, 2014.

\bibitem{strohmer_demixing}
S.~Ling and T.~Strohmer, ``\BIBforeignlanguage{English}{Blind deconvolution
  meets blind demixing: algorithms and performance bounds},''
  \emph{\BIBforeignlanguage{English}{IEEE Trans. Inf. Theory}}, vol.~63, no.~7,
  pp. 4497--4520, 2017.

\bibitem{jung_demixing}
P.~Jung, F.~Krahmer, and D.~St{\"o}ger, ``\BIBforeignlanguage{English}{Blind
  demixing and deconvolution at near-optimal rate},''
  \emph{\BIBforeignlanguage{English}{IEEE Trans. Inf. Theory}}, vol.~64, no.~2,
  pp. 704--727, 2018.

\bibitem{fuchs_overview}
T.~Fuchs, D.~Gross, P.~Jung, F.~Krahmer, R.~Kueng, and D.~St{\"o}ger, ``Proof
  methods for robust low-rank matrix recovery,'' \emph{arXiv preprint
  arXiv:2106.04382}, 2021.

\bibitem{jain2010guaranteed}
P.~Jain, R.~Meka, and I.~Dhillon, ``Guaranteed rank minimization via singular
  value projection,'' \emph{Advances in Neural Information Processing Systems},
  vol.~23, 2010.

\bibitem{lee2010admira}
K.~Lee and Y.~Bresler, ``{ADMiRA}: Atomic decomposition for minimum rank
  approximation,'' \emph{IEEE Trans. Inf. Theory}, vol.~56, no.~9, pp.
  4402--4416, 2010.

\bibitem{massimo_irls}
M.~Fornasier, H.~Rauhut, and R.~Ward, ``\BIBforeignlanguage{English}{Low-rank
  matrix recovery via iteratively reweighted least squares minimization},''
  \emph{\BIBforeignlanguage{English}{SIAM J. Optim.}}, vol.~21, no.~4, pp.
  1614--1640, 2011.

\bibitem{mohan_irls}
K.~Mohan and M.~Fazel, ``\BIBforeignlanguage{English}{Iterative reweighted
  algorithms for matrix rank minimization},''
  \emph{\BIBforeignlanguage{English}{J. Mach. Learn. Res.}}, vol.~13, pp.
  3441--3473, 2012.

\bibitem{GD_wirtinger}
E.~J. Cand{\`e}s, X.~Li, and M.~Soltanolkotabi,
  ``\BIBforeignlanguage{English}{Phase retrieval via {Wirtinger} flow: theory
  and algorithms},'' \emph{\BIBforeignlanguage{English}{IEEE Trans. Inf.
  Theory}}, vol.~61, no.~4, pp. 1985--2007, 2015.

\bibitem{tu2016low}
S.~Tu, R.~Boczar, M.~Simchowitz, M.~Soltanolkotabi, and B.~Recht, ``Low-rank
  solutions of linear matrix equations via procrustes flow,'' in
  \emph{International Conference on Machine Learning}.\hskip 1em plus 0.5em
  minus 0.4em\relax PMLR, 2016, pp. 964--973.

\bibitem{chi2019nonconvex}
Y.~Chi, Y.~M. Lu, and Y.~Chen, ``Nonconvex optimization meets low-rank matrix
  factorization: An overview,'' \emph{IEEE Trans. Signal Process.}, vol.~67,
  no.~20, pp. 5239--5269, 2019.

\bibitem{GD_blinddeconv}
X.~Li, S.~Ling, T.~Strohmer, and K.~Wei, ``\BIBforeignlanguage{English}{Rapid,
  robust, and reliable blind deconvolution via nonconvex optimization},''
  \emph{\BIBforeignlanguage{English}{Appl. Comput. Harmon. Anal.}}, vol.~47,
  no.~3, pp. 893--934, 2019.

\bibitem{gd_blinddeconv2}
W.~Huang and P.~Hand, ``\BIBforeignlanguage{English}{Blind deconvolution by a
  steepest descent algorithm on a quotient manifold},''
  \emph{\BIBforeignlanguage{English}{SIAM J. Imaging Sci.}}, vol.~11, no.~4,
  pp. 2757--2785, 2018.

\bibitem{GD_blinddemixing}
S.~Ling and T.~Strohmer, ``\BIBforeignlanguage{English}{Regularized gradient
  descent: a non-convex recipe for fast joint blind deconvolution and
  demixing},'' \emph{\BIBforeignlanguage{English}{Inf. Inference}}, vol.~8,
  no.~1, pp. 1--49, 2019.

\bibitem{gd_phaseretrieval}
Y.~Chen and E.~J. Cand{\`e}s, ``\BIBforeignlanguage{English}{Solving random
  quadratic systems of equations is nearly as easy as solving linear
  systems},'' \emph{\BIBforeignlanguage{English}{Commun. Pure Appl. Math.}},
  vol.~70, no.~5, pp. 822--883, 2017.

\bibitem{phaseretrieval_landscape}
J.~Sun, Q.~Qu, and J.~Wright, ``\BIBforeignlanguage{English}{A geometric
  analysis of phase retrieval},'' \emph{\BIBforeignlanguage{English}{Found.
  Comput. Math.}}, vol.~18, no.~5, pp. 1131--1198, 2018.

\bibitem{matrixcompletion_landscape1}
R.~Ge, J.~D. Lee, and T.~Ma, ``Matrix completion has no spurious local
  minimum,'' \emph{Advances in Neural Information Processing Systems}, vol.~29,
  2016.

\bibitem{matrixcompletion_landscape2}
J.~Chen and X.~Li, ``\BIBforeignlanguage{English}{Model-free nonconvex matrix
  completion: local minima analysis and applications in memory-efficient kernel
  {PCA}},'' \emph{\BIBforeignlanguage{English}{J. Mach. Learn. Res.}}, vol.~20,
  p.~39, 2019, id/No 142.

\bibitem{saddle_avoidance}
J.~D. Lee, I.~Panageas, G.~Piliouras, M.~Simchowitz, M.~I. Jordan, and
  B.~Recht, ``\BIBforeignlanguage{English}{First-order methods almost always
  avoid strict saddle points},'' \emph{\BIBforeignlanguage{English}{Math.
  Program.}}, vol. 176, no. 1-2 (B), pp. 311--337, 2019.

\bibitem{du2017gradient}
S.~S. Du, C.~Jin, J.~D. Lee, M.~I. Jordan, A.~Singh, and B.~Poczos, ``Gradient
  descent can take exponential time to escape saddle points,'' \emph{Advances
  in Neural Information Processing Systems}, vol.~30, 2017.

\bibitem{ma_algorithmic}
Y.~Li, T.~Ma, and H.~Zhang, ``Algorithmic regularization in over-parameterized
  matrix sensing and neural networks with quadratic activations,'' in
  \emph{Conference On Learning Theory}.\hskip 1em plus 0.5em minus 0.4em\relax
  PMLR, 2018, pp. 2--47.

\bibitem{stoger2021small}
D.~St{\"o}ger and M.~Soltanolkotabi, ``Small random initialization is akin to
  spectral learning: Optimization and generalization guarantees for
  overparameterized low-rank matrix reconstruction,'' \emph{Advances in Neural
  Information Processing Systems}, vol.~34, 2021.

\bibitem{ye_algorithmic}
T.~Ye and S.~S. Du, ``Global convergence of gradient descent for asymmetric
  low-rank matrix factorization,'' \emph{Advances in Neural Information
  Processing Systems}, vol.~34, 2021.

\bibitem{jiang_algorithmic}
L.~Jiang, Y.~Chen, and L.~Ding, ``Algorithmic regularization in model-free
  overparametrized asymmetric matrix factorization,'' \emph{arXiv preprint
  arXiv:2203.02839}, 2022.

\bibitem{kroonenberg1980principal}
P.~M. Kroonenberg and J.~De~Leeuw, ``Principal component analysis of three-mode
  data by means of alternating least squares algorithms,''
  \emph{Psychometrika}, vol.~45, no.~1, pp. 69--97, 1980.

\bibitem{o2007alternating}
J.~A. O'Sullivan and J.~Benac, ``Alternating minimization algorithms for
  transmission tomography,'' \emph{IEEE Transactions on Medical Imaging},
  vol.~26, no.~3, pp. 283--297, 2007.

\bibitem{zhao_matrixCompletion}
\BIBentryALTinterwordspacing
T.~Zhao, Z.~Wang, and H.~Liu, ``A nonconvex optimization framework for low rank
  matrix estimation,'' in \emph{Advances in Neural Information Processing
  Systems}, C.~Cortes, N.~Lawrence, D.~Lee, M.~Sugiyama, and R.~Garnett, Eds.,
  vol.~28.\hskip 1em plus 0.5em minus 0.4em\relax Curran Associates, Inc.,
  2015. [Online]. Available:
  \url{https://proceedings.neurips.cc/paper/2015/file/39461a19e9eddfb385ea76b26521ea48-Paper.pdf}
\BIBentrySTDinterwordspacing

\bibitem{kiryung_powerfactorization}
K.~Lee, Y.~Wu, and Y.~Bresler, ``\BIBforeignlanguage{English}{Near-optimal
  compressed sensing of a class of sparse low-rank matrices via sparse power
  factorization},'' \emph{\BIBforeignlanguage{English}{IEEE Trans. Inf.
  Theory}}, vol.~64, no.~3, pp. 1666--1698, 2018.

\bibitem{deconv_powerfactorization}
K.~Lee, Y.~Li, M.~Junge, and Y.~Bresler, ``\BIBforeignlanguage{English}{Blind
  recovery of sparse signals from subsampled convolution},''
  \emph{\BIBforeignlanguage{English}{IEEE Trans. Inf. Theory}}, vol.~63, no.~2,
  pp. 802--821, 2017.

\bibitem{jakob_powerfactorization}
J.~Geppert, F.~Krahmer, and D.~St{\"o}ger,
  ``\BIBforeignlanguage{English}{Sparse power factorization: balancing
  peakiness and sample complexity},'' \emph{\BIBforeignlanguage{English}{Adv.
  Comput. Math.}}, vol.~45, no.~3, pp. 1711--1728, 2019.

\bibitem{fienup1982phase}
J.~R. Fienup, ``Phase retrieval algorithms: a comparison,'' \emph{Applied
  optics}, vol.~21, no.~15, pp. 2758--2769, 1982.

\bibitem{gerchberg1972practical}
R.~W. Gerchberg, ``A practical algorithm for the determination of phase from
  image and diffraction plane pictures,'' \emph{Optik}, vol.~35, pp. 237--246,
  1972.

\bibitem{phase_netrapalli}
P.~Netrapalli, P.~Jain, and S.~Sanghavi, ``Phase retrieval using alternating
  minimization,'' \emph{IEEE Trans. Signal Process.}, vol.~63, no.~18, pp.
  4814--4826, 2015.

\bibitem{phaseretrieval_waldspurger}
I.~Waldspurger, ``\BIBforeignlanguage{English}{Phase retrieval with random
  {Gaussian} sensing vectors by alternating projections},''
  \emph{\BIBforeignlanguage{English}{IEEE Trans. Inf. Theory}}, vol.~64, no.~5,
  pp. 3301--3312, 2018.

\bibitem{phaseretrieval_zhang}
T.~Zhang, ``\BIBforeignlanguage{English}{Phase retrieval by alternating
  minimization with random initialization},''
  \emph{\BIBforeignlanguage{English}{IEEE Trans. Inf. Theory}}, vol.~66, no.~7,
  pp. 4563--4573, 2020.

\bibitem{zhong_leaveoneout}
Y.~Zhong and N.~Boumal, ``\BIBforeignlanguage{English}{Near-optimal bounds for
  phase synchronization},'' \emph{\BIBforeignlanguage{English}{SIAM J.
  Optim.}}, vol.~28, no.~2, pp. 989--1016, 2018.

\bibitem{chen_leaveoneout}
Y.~Chen, Y.~Chi, J.~Fan, C.~Ma, and Y.~Yan,
  ``\BIBforeignlanguage{English}{Noisy matrix completion: understanding
  statistical guarantees for convex relaxation via nonconvex optimization},''
  \emph{\BIBforeignlanguage{English}{SIAM J. Optim.}}, vol.~30, no.~4, pp.
  3098--3121, 2020.

\bibitem{li_leaveoneout}
Y.~Li, C.~Ma, Y.~Chen, and Y.~Chi, ``\BIBforeignlanguage{English}{Nonconvex
  matrix factorization from rank-one measurements},''
  \emph{\BIBforeignlanguage{English}{IEEE Trans. Inf. Theory}}, vol.~67, no.~3,
  pp. 1928--1950, 2021.

\bibitem{ding_leaveoneout}
L.~Ding and Y.~Chen, ``\BIBforeignlanguage{English}{Leave-one-out approach for
  matrix completion: primal and dual analysis},''
  \emph{\BIBforeignlanguage{English}{IEEE Trans. Inf. Theory}}, vol.~66,
  no.~11, pp. 7274--7301, 2020.

\bibitem{candes2011tight}
E.~J. Candes and Y.~Plan, ``Tight oracle inequalities for low-rank matrix
  recovery from a minimal number of noisy random measurements,'' \emph{IEEE
  Transactions on Information Theory}, vol.~57, no.~4, pp. 2342--2359, 2011.

\bibitem{rauhut_foucart}
S.~Foucart and H.~Rauhut, \emph{\BIBforeignlanguage{English}{A mathematical
  introduction to compressive sensing}}, ser. Appl. Numer. Harmon. Anal.\hskip
  1em plus 0.5em minus 0.4em\relax New York, NY: Birkh{\"a}user/Springer, 2013.

\bibitem{vershynin2018high}
R.~Vershynin, \emph{High-dimensional probability: An introduction with
  applications in data science}.\hskip 1em plus 0.5em minus 0.4em\relax
  Cambridge university press, 2018, vol.~47.

\end{thebibliography}

\appendix

\section{Proofs of concentration inequalities}

\subsection{Proof of Lemma \ref{lemma:RIPlemma}}\label{app:rip_consequences}

The inequality in \eqref{eq:rop} is well known (see, e.g., \cite[Exercise 6.24]{rauhut_foucart}). 
In fact, since we assumed the RIP to hold for all matrices of rank at most $4$ in \eqref{eq:rip}, we even obtain the stronger statement that
\begin{equation}\label{eq:rop1}
\big\vert \innerproduct{ \mathcal{A} \left(Z_1 \right), \mathcal{A} \left(Z_2 \right) } -\froinnerproduct{Z_1,Z_2}	\big\vert \le  \delta \fronorm{ Z_1  } \cdot \fronorm{ Z_2}.   
\end{equation}
for all matrices $Z_1$ and $Z_2$ of rank at most $2$ (see \cite[Section 6]{rauhut_foucart}. 

We are going to derive the other inequalities in \eqref{eq:res_rip1}, \eqref{eq:res_rip2}, and \eqref{eq:res_rip3} from \eqref{eq:rop1}. For that, we note first that 
\[
\mathcal{O}^* \mathcal{D} = \mathcal{P_O} \mathcal{A}^* \mathcal{A} (\mathrm{Id} - \mathcal{P_O}).
\]
Then there exist $\hat{x} \in \mathbb{R}^{n_1}$ and $ \hat{y} \in \mathbb{R}^{n_2}$ with $ \twonorm{\hat{x}}=\twonorm{ \hat{y} }=1 $ such that
\begin{equation*}
\big\Vert \mathcal{O}^* \mathcal{D} \left(uv^\top\right)  \big\Vert = \big\vert \froinnerproduct{\hat{x}  \hat{y}^\top, \mathcal{O}^* \mathcal{D} \left(uv^\top\right)} \big\vert. 
\end{equation*}
Then it follows that the left-hand side of \eqref{eq:res_rip1} is upper-bounded by
\begin{align*}
\big\Vert \mathcal{O}^* \mathcal{D} \left( uv^\top \right)  \big\Vert 
&=\big\vert \froinnerproduct{ \hat{x}  \hat{y}^\top, \mathcal{O}^* \mathcal{D} \left(uv^\top\right)} \big\vert \\
&=  \big\vert \froinnerproduct{ \hat{x}  \hat{y}^\top, \mathcal{P_O} \mathcal{A}^* \mathcal{A} (\mathrm{Id} - \mathcal{P_O}) \left(uv^\top\right)} \big\vert \\
&=  \big\vert \innerproduct{\mathcal{A} \mathcal{P_O} \left( \hat{x}  \hat{y}^\top \right), \mathcal{A} (\mathrm{Id} - \mathcal{P_O}) \left(uv^\top\right)} \big\vert \\
&\overset{\mathrm{(a)}}{\le} \delta \fronorm{\mathcal{P_O} ( \hat{x}  \hat{y}^\top)} \cdot \fronorm{(\mathrm{Id} - \mathcal{P_O})(uv^\top)} \\
&\le  \delta \fronorm{ \hat{x}  \hat{y}^\top} \cdot \fronorm{uv^\top} \\
&= \delta \twonorm{u} \cdot \twonorm{v},
\end{align*}

where (a) is due to \eqref{eq:rop1} and the fact that $\mathcal{P_O} (\hat{x} \hat{y}^\top $ and $(\mathrm{Id} - \mathcal{P_O}) \left(uv^\top\right) $ have rank at most $2$ each.
This proves inequality \eqref{eq:res_rip1}.
Inequality \eqref{eq:res_rip2} can be derived in an analogous way.

In order to show inequality \eqref{eq:res_rip3}, we again note there is $\tilde{x} \in \mathbb{R}^{n_1}$ and $\tilde{y} \in \mathbb{R}^{n_2}$ with $ \twonorm{\tilde{x}}=\twonorm{\tilde{y}} =1 $ such that
\begin{align*}
\Vert \left( \mathcal{O}^* \mathcal{O} - \mathcal{P_O} \right) \left(uv^\top \right)  \Vert 
&= \big\vert \froinnerproduct{ \tilde{x} \tilde{y}^\top, \left( \mathcal{O}^* \mathcal{O} - \mathcal{P_O} \right) \left(uv^\top \right)    } \big\vert
\end{align*}
holds.
From $\mathcal{O} = \mathcal{A} \mathcal{P_O}$ it follows that
\begin{align*}
 \Vert \left( \mathcal{O}^* \mathcal{O} - \mathcal{P_O} \right) \left(uv^\top \right)  \Vert 
 = \big\vert \innerproduct{ \mathcal{A} \left( \mathcal{P}_O  (\tilde{x} \tilde{y}^\top)  \right) , \mathcal{A} \left( \mathcal{P_O} \left(uv^\top \right) \right) }-  \froinnerproduct{ \mathcal{P}_O  ( \tilde{x} \tilde{y}^\top) ,  \mathcal{P_O} \left(uv^\top \right)  } \big\vert.
\end{align*}
Then it follows from \eqref{eq:rop1} that
\begin{align*}
 \Vert \left( \mathcal{O}^* \mathcal{O} - \mathcal{P_O} \right) \left(uv^\top \right)  \Vert 
 &\le \delta \fronorm{ \mathcal{P}_O  (\tilde{x} \tilde{y}^\top)  } \cdot \fronorm{ \mathcal{P}_O  (uv^\top)  }\\
 &\le \delta \twonorm{u} \cdot \twonorm{v}.
\end{align*}
This finishes the proof.

\subsection{Proof of Lemma \ref{lemma:randombound1}}\label{app:concentration1}
Note that the first entry of the vector $\sum_{i=1}^{m} (A_i)_{1,1}  O_i e_1 \in \mathbb{R}^{n_1}$  vanishes. Conditioned on $\{ (A_i)_{1,1} \}^m_{i=1} $, all other entries are i.i.d. random variables with distribution $ \mathcal{N} (  0, \sum_{i=1}^{m} \vert (A_i)_{1,1}  \vert^2 ) $. 
In particular, this implies that conditioned on $ \{ (A_i)_{1,1} \}^m_{i=1}  $ with probability at least $1-\mathcal{O} \left( \exp \left(-cn_1\right) \right)$ we have that
\begin{align}\label{ineq:help2}
\Big\Vert \sum_{i=1}^{m} \left(A_i\right)_{1,1}  O_i e_1 \Big\Vert 
\le 2 \sqrt{n_1 \sum_{i=1}^{m} \left(A_i\right)_{1,1}^2 }.
\end{align}
This is the standard concentration of the norm of a Gaussian vector (see, e.g. \cite[Theorem 3.1.1]{vershynin2018high}).
Similarly, it holds with probability at least $  1- \mathcal{O} \left( \exp \left(-cm\right) \right)$  that
\begin{equation}\label{ineq:help3}
\sum_{i=1}^{m} \vert   \left(A_i\right)_{1,1}  \vert^2 \le 2 m.
\end{equation}
Inserting inequality \eqref{ineq:help3} into inequality \eqref{ineq:help2} provides the first assertion in Lemma \ref{lemma:randombound1}.
The second assertion can be obtained analogously.

\subsection{Proof of Lemma \ref{lemma:indepbound1}}\label{app:concentration2}
We prove only the first assertion. The proof for the second assertion is analogous. 
We first note that by the concentration of the norm of Gaussian vector (e.g., \cite[Theorem 3.1.1]{vershynin2018high}), it holds with probability at least $1 - \mathcal{O} \left( \exp \left(-cm\right) \right)$ that
\begin{equation}\label{ineqe:conditionedA}
\sum_{i=1}^{m} \vert   \left(A_i\right)_{1,1}  \vert^2 \le 2m.
\end{equation}
In the following we will proceed conditioned on this event. 
Since by definition the first entry of $\vtaux^{\perp}$ vanishes, only the first entry of $O_i \vtaux^{\perp} $ is non-zero due to the structure of the matrix $O_i$.
In particular, we have that
\begin{equation*}
O_i \vtaux^{\perp} = \innerproduct{  O^\top_i e_1,   \vtaux^{\perp}  }    e_1.
\end{equation*}
This implies that
\begin{align*}
 \frac{1}{m} \Big\Vert  \left[  \sum_{i=1}^{m} \left(A_i\right)_{1,1}  O_i    \right] \vtaux^{\perp} \Big\Vert
 = \frac{1}{m} \Big\Vert  \sum_{i=1}^{m}  \innerproduct{  O^\top_i e_1 ,  \vtaux^{\perp}}  \left(A_i\right)_{1,1} e_1 \Big\Vert
 = \frac{1}{m} \Big\vert  \sum_{i=1}^{m}  \innerproduct{  O^\top_i e_1 ,  \vtaux^{\perp}}  \left(A_i\right)_{1,1}  \Big\vert.
\end{align*}
We observe that $\vtaux $ and $ \left(A_i \right)_{1,1}  $ are independent of  $ O_i$ for all $i \in \left[m\right]$ due to their definitions. 
Hence, conditioned on $ \{ (A_i)_{1,1} \}^m_{i=1}  $ and $\vtaux$ it holds that 
\begin{equation*}
 \innerproduct{  O^\top_i e_1 ,  \vtaux^{\perp}}  \left(A_i\right)_{1,1}  
  \sim \mathcal{N} \left(0, \big\vert \left(A_i \right)_{1,1} \big\vert \Vert \vtaux^{\perp}  \Vert   \right), \quad \text{for all } i \in [m]
\end{equation*}
and, hence, 
\[
\frac{1}{m} \sum_{i=1}^{m}  \innerproduct{  O^\top_i e_1 ,  \vtaux^{\perp}}  \left(A_i\right)_{1,1} 
\sim \mathcal{N} \left(0, \frac{1}{m} \sqrt{ \sum_{i=1}^m \left(A_i \right)_{1,1}^2} \Vert \vtaux^{\perp}  \Vert   \right).
\]
In particular, conditioned on $ \{ (A_i)_{1,1} \}^m_{i=1}$ and $\vtaux$ we obtain by a union bound that with probability $1-\eeta^{-1}$ it holds for all $t \in [T]$ simultaneously that  
\begin{align*}
\frac{1}{m} \Big\vert  \sum_{i=1}^{m}  \innerproduct{  O^\top_i e_1 ,  \vtaux^{\perp}}  \left(A_i\right)_{1,1}  \Big\vert 
& \lesssim  
\frac{\sqrt{\log T + \log \eeta}}{m} 
\cdot \sqrt{ \sum_{i=1}^{m} \left( \left(A_i\right)_{1,1} \right)^2 } 
\cdot \Vert \vtaux^{\perp} \Vert. 
\end{align*}
By inserting \eqref{ineqe:conditionedA} into the above inequality and by integrating over all events $ \{ (A_i)_{1,1} \}^m_{i=1}$, which satisfy \eqref{ineqe:conditionedA}, the first assertion in Lemma \ref{lemma:indepbound1} is obtained.
The second assertion in Lemma \ref{lemma:indepbound1} is obtained analogously.

\subsection{Proof of Lemma \ref{lemma:indepbound2}}\label{app:concentration3}
We note that $ \left\{  \froinnerproduct{D_i, \utplushalfaux^{\perp} \left( \vtaux^{\perp} \right)^\top } \vtaux^{\perp} \right\}_{i=1}^m   $ is independent from $ \left\{ O_i^\top e_1\right\}^{m}_{i=1} $. 
This implies that conditioned on $ \left\{  \froinnerproduct{D_i, \utplushalfaux^{\perp} \left( \vtaux^{\perp} \right)^\top } \vtaux^{\perp} \right\}_{i=1}^m   $ we have 
\begin{equation*}
    \sum_{i=1}^m \innerproduct{ O_i^\top e_i, \vtaux } \froinnerproduct{ D_i, \utplushalfaux^{\perp} (\vtaux^\perp)^\top } 
    \sim  \sqrt{ \sum_{i=1}^m \froinnerproduct{  D_i   ,  \utplushalfaux^{\perp} \left( \vtaux^{\perp} \right)^\top  }^2  \twonorm{\vtaux^\perp}^2       }  \cdot \mathcal{N} \left(0,1\right)
\end{equation*}
Hence, we obtain that conditioned on $ \left\{  \froinnerproduct{D_i, \utplushalfaux^{\perp} \left( \vtaux^{\perp} \right)^\top }  \vtaux^{\perp} \right\}_{i=1}^m$ with probability $1-\eeta^{-1}$ it holds for all $t \in [T] $ simultaneously that
\begin{align*}
\frac{1}{m} \big\vert  \sum_{i=1}^m \froinnerproduct{ O_i^\top e_i, \vtaux } \froinnerproduct{ D_i, \utplushalfaux^{\perp} (\vtaux^\perp)^\top  } \big\vert    
& \lesssim \frac{\sqrt{\log T + \log \eeta}}{m}  \cdot \sqrt{ \sum_{i=1}^m \froinnerproduct{  D_i   ,  \utplushalfaux^{\perp} \left( \vtaux^{\perp} \right)^\top  }^2  \twonorm{\vtaux^\perp}^2       }\\
&= \sqrt{\frac{\log T + \log \eeta}{m}} \cdot \twonorm{\vtaux^\perp} \cdot \sqrt{ \frac{1}{m} \sum_{i=1}^m \froinnerproduct{  A_i   ,  \utplushalfaux^{\perp} \left( \vtaux^{\perp} \right)^\top  }^2  } \\
&= \sqrt{\frac{\log T + \log \eeta}{m}} \cdot  \twonorm{\vtaux^\perp} \cdot \Big\Vert \mathcal{A} \left( \utplushalfaux^{\perp} \left( \vtaux^{\perp}\right)^\top \right) \Big\Vert \\  
&\le \sqrt{\frac{\log T + \log \eeta}{m}} \cdot \Big\Vert \mathcal{A} \left( \utplushalfaux^{\perp} \left( \vtaux^{\perp}\right)^\top \right) \Big\Vert. 
\end{align*}
This finishes the proof of the first assertion.
The second assertion is obtained analogously.

\subsection{Proof of Lemma \ref{lemma:stochasticbound}}\label{app:concentration4}
Recall without loss of generality that we assumed $\ustar = e_1$ and $\vstar = e_1$. 
This implies that we have
\begin{align*}
	\left[ \left( \mathcal{A}^* \mathcal{A} \right) \ustar  \vstar^\top  \right] \vtaux
	= \frac{1}{m} \left( \sum_{i=1}^{m} A_i  \froinnerproduct{ A_i  , \ustar \vstar^\top} \right) \vtaux 
	= \frac{1}{m} \left( \sum_{i=1}^{m} A_i  \left(A_i\right)_{1,1} \right) \vtaux
\end{align*}
and
\begin{align*}
	\left[ \left( \tilde{ \mathcal{A}}^* \tilde{\mathcal{A}} \right) \ustar \vstar^\top  \right] \vtaux 
	= \frac{1}{m} \left( \sum_{i=1}^{m} \tilde{A}_i  \froinnerproduct{ \tilde{A}_i  , \ustar  \vstar^\top} \right) \vtaux 
	= \frac{1}{m} \left( \sum_{i=1}^{m} \tilde{A}_i  \left(A_i\right)_{1,1} \right) \vtaux.
\end{align*}
Then it follows that 
\begin{align*}
	\left[  \left(    \mathcal{A}^* \mathcal{A}  - \tilde{ \mathcal{A}}^* \tilde{\mathcal{A}}   \right) \left(    \ustar  \vstar^\top \right)   \right] \vtaux &= \frac{1}{m}   \left[  \sum_{i=1}^{m} \left(A_i\right)_{1,1} \left( A_i - \tilde{A}_i \right)  \right] \vtaux.
\end{align*}
In order to proceed recall that we have decomposition $A_i= D_i + O_i$ and $\tilde{A}_i= D_i + \tilde{O}_i $ for all $ i \in \left[m\right] $. 
This implies that $A_i - \tilde{A}_i = O_i - \tilde{O}_i$. 
Hence, we obtain that
\begin{equation}\label{ineq:intern8}
	\begin{split}
		&\Big\Vert  \left[  \left(    \mathcal{A}^* \mathcal{A}  - \tilde{ \mathcal{A}}^* \tilde{\mathcal{A}}   \right) \left(    \ustar  \vstar^\top  \right)   \right] \vtaux \Big\Vert \\
		=& \frac{1}{m} \Big\Vert  \left[  \sum_{i=1}^{m} \left(A_i\right)_{1,1} \left( O_i - \tilde{O}_i \right)  \right] \vtaux \Big\Vert\\
		\le &  \frac{1}{m} \Big\Vert  \left[  \sum_{i=1}^{m} \left(A_i\right)_{1,1}  O_i    \right]\vtaux \Big\Vert +   \frac{1}{m} \Big\Vert  \left[  \sum_{i=1}^{m} \left(A_i\right)_{1,1}  \tilde{O}_i   \right]\vtaux \Big\Vert\\
		\le & \underset{=:(a)}{\underbrace{\frac{1}{m} \Big\Vert  \left[  \sum_{i=1}^{m} \left(A_i\right)_{1,1}  O_i    \right] \vtaux \Big\Vert} }+ \underset{=:(b)}{ \underbrace{ \frac{1}{m} \Big\Vert  \left[  \sum_{i=1}^{m} \left(A_i\right)_{1,1}  \tilde{O}_i    \right] \left(\vt -\vtaux \right) \Big\Vert} }  +   \underset{=:(c)}{\underbrace{\frac{1}{m} \Big\Vert  \left[  \sum_{i=1}^{m} \left(A_i\right)_{1,1}  \tilde{O}_i   \right]v_t \Big\Vert}}.
	\end{split}
\end{equation}
We estimate the three summands in the right-hand side of \eqref{ineq:intern8} individually.\\

\noindent\textbf{Estimating $(a)$: }
In order to upper-bound the first summand (a) we note that by the triangle inequality it holds that
\begin{equation}\label{ineq:intern7}
\frac{1}{m} \Big\Vert  \left[  \sum_{i=1}^{m} \left(A_i\right)_{1,1}  O_i    \right] \vtaux \Big\Vert \le  \frac{1}{m} \Big\Vert  \left[  \sum_{i=1}^{m} \left(A_i\right)_{1,1}  O_i    \right] \vtaux^{\perp} \Big\Vert +  \frac{1}{m} \Big\Vert  \left[  \sum_{i=1}^{m} \left(A_i\right)_{1,1}  O_i    \right] \vtaux^{\parallel} \Big\Vert.
\end{equation}
Then \eqref{ineq:randombound1a} and \eqref{ineq:indepbound1a} respectively imply that 
\begin{align*}
\frac{1}{m} \Big\Vert  \left[  \sum_{i=1}^{m} \left(A_i\right)_{1,1}  O_i    \right] \vtaux^{\parallel} \Big\Vert &= \frac{1}{m} \Big\Vert  \left[  \sum_{i=1}^{m} \left(A_i\right)_{1,1}  O_i    \right] e_1 \Big\Vert \cdot \Vert \vtaux^{\parallel}  \Vert \le 4 \sqrt{\frac{n_1}{m}} \Vert  \vtaux^{\parallel} \Vert
\intertext{and}
\frac{1}{m} \Big\Vert  \left[  \sum_{i=1}^{m} \left(A_i\right)_{1,1}  O_i \right] \vtaux^{\perp} \Big\Vert 
& \lesssim \sqrt{\frac{ \log T + \log \eeta}{m}}. 
\end{align*}
Plugging in these two estimates into \eqref{ineq:intern7} provides
\begin{equation*}
\frac{1}{m} \Big\Vert  \left[  \sum_{i=1}^{m} \left(A_i\right)_{1,1}  O_i    \right] \vtaux \Big\Vert 
\lesssim 
\sqrt{\frac{\log T}{m}}+  \sqrt{  \frac{n_1}{m}}  \Vert \vtaux^{\parallel} \Vert.
\end{equation*}
	
\noindent \textbf{Estimating $(b)$:} It follows from the restricted isometry property that
\begin{align*}
\frac{1}{m} \Big\Vert  \left[  \sum_{i=1}^{m} \left(A_i\right)_{1,1}  \tilde{O_i}    \right] \left(\vt -\vtaux \right) \Big\Vert  
= \Big\Vert \left[  \left(  \tilde{\mathcal{O}}^* \mathcal{D}  \right)  \left(    \ustar  \vstar^\top  \right)   \right] \left(  \vt -\vtaux \right)   \Big\Vert 
\le \delta \Vert \vt -\vtaux \Vert,
\end{align*}
where in the last line we used Lemma \ref{lemma:RIPlemma}.\\
	
\noindent \textbf{Estimating $(c)$:} By an analogous argument as for the first summand (a) we obtain for the third summand (c) that
\begin{equation*}
\frac{1}{m} \Big\Vert  \left[  \sum_{i=1}^{m} \left(A_i\right)_{1,1}  \tilde{O}_i   \right]v_t \Big\Vert 
\lesssim 
\sqrt{\frac{\log T + \log \eeta}{m}}+  \sqrt{  \frac{n_1}{m}}  \Vert \vtaux^{\parallel} \Vert.
\end{equation*}
Hence, by summing up these estimates we have shown that
\begin{align*}
\Big\Vert  \left[  \left(    \mathcal{A}^* \mathcal{A}  - \tilde{ \mathcal{A}}^* \tilde{\mathcal{A}}   \right) \left(    \ustar  \vstar^\top  \right)   \right] \vtaux \Big\Vert
\le C \left(  \sqrt{\frac{\log T}{m}}+  \sqrt{  \frac{n_1}{m}}  \Vert \vtaux^{\parallel} \Vert \right) + \delta \Vert \vt -\vtaux \Vert,
\end{align*}
which finishes the proof.

\subsection{Proof of Lemma \ref{lemma:help1}}\label{app:concentration5}
It follows from $ \mathcal{A} = \mathcal{D} + \mathcal{O} $ and $\mathcal{A} = \mathcal{D} + \tilde{ \mathcal{O}}  $ that
\begin{equation}\label{equ:intern10}
\begin{split}
\mathcal{A}^* \mathcal{A} - \tilde{\mathcal{A}}^* \tilde{\mathcal{A}} &= \left(  \mathcal{D} + \mathcal{O}   \right)^* \left(  \mathcal{D} + \mathcal{O}  \right) - \left(  \mathcal{D} + \tilde{ \mathcal{O}}    \right)^* \left(  \mathcal{D} +  \tilde{ \mathcal{O}}   \right)\\
&=   \mathcal{D}^*   \mathcal{O}  +   \mathcal{O}^*  \mathcal{D}  + \mathcal{O}^*   \mathcal{O}-  \mathcal{D}^*   \tilde{ \mathcal{O}} - \tilde{ \mathcal{O}}^*   \mathcal{D}  - \tilde{ \mathcal{O}}^*   \tilde{ \mathcal{O}}\\
&= \mathcal{D}^*  \left(   \mathcal{O}-  \tilde{ \mathcal{O}}  \right) +  \left(  \mathcal{O}  - \tilde{ \mathcal{O}}   \right)^*  \mathcal{D} + \left( \mathcal{O}^*   \mathcal{O} -  \tilde{ \mathcal{O}}^*   \tilde{ \mathcal{O}}\right).
\end{split}
\end{equation}
Using decomposition in \eqref{equ:intern10} and the triangle inequality we obtain that
\begin{align*}
&\Big\Vert \left[  \left(    \mathcal{A}^* \mathcal{A}  - \tilde{ \mathcal{A}}^* \tilde{\mathcal{A}}   \right) \left(  \utplushalfaux \vtaux^\top  \right)   \right] \vtaux   \Big\Vert\\
& \le \underset{=:(I)}{ \underbrace{ \Big\Vert \left[ \left(   \mathcal{D}^* \left(   \mathcal{O}-  \tilde{ \mathcal{O}}  \right)   \right) \left(  \utplushalfaux \vtaux^\top  \right)   \right] \vtaux  \Big\Vert }} + \underset{=:(II)}{ \underbrace{ \Big\Vert \left[ \left(   \left(   \mathcal{O}-  \tilde{ \mathcal{O}}  \right)^* \mathcal{D}   \right) \left(  \utplushalfaux \vtaux^\top  \right)   \right] \vtaux  \Big\Vert}}\\
& + \underset{=:(III)}{ \underbrace{ \Big\Vert  \left[  \left(  \mathcal{O}^*  \mathcal{O} -  \tilde{ \mathcal{O}}^* \tilde{ \mathcal{O}} \right)   \right] \left(  \utplushalfaux \vtaux^\top  \right)  \vtaux   \Big\Vert }}.
\end{align*}
We estimate these three summands separately.\\

\noindent\textbf{Bounding $(I)$:} 
Note that
\begin{align*}
&\Big\Vert \left[ \left(   {\mathcal{D}^*} {\left(   \mathcal{O} - \tilde{ \mathcal{O}} \right)}   \right) \left(  \utplushalfaux \vtaux^\top  \right) \right] \vtaux  \Big\Vert\\
&\overeq{(a)} \Big\Vert \left[ \left(   \mathcal{D}^*   {\left(   \mathcal{O}-  \tilde{ \mathcal{O}}  \right)}     \right) \left( \utplushalfaux^{\parallel} (\vtaux^{\perp})^\top  + \utplushalfaux^{\perp} (\vtaux^{\parallel})^\top  \right)   \right] \vtaux  \Big\Vert\\
& \overleq{(b)} \Big\Vert \left[ \left(   \mathcal{D}^*  {\left(   \mathcal{O}-  \tilde{ \mathcal{O}}  \right)}   \right) \left( \utplushalfaux^{\parallel} (\vtaux^{\perp})^\top   \right)   \right] \vtaux  \Big\Vert + \Big\Vert \left[ \left(   \mathcal{D}^* \left(   \mathcal{O}-  \tilde{ \mathcal{O}}  \right)   \right) \left(  \utplushalfaux^{\perp} (\vtaux^{\parallel})^\top  \right)   \right] \vtaux  \Big\Vert\\
& \overleq{(c)}  2 \delta \left(    \Vert  \utplushalfaux^{\parallel} \Vert \Vert \vtaux^{\perp} \Vert + \Vert  \utplushalfaux^{\perp}  \Vert \Vert  \vtaux^{\parallel} \Vert   \right) \\
& \overleq{(d)} 2 \delta \left(    \Vert  \utplushalfaux^{\parallel} \Vert + 2  \Vert  \vtaux^{\parallel} \Vert \right), 
\end{align*}
where equality (a) follows from the definition of $\mathcal{O}$ and $\tilde{ \mathcal{O}}$; Inequality (b) follows from the triangle inequality; Inequality (c) is due to Lemma \ref{lemma:RIPlemma} and the assumption that $\Vert \vtaux \Vert = 1$; Inequality (d) follows from $\Vert \vtaux \Vert =1 $ and $ \Vert  \utplushalfaux  \Vert \le 2 $.\\

\noindent\textbf{Bounding $(II)$:} By definition of $\mathcal{D}$ we have that
\begin{align*}
 \left[ \left(   \left(   \mathcal{O}-  \tilde{ \mathcal{O}}  \right)^* \mathcal{D}   \right) \left(  \utplushalfaux \vtaux^\top  \right)   \right] \vtaux  
 =  \left[ \left(   \left(   \mathcal{O}-  \tilde{ \mathcal{O}}  \right)^* \mathcal{D}   \right) \left(  \utplushalfaux^{\parallel} (\vtaux^{\parallel})^\top + \utplushalfaux^{\perp} (\vtaux^{\perp} )^\top \right)   \right] \vtaux.
\end{align*}
Hence by the triangle inequality it follows that
\begin{align*}
& \Big\Vert  \left[ \left(   \left(   \mathcal{O}-  \tilde{ \mathcal{O}}  \right)^* \mathcal{D}   \right) \left(  \utplushalfaux \vtaux^\top  \right)   \right] \vtaux \Big\Vert\\
& \le  \underset{ =:(\S) }{ \underbrace{ \Big\Vert \left[ \left(   \left(   \mathcal{O}-  \tilde{ \mathcal{O}}  \right)^* \mathcal{D}   \right)  \left(   \utplushalfaux^{\parallel} (\vtaux^{\parallel})^\top \right)   \right] \vtaux \Big\Vert } } + \underset{ =:(\S\S) }{ \underbrace{ \Big\Vert \left[ \left(   \left(   \mathcal{O}-  \tilde{ \mathcal{O}}  \right)^* \mathcal{D}   \right) \left(  \utplushalfaux^{\perp} \left(\vtaux^{\perp}  \right)^\top  \right)  \right] \vtaux \Big\Vert } }.
\end{align*}
\textbf{Estimating $(\S)$:} In order to bound the first term we note that
\begin{align*}
&\Big\Vert \left[ \left(   \left(   \mathcal{O}-  \tilde{ \mathcal{O}}  \right)^* \mathcal{D}   \right) \left(  \utplushalfaux^{\parallel} (\vtaux^{\parallel})^\top \right)   \right] \vtaux \Big\Vert
= \big\Vert \utplushalfaux^{\parallel} \big\Vert \cdot \big\Vert \vtaux^{\parallel} \big\Vert \cdot \Big\Vert \left[ \left(   \left(   \mathcal{O}-  \tilde{ \mathcal{O}}  \right)^* \mathcal{D}   \right) \left( \ustar  \vstar^\top \right)   \right] \vtaux \Big\Vert.
\end{align*}
Moreover note that
\begin{equation*}
 \left[ \left(   \left(   \mathcal{O}-  \tilde{ \mathcal{O}}  \right)^* \mathcal{D}   \right) \left( \ustar  \vstar^\top \right)   \right] \vtaux = \frac{1}{m}  \left[  \sum_{i=1}^{m} \left(A_i\right)_{1,1} \left( O_i - \tilde{O}_i \right)  \right] \vtaux.
\end{equation*}
Note that this is exactly the term, which appeared already in the inequality chain \eqref{ineq:intern8}. 
Hence, by exactly the same argument, since we assumed that \cref{ineq:randombound1a,ineq:randombound1b,ineq:indepbound1a,ineq:indepbound1b} hold, we then obtain that
\begin{equation}\label{eq:part1ofII}
\begin{aligned}
&\Big\Vert \left[ \left(   \left(   \mathcal{O}-  \tilde{ \mathcal{O}}  \right)^* \mathcal{D}   \right) \left(  \utplushalfaux^{\parallel} (\vtaux^{\parallel})^\top \right)   \right] \vtaux \Big\Vert\\
& \lesssim \big\Vert \utplushalfaux^{\parallel} \big\Vert \cdot \big\Vert \vt^{\parallel} \big\Vert \cdot \left(    \sqrt{\frac{\log T + \log \eeta}{m}}+  \sqrt{  \frac{n_1}{m}}  \Vert \vtaux^{\parallel} \Vert  + \delta \Vert \vt -\vtaux \Vert  \right).
\end{aligned}
\end{equation}

\noindent\textbf{Estimating $(\S\S)$:} In order to bound term $(2)$ we note that
\begin{equation*}
\left[ \left(   \left(   \mathcal{O}-  \tilde{ \mathcal{O}}  \right)^* \mathcal{D}   \right) \left(  \utplushalfaux^{\perp} \left(\vtaux^{\perp}  \right)^\top  \right)   \right] \vtaux = \frac{1}{m} \sum_{i=1}^{m}  \froinnerproduct{D_i, \utplushalfaux^{\perp} \left(\vtaux^{\perp}  \right)^\top }  \left( O_i - \tilde{O}_i \right) \vtaux .
\end{equation*}
Due to the triangle inequality it follows that
\begin{align*}
\Big\Vert \frac{1}{m} \sum_{i=1}^{m}  \froinnerproduct{D_i, \utplushalfaux^{\perp} \left(\vtaux^{\perp}  \right)^\top }  \left( O_i - \tilde{O}_i \right) \vtaux \Big\Vert &\le  \underset{=:(a)}{\underbrace{\Big\Vert \frac{1}{m} \sum_{i=1}^{m}  \innerproduct{D_i, \utplushalfaux^{\perp} \left(\vtaux^{\perp}  \right)^\top }   O_i  \tilde{\vt} \Big\Vert}}\\
 &+ \underset{=:(b)}{\underbrace{\Big\Vert \frac{1}{m} \sum_{i=1}^{m}  \froinnerproduct{D_i, \utplushalfaux^{\perp} \left(\vtaux^{\perp}  \right)^\top }  \tilde{O}_i  \left( \tilde{\vt} -\vt \right) \Big\Vert}}\\
 &+ \underset{=:(c)}{\underbrace{\Big\Vert \frac{1}{m} \sum_{i=1}^{m}  \froinnerproduct{D_i, \utplushalfaux^{\perp} \left(\vtaux^{\perp} -  \vt^{\perp}  \right)^\top }   \tilde{O}_i  v_t \Big\Vert}} \\
 & + \underset{=:(d)}{\underbrace{\Big\Vert \frac{1}{m} \sum_{i=1}^{m}  \froinnerproduct{D_i, \left(\utplushalfaux^{\perp} - \utplushalf^{\perp}  \right) \left(\vt^{\perp}  \right)^\top }   \tilde{O}_i  v_t \Big\Vert}} \\
& + \underset{=:(e)}{\underbrace{\Big\Vert \frac{1}{m} \sum_{i=1}^{m}  \froinnerproduct{D_i, \utplushalf^{\perp} \left(\vt^{\perp}  \right)^\top }   \tilde{O}_i  v_t \Big\Vert}}.
\end{align*}
We will estimate the summands individually.\\

\noindent \textbf{Estimating $(b)$, $(c)$, and $(d)$ :} By the consequences of RIP in Lemma \ref{lemma:RIPlemma}, the term $(b)$ is upper-bounded by
\begin{align*}
\Big\Vert \frac{1}{m} \sum_{i=1}^{m}  \froinnerproduct{D_i, \utplushalfaux^{\perp} \left(\vtaux^{\perp}  \right)^\top }  O_i  \left( \tilde{\vt} -\vt \right) \Big\Vert 
\le \delta 
\Vert \utplushalfaux^{\perp} \Vert 
\cdot 
\Vert \vtaux^{\perp} \Vert 
\cdot 
\Vert \tilde{\vt} -\vt \Vert
\le 2 \delta \Vert  \tilde{\vt} -\vt \Vert,
\end{align*}
where we used $\Vert \utplushalfaux^{\perp} \Vert \le 2 $ and $\Vert \vtaux^{\perp} \Vert \le 1 $. Similarly we obtain that
\begin{equation*}
\Big\Vert \frac{1}{m} \sum_{i=1}^{m}  \froinnerproduct{D_i, \utplushalfaux^{\perp} \left(\vtaux^{\perp} -  \vt^{\perp}  \right)^\top }   \tilde{O}_i  v_t \Big\Vert \le 2 \delta \Vert \vtaux^{\perp} -  \vt^{\perp}  \Vert \le 2\delta \Vert \vtaux -  \vt  \Vert 
\end{equation*}
and
\begin{equation*}
\Big\Vert \frac{1}{m} \sum_{i=1}^{m}  \froinnerproduct{D_i, \left(\utplushalfaux^{\perp} - \utplushalf^{\perp}  \right) \left(\vt^{\perp}  \right)^\top }   \tilde{O}_i  v_t \Big\Vert \le \delta \Vert \utplushalfaux^{\perp} - \utplushalf^{\perp}  \Vert \le \delta \Vert \utplushalfaux - \utplushalf  \Vert.   
\end{equation*}

\noindent\textbf{Estimating $(a)$:} By the triangle inequality it holds that
\begin{equation}\label{ineq:aux41}
\begin{split}
& \Big\Vert \frac{1}{m} \sum_{i=1}^{m}  \froinnerproduct{D_i, \utplushalfaux^{\perp} \left(\vtaux^{\perp}  \right)^\top }   O_i  \tilde{\vt} \Big\Vert\\
\le & \Big\Vert \frac{1}{m} \sum_{i=1}^{m}  \froinnerproduct{D_i,  \utplushalfaux^{\perp} \left( \vtaux^{\perp}  \right)^\top  }   O_i  \tilde{\vt}^{\perp} \Big\Vert       +  \Big\Vert \frac{1}{m} \sum_{i=1}^{m}  \froinnerproduct{D_i,  \utplushalfaux^{\perp} \left( \vtaux^{\perp}  \right)^\top  }   O_i  \tilde{\vt}^{\parallel} \Big\Vert.
\end{split}
\end{equation}
We estimate the two summands individually.
Note that from the definition of $O_i$ and $ \vtaux^\perp$ it follows that only the first entry of $ O_i \vtaux^\perp $ is non-zero.
It follows that
\begin{align*}
 \Big\Vert \frac{1}{m} \sum_{i=1}^{m}  \froinnerproduct{D_i,  \utplushalfaux^{\perp} \left( \vtaux^{\perp}  \right)^\top  }   O_i  \tilde{\vt}^{\perp} \Big\Vert &=\Big\vert \frac{1}{m} \sum_{i=1}^{m}  \froinnerproduct{D_i,  \utplushalfaux^{\perp} \left( \vtaux^{\perp}  \right)^\top  }   \innerproduct{O_i^\top e_i,  \tilde{\vt}^{\perp}  } \Big\vert
\end{align*}
Hence, it follows from \eqref{ineq:indepbound2a} that
\begin{align*}
 \Big\Vert \frac{1}{m} \sum_{i=1}^{m}  \froinnerproduct{D_i,  \utplushalfaux^{\perp} \left( \vtaux^{\perp}  \right)^\top  }   O_i  \tilde{\vt}^{\perp} \Big\Vert	
 &\lesssim \sqrt{\frac{\log T + \log \eeta}{m}} \cdot \twonorm{ \mathcal{A} \left( \utplushalfaux^\perp (\vtaux^\perp )^\top  \right) }\\
 &\lesssim \sqrt{\frac{\log T + \log \eeta}{m}}, 
\end{align*}
where in the second inequality we used the RIP of $\mathcal{A}$ as well as the assumption $ \twonorm{\utplushalfaux} \le 2$.
This provides an upper bound on the first summand of the right-hand side in \eqref{ineq:aux41}.
In order to bound the second summand we first choose a vector $u \in \mathbb{C}^{n_1}$ that satisfies $\Vert u \Vert =1 $, $ \innerproduct{u,\ustar } = 0 $, and 
\begin{align*}
\Big\Vert \frac{1}{m} \sum_{i=1}^{m}  \froinnerproduct{D_i,  \utplushalfaux^{\perp} \left( \vtaux^{\perp}  \right)^\top  }   O_i  \tilde{\vt}^{\parallel} \Big\Vert 
&=  \frac{1}{m} \sum_{i=1}^{m}  \froinnerproduct{D_i,  \utplushalfaux^{\perp} \left( \vtaux^{\perp}  \right)^\top  }  \innerproduct{  O_i  \tilde{\vt}^{\parallel}, u}.
\end{align*}
Such a vector exists due to the definitions of $O_i$ and $\vtaux^\parallel$ and the fact that the vector $O_i \vtaux^\parallel$ is orthogonal to $\ustar $.
Hence, we obtain that
\begin{align*}
\Big\Vert \frac{1}{m} \sum_{i=1}^{m}  \froinnerproduct{D_i,  \utplushalfaux^{\perp} \left( \vtaux^{\perp}  \right)^\top  }   O_i  \tilde{\vt}^{\parallel} \Big\Vert 
&=  \frac{1}{m} \sum_{i=1}^{m}  \froinnerproduct{D_i,  \utplushalfaux^{\perp} \left( \vtaux^{\perp}  \right)^\top  }  \froinnerproduct{  O_i , u\left(\tilde{\vt}^{\parallel} \right)^\top} \\
&\overeq{(i)}  \frac{1}{m} \sum_{i=1}^{m}  \froinnerproduct{A_i,  \utplushalfaux^{\perp} \left( \vtaux^{\perp}  \right)^\top  }  \froinnerproduct{  A_i , u\left(\tilde{\vt}^{\parallel} \right)^\top} \\
&\overeq{(ii)} \innerproduct{ \Aop{  \utplushalfaux^{\perp} \left( \vtaux^{\perp}  \right)^\top   }  , \Aop{  u\left( \tilde{\vt}^{\parallel} \right)^\top }} \\
& \overleq{(iii)} \delta \Vert  \utplushalfaux^{\perp} \left( \vtaux^{\perp}  \right)^\top \Vert_F \cdot \Vert u\left( \tilde{\vt}^{\parallel} \right)^\top  \Vert_F \\
&\le \delta \Vert  \utplushalfaux^{\perp} \Vert \cdot \Vert  \vtaux^{\perp}  \Vert \cdot \Vert u\Vert \cdot \Vert \tilde{\vt}^{\parallel} \Vert \\
&\overleq{(iv)} 2\delta \Vert  \tilde{\vt}^{\parallel} \Vert,
\end{align*}
where the identity $(i)$ follows from our choice of $u$ and the definition of $D_i$ and $O_i$; Equation $(ii)$ follows from the definition of $\mathcal{A}$; Inequality $(iii)$ is due to the consequences of RIP in Lemma \ref{lemma:RIPlemma}; Inequality $(iv)$ is obtained by $ \Vert  \utplushalfaux^{\perp} \Vert \le 2 $, $ \Vert u \Vert =1 $ and $  \Vert  \tilde{\vt}^{\perp} \Vert \le 1 $.
Hence, we have shown that 
\begin{equation*}
(a) \lesssim   \sqrt{\frac{\log T + \log \eeta}{m}} + \delta \Vert  \tilde{\vt}^{\parallel} \Vert.
\end{equation*}
\textbf{Estimating $(e)$:} We can upper-bound this term in an analogous way to term $(a)$, which yields that
\begin{equation*}
(e) \lesssim  \sqrt{\frac{\log T + \log \eeta}{m}} + \delta \Vert  \vt^{\parallel}  \Vert.
\end{equation*}
Summing up terms yields that
\begin{equation}\label{eq:part2ofII}
\begin{aligned}
(\S\S) 
= &\left(a\right) + (b) + (c) + (d) + (e) \\
\lesssim & 
\sqrt{\frac{\log T+\log \eeta}{m}} 
+ \delta \Vert \tilde{\vt}^{\parallel} \Vert 
+ \delta \Vert \tilde{\vt} -\vt \Vert 
+ \delta \Vert \utplushalfaux - \utplushalf  \Vert \delta \Vert  \vt^{\parallel}  \Vert \\
\lesssim &  \sqrt{\frac{\log T + \log \eeta}{m}} +  \delta \Vert  \tilde{\vt}^{\parallel} \Vert   +  \delta \Vert \vtaux -  \vt  \Vert   + \delta \Vert \utplushalfaux - \utplushalf  \Vert.
\end{aligned}
\end{equation}
By combining \eqref{eq:part1ofII} and \eqref{eq:part2ofII}, we obtain
\begin{align*}
(II)
&= (\S) + (\S\S) \\
& \lesssim  \big\Vert \utplushalfaux^{\parallel} \big\Vert \big\Vert \vt^{\parallel} \big\Vert \left(     \sqrt{\frac{\log T + \log \eeta}{m}}+  \sqrt{  \frac{n_1}{m}}  \Vert \vtaux^{\parallel} \Vert  + \delta \Vert \vt -\vtaux \Vert  \right)  \\
& + \sqrt{\frac{\log T + \log \eeta}{m}} +  \delta \Vert  \tilde{\vt}^{\parallel} \Vert   + \delta \Vert \vtaux -  \vt  \Vert   + \delta \Vert \utplushalfaux - \utplushalf  \Vert\\
& \lesssim  \sqrt{\frac{\log T + \log \eeta}{m}} +  \left( \delta + \sqrt{\frac{n_1}{m}} \right) \Vert  \tilde{\vt}^{\parallel} \Vert   +  \delta \Vert \vtaux -  \vt  \Vert   + \delta \Vert \utplushalfaux - \utplushalf  \Vert.
\end{align*}

\noindent\textbf{Bounding $(III)$:} Observe that
\begin{align*}
&\Big\Vert  \left[  \left(  \mathcal{O}^*  \mathcal{O} -  \tilde{ \mathcal{O}}^* \tilde{ \mathcal{O}} \right)   \right] \left(  \utplushalfaux \vtaux^\top  \right)  \vtaux   \Big\Vert\\ 
& \overeq{(a)} \Big\Vert  \left[  \left(  \mathcal{O}^*  \mathcal{O} -  \tilde{ \mathcal{O}}^* \tilde{ \mathcal{O}} \right) \left(  \utplushalfaux^{\parallel} (\vtaux^{\perp})^\top   + \utplushalfaux^{\perp} (\vtaux^{\parallel})^\top   \right)  \right]   \vtaux   \Big\Vert\\
& \le \Big\Vert  \left[  \left(  \mathcal{O}^*  \mathcal{O} -  \tilde{ \mathcal{O}}^* \tilde{ \mathcal{O}} \right) \left(  \utplushalfaux^{\parallel} (\vtaux^{\perp})^\top    \right)  \right]   \vtaux   \Big\Vert + \Big\Vert  \left[  \left(  \mathcal{O}^*  \mathcal{O} -  \tilde{ \mathcal{O}}^* \tilde{ \mathcal{O}} \right) \left(    \utplushalfaux^{\perp} (\vtaux^{\parallel})^\top   \right)  \right]   \vtaux   \Big\Vert\\
& \le \Big\Vert  \left[  \left(  \mathcal{O}^*  \mathcal{O} - P_{O} \right) \left(  \utplushalfaux^{\parallel} (\vtaux^{\perp})^\top    \right)  \right]   \vtaux   \Big\Vert  + \Big\Vert  \left[  \left( P_{O}-  \tilde{ \mathcal{O}}^* \tilde{ \mathcal{O}} \right) \left(  \utplushalfaux^{\parallel} (\vtaux^{\perp})^\top    \right)  \right]   \vtaux   \Big\Vert   \\
& \quad +\Big\Vert  \left[  \left(  \mathcal{O}^*  \mathcal{O} - P_{O}  \right) \left(    \utplushalfaux^{\perp} (\vtaux^{\parallel})^\top   \right)  \right]   \vtaux   \Big\Vert + \Big\Vert  \left[  \left(  P_{O}-  \tilde{ \mathcal{O}}^* \tilde{ \mathcal{O}} \right) \left(    \utplushalfaux^{\perp} (\vtaux^{\parallel})^\top   \right)  \right]   \vtaux   \Big\Vert\\
& \overleq{(b)} 2\delta  \left(  \big\Vert  \utplushalfaux^{\parallel} \big\Vert \cdot \Vert \vtaux^{\perp} \big\Vert + \big\Vert  \utplushalfaux^{\perp}  \big\Vert \cdot \Vert  \vtaux^{\parallel} \big\Vert  \right)  \\
& \overleq{(c)} 4 \delta  \left(  \big\Vert  \utplushalfaux^{\parallel} \big\Vert +  \big\Vert \vtaux^{\parallel} \big\Vert  \right),
\end{align*}
where the identity $(a)$ follows from the definition of $\mathcal{O}$ and $\tilde{\mathcal{O}}$; Inequality $(b)$ is due to Lemma \ref{lemma:RIPlemma}; Inequality $(c)$ follows from $\Vert \vtaux \Vert =1$ and $ \Vert \utplushalfaux  \Vert \le 2 $.\\ 

\noindent Finally, by combining the upper estimates of $(I)$, $(II)$, and $(III)$, we obtain 
\begin{align*}
& \Big\Vert \left[  \left(    \mathcal{A}^* \mathcal{A}  - \tilde{ \mathcal{A}}^* \tilde{\mathcal{A}}   \right) \left(  \utplushalfaux \vtaux^\top   \right)   \right] \vtaux   \Big\Vert\\
\le & (I)+(II)+(III)\\
& \lesssim \delta \left(    \Vert  \utplushalfaux^{\parallel} \Vert  +\Vert  \vt^{\parallel} \Vert   \right) \\
& \quad + \left( \sqrt{\frac{\log T + \log \eeta}{m}} + \delta + \sqrt{\frac{n_1}{m}} \right) \Vert  \tilde{\vt}^{\parallel} \Vert   +  \delta \Vert \vtaux -  \vt  \Vert   + \delta \Vert \utplushalfaux - \utplushalf  \Vert \\
& \quad + \delta  \left(  \big\Vert  \utplushalfaux^{\parallel} \big\Vert +  \big\Vert \vt^{\parallel} \big\Vert  \right)\\
\lesssim & \sqrt{\frac{\log T + \log \eeta}{m}} + \left( \delta + \sqrt{\frac{n_1}{m}} \right) \Vert  \tilde{\vt}^{\parallel} \Vert  + \delta  \big\Vert  \utplushalfaux^{\parallel} \big\Vert + \delta \Vert \vtaux -  \vt  \Vert  + \delta  \Vert \utplushalfaux - \utplushalf  \Vert.
\end{align*}
This completes the proof.

\subsection{Proof of Lemma \ref{lemma:nearindependencebounds}}\label{app:concentration6}
The RIP of $\mathcal{A}$ provides 
\begin{equation}\label{ineq:intern20}
\begin{split}
& \Big\vert \innerproduct{ \mathcal{A}  \left(  \ustar \left( v^{\perp}_t  \right)^\top \right), \mathcal{A} \left(  \ustar \vstar^\top \right) } \Big\vert\\
& \le \Big\vert \innerproduct{ \mathcal{A}  \left(  \ustar \left( v^{\perp}_t -\vtaux^{\perp}  \right)^\top \right), \mathcal{A} \left(  \ustar \vstar^\top \right)     } \Big\vert  + \Big\vert \innerproduct{ \mathcal{A}  \left(  \ustar \left( \vtaux^{\perp}  \right)^\top \right), \mathcal{A} \left(  \ustar \vstar^\top   \right)     } \Big\vert \\
& \le \delta \Vert  v^{\perp}_t -\vtaux^{\perp} \Vert + \Big\vert \innerproduct{ \mathcal{A}  \left(  \ustar \left( \vtaux^{\perp}  \right)^\top \right), \mathcal{A} \left(  \ustar \vstar^\top   \right)     } \Big\vert.
\end{split}
\end{equation} 
The second term in the right-hand side of \eqref{ineq:intern20} is rewritten as 
\begin{align*}
\Big\vert \innerproduct{ \mathcal{A}  \left(  \ustar \left( \vtaux^\perp \right)^\top \right), \mathcal{A} \left(  \ustar \vstar^\top   \right)     } \Big\vert 
&= \frac{1}{m} \Big\vert \sum_{i=1}^m \froinnerproduct{A_i, \ustar \left( \vtaux^\perp \right)^\top }  \froinnerproduct{A_i, \ustar \vstar^\top} \Big\vert\\
&= \frac{1}{m} \Big\vert \sum_{i=1}^m \froinnerproduct{A_i, \ustar \left( \vtaux^\perp \right)^\top } \left(A_i\right)_{1,1}  \Big\vert\\
&= \frac{1}{m} \Big\vert \sum_{i=1}^m \froinnerproduct{O_i, \ustar \left( \vtaux^\perp \right)^\top } \left(A_i\right)_{1,1}  \Big\vert\\
&= \frac{1}{m} \Big\vert \froinnerproduct{ \sum_{i=1}^m  \left(A_i\right)_{1,1} O_i, \ustar \left( \vtaux^\perp \right)^\top }   \Big\vert \\
&= \frac{1}{m} \Big\vert \innerproduct{ \sum_{i=1}^m  \left(A_i\right)_{1,1} O_i  \vtaux^\perp  , \ustar  }   \Big\vert \\
&= \frac{1}{m} \twonorm{   \sum_{i=1}^m  \left(A_i\right)_{1,1} O_i  \vtaux^\perp   }.
\end{align*}
Hence, the assumption in \eqref{ineq:indepbound1a} implies 
\begin{equation*}
\Big\vert \innerproduct{ \mathcal{A}  \left(  \ustar \left( v^{\perp}_t  \right)^\top \right), \mathcal{A} \left(  \ustar \vstar^\top   \right)     } \Big\vert \lesssim   \sqrt{\frac{ \log T + \log n}{m}} \Vert  \vtaux^{\perp}  \Vert.
\end{equation*}
Inserting this inequality into \eqref{ineq:intern20} yields inequality \eqref{ineq:intern5}.\\

It remains to show the inequality in \eqref{ineq:intern6}. By applying the triangle inequality several times in combination with the RIP of $\mathcal{A}$ we obtain that
\begin{equation}\label{ineq:intern33}
\begin{split}
& \Big\vert \innerproduct{\mathcal{A} \left( \ustar \left( v^{\perp}_t \right)^\top \right)  , \mathcal{A} \left(   \utplushalf^{\perp} \left( v^{\perp}_t \right)^\top  \right) } \Big\vert \\
& \le \delta \Vert \utplushalf^{\perp} - \utplushalfaux^{\perp} \Vert + 2\delta \Vert v^{\perp}_t -\vtaux^{\perp} \Vert \cdot \twonorm{ \utplushalf^{\perp} } +    	\Big\vert \innerproduct{\mathcal{A} \left( \ustar \left( \vtaux^{\perp} \right)^\top \right)  , \mathcal{A} \left(   \utplushalfaux^{\perp} \left( \vtaux^{\perp} \right)^\top  \right) } \Big\vert\\
& \le \delta \Vert \utplushalf^{\perp} - \utplushalfaux^{\perp} \Vert + 4\delta \Vert v^{\perp}_t -\vtaux^{\perp} \Vert +    	\Big\vert \innerproduct{\mathcal{A} \left( \ustar \left( \vtaux^{\perp} \right)^\top \right)  , \mathcal{A} \left(   \utplushalfaux^{\perp} \left( \vtaux^{\perp} \right)^\top  \right) } \Big\vert, 
\end{split}
\end{equation}
where in the last inequality we used that $\twonorm{\utplushalf^\perp} \le 2$, which holds by Lemma~\ref{lemma:perpdecrease} due to the RIP of $\mathcal{A}$. Next, we note that
\begin{align*}
\Big\vert \innerproduct{\mathcal{A} \left( \ustar \left( \vtaux^{\perp} \right)^\top \right)  , \mathcal{A} \left(   \utplushalfaux^{\perp} \left( \vtaux^{\perp} \right)^\top  \right) } \Big\vert 
&= \frac{1}{m} \Big\vert \sum_{i=1}^{m} \froinnerproduct{A_i,\ustar \left( \vtaux^{\perp} \right)^\top} \froinnerproduct{A_i, \utplushalfaux^{\perp} \left( \vtaux^{\perp} \right)^\top }    \Big\vert\\
&= \frac{1}{m} \Big\vert \sum_{i=1}^{m} \froinnerproduct{ O_i ,\ustar \left( \vtaux^{\perp} \right)^\top} \froinnerproduct{D_i, \utplushalfaux^{\perp} \left( \vtaux^{\perp} \right)^\top }    \Big\vert\\
&= \frac{1}{m} \Big\vert \sum_{i=1}^{m} \innerproduct{ O_i^\top e_1 , \vtaux^{\perp}} \froinnerproduct{D_i, \utplushalfaux^{\perp} \left( \vtaux^{\perp} \right)^\top }    \Big\vert.
\end{align*}
Hence, it follows from \eqref{ineq:indepbound2a}, the RIP of $\mathcal{A}$, and $ \twonorm{ \utplushalfaux } \le 2$ that
\begin{align*}
\Big\vert \innerproduct{\mathcal{A} \left( \ustar \left( \vtaux^{\perp} \right)^\top \right)  , \mathcal{A} \left(   \utplushalfaux^{\perp} \left( \vtaux^{\perp} \right)^\top  \right) } \Big\vert 
&\lesssim \sqrt{ \frac{\log T + \log \eeta}{m} } \cdot \twonorm{ \mathcal{A} \left(  \utplushalfaux^\perp (\vtaux^\perp )^\top  \right)  }\\
&\lesssim \sqrt{ \frac{\log T + \log \eeta}{m} } .
\end{align*}
Combining this inequality with \eqref{ineq:intern33} yields \eqref{ineq:intern6}.

\section{Proofs of Lemmas in Phase 1}

\subsection{Proof of Lemma~\ref{lemma:perpdecrease}}
\label{app:perpdecrease}
It follows from the normal equations that
\begin{align*}
&\utplushalf - \innerproduct{\vstar ,v_t} \ustar \\
=& \left[ \left( \Id -\mathcal{A}^* \mathcal{A} \right) \left(  \utplushalf \vt^\top  -   \ustar  \vstar^\top    \right)   \right] \vt \\
=& \left[ \left( \Id -\mathcal{A}^* \mathcal{A} \right) \left( \left(  \utplushalf - \innerproduct{\vstar ,v_t} \ustar  \right) \vt^\top    \right)   \right] \vt + \left[ \left( \Id -\mathcal{A}^* \mathcal{A} \right) \left(    \ustar  \left( \innerproduct{\vstar ,v_t} \vt^\top  -  \vstar^\top   \right)   \right)   \right] \vt.
\end{align*}
In the following we will set for convenience that $\lambda_t = \innerproduct{\vstar ,\vt} $. Then we obtain by the previous calculation, the triangle inequality, and the Restricted Isometry Property that
\begin{align*}
&\Vert \utplushalf - \lambda_t \ustar  \Vert\\
\le &\Big\Vert  \left[ \left( \Id -\mathcal{A}^* \mathcal{A} \right) \left( \left(  \utplushalf - \lambda_t \ustar  \right) \vt^\top    \right)   \right] \vt \Big\Vert + \Big\Vert  \left[ \left( \Id -\mathcal{A}^* \mathcal{A} \right) \left(    \ustar  \left( \lambda_t \vt -  \vstar  \right)^\top   \right)   \right] \vt \Big\Vert \\
\le & \delta \left( \Vert  \utplushalf - \lambda_t \ustar  \Vert +    \Vert  \lambda_t \vt -  \vstar    \Vert \right),
\end{align*}
where in the last line we have used that $ \twonorm{\vt}=1 $.
Rearranging terms yields that
\begin{equation}\label{ineq:intern3}
\Vert \utplushalf - \lambda_t \ustar  \Vert \le \frac{\delta}{1-\delta}   \Big\Vert  \lambda_t \vt -  \vstar    \Big\Vert.
\end{equation}
We compute that
\begin{align*}
\lambda_t \vt -  \vstar   &= \lambda^2_t \vstar  + \lambda_t \vt^\perp  - \vstar   = \left(  \lambda^2_t-1  \right) \vstar  + \lambda_t \vt^\perp.
\end{align*}
Due to $ 1-\lambda_t^2   = 1- \innerproduct{\vstar ,v_t}^2   = \Vert \vt^\perp \Vert^2  $ and $ \lambda_t^2 = \Vert \vstar ^{\parallel} \Vert^2 $ this implies that
\begin{equation*}
\begin{split}
\Vert \lambda_t \vt -  \vstar   \Vert^2 &= \left(1-\lambda_t^2\right)^2  \Vert \vstar  \Vert^2 + \lambda^2_t \Vert \vt^{\perp} \Vert^2  \\
 &= \Vert \vt^{\perp} \Vert^4   + \Vert \vt^{\parallel} \Vert^2 \Vert \vt^{\perp} \Vert^2\\
 &= \Vert \vt^{\perp} \Vert^2,
\end{split}
\end{equation*}
where in the last line we used that $ \twonorm{ \vt^\parallel }^2 + \twonorm{\vt^\perp}^2 = \twonorm{\vt}^2= 1$.
Together with \eqref{ineq:intern3}  this shows \eqref{equ:bound1}. 
Since $\Vert \utplushalf^\perp\Vert \le \big\Vert \utplushalf - \innerproduct{\vstar ,v_t} \ustar  \big\Vert $ this implies \eqref{equ:bound2}.  
In order to prove inequality \eqref{equ:bound3} we note that
\begin{align*}
\Vert \utplushalf \Vert &\le \Vert \utplushalf -  \innerproduct{\vstar ,\vt}  \ustar  \Vert + \vert  \innerproduct{\vstar ,\vt}  \vert \Vert \ustar  \Vert \\
&\le \frac{\delta}{1-\delta} \Vert \vt^{\perp} \Vert + \vert  \innerproduct{\vstar ,\vt}  \vert \\
&\le \frac{\delta}{1-\delta}+ \vert \innerproduct{\vstar ,\vt}  \vert \\
&\le 2,
\end{align*}
where the third line follows from inequality \eqref{equ:bound1} and from $\Vert \ustar  \Vert =1 $. 
In the last line we used the assumption that $\delta \le \frac{1}{2} $ and $ \twonorm{\vt} =\twonorm{\vstar } =1 $. 
This shows inequality \eqref{equ:bound3}.

\subsection{Proof of Lemma~\ref{lemma:closeness3}}
\label{app:closeness3}
We will first show the following auxiliary inequality:
\begin{equation}\label{ineq:closeness3}
\begin{split}
\twonorm{  	\utplushalf^{\parallel} - 	\utplushalfaux^{\parallel}} 
\le & \left( \ct + C\delta \left(\ct+1 \right) \right) \twonorm{\vt^\parallel} 
+C\delta  \twonorm{\utplushalf^\perp - \utplushalfaux^\perp} 
\end{split},
\end{equation}
where $C>0$ is an absolute constant chosen large enough.\\

\noindent \textbf{Proof of inequality \eqref{ineq:closeness3}}:
Recall that $\utplushalf$ satisfies
\begin{align*}
\utplushalf - \innerproduct{\vt,\vstar } \ustar  &= [(\mathrm{Id} - \mathcal{A}^* \mathcal{A}) (\utplushalf\vt^\top  - \ustar  \vstar^\top )] \vt.
\end{align*}
Then it follows that 
\begin{align*}
\utplushalf - \innerproduct{\vt,\vstar } \ustar  &= [(\mathrm{Id} - \mathcal{A}^* \mathcal{A}) (\utplushalf\vt^\top  - \ustar  \vstar^\top )] \vt,
\end{align*}
which is equivalently rewritten as 
\begin{equation}\label{eq:utplushalf_para}
\utplushalf^{\parallel} - \innerproduct{\vt,\vstar } \ustar =\innerproduct{\ustar  \vt^\top , (\mathrm{Id} - \mathcal{A}^* \mathcal{A}) (\utplushalf\vt^\top  - \ustar  \vstar^\top )} \ustar.
\end{equation}
Similarly $\utplushalfaux$ also satisfies
\begin{equation}\label{eq:utplushalfaux_para}
\utplushalfaux^{\parallel} -  \innerproduct{\vtaux,\vstar }  \ustar = \innerproduct{ \ustar  \vtaux^\top ,  (\mathrm{Id} - \tilde{\mathcal{A}}^* \tilde{\mathcal{A}}) (\utplushalfaux \vtaux^\top  - \ustar  \vstar^\top )} \ustar.
\end{equation}
We obtain from \eqref{eq:utplushalf_para} and \eqref{eq:utplushalfaux_para} that 
\begin{align*}
	\utplushalf^{\parallel} - 	\utplushalfaux^{\parallel}  =& \innerproduct{\vt- \vtaux,\vstar } \ustar   +\innerproduct{\ustar  \vt^\top , (\mathrm{Id} - \mathcal{A}^* \mathcal{A}) (\utplushalf\vt^\top - \ustar  \vstar^\top )} \ustar \\
	 &- \froinnerproduct{ \ustar  \vtaux^\top ,  (\mathrm{Id} - \tilde{\mathcal{A}}^* \tilde{\mathcal{A}}) (\utplushalfaux \vtaux^\top  - \ustar  \vstar^\top )} \ustar \\
=& \innerproduct{\vt- \vtaux,\vstar } \ustar   +  \froinnerproduct{\ustar   \left( \vt - \vtaux  \right)^\top, (\mathrm{Id} - \mathcal{A}^* \mathcal{A}) (\utplushalf\vt^\top  - \ustar  \vstar^\top )} \ustar \\
& +  \innerproduct{\ustar  \vtaux^\top , (\mathrm{Id} - \mathcal{A}^* \mathcal{A}) (\utplushalf\vt^\top  - \ustar  \vstar^\top ) - (\mathrm{Id} - \tilde{\mathcal{A}}^* \tilde{\mathcal{A}}) (\utplushalfaux  \vtaux^\top  - \ustar  \vstar^\top )} \ustar \\
=& \innerproduct{\vt- \vtaux,\vstar } \ustar   +  \froinnerproduct{\ustar   \left( \vt - \vtaux  \right)^\top , (\mathrm{Id} - \mathcal{A}^* \mathcal{A}) (\utplushalf\vt^\top  - \ustar  \vstar^\top )} \ustar \\
& + \froinnerproduct{\ustar  \vtaux^\top, (\mathrm{Id} - \mathcal{A}^* \mathcal{A}) (\utplushalf\vt^\top ) - (\mathrm{Id} - \tilde{\mathcal{A}}^* \tilde{\mathcal{A}}) (\utplushalfaux  \vtaux^\top )} \ustar \\
&+ \froinnerproduct{  \ustar  \vtaux^\top , \left(   \tilde{\mathcal{A}}^* \tilde{\mathcal{A}}-  \mathcal{A}^* \mathcal{A}  \right) \left(  \ustar  \vstar^\top  \right) } \ustar \\
=& \innerproduct{\vt- \vtaux,\vstar } \ustar   +  \froinnerproduct{\ustar   \left( \vt - \vtaux  \right)^\top , (\mathrm{Id} - \mathcal{A}^* \mathcal{A}) (\utplushalf\vt^\top  - \ustar  \vstar^\top )} \ustar \\
&+ \froinnerproduct{\ustar  \vtaux^\top , (\mathrm{Id} - \mathcal{A}^* \mathcal{A})  \left( \utplushalf  \vt^\top   - \utplushalfaux  \vtaux^\top   \right)  } \ustar \\
&+ \froinnerproduct{  \ustar  \vtaux^\top , \left(   \tilde{\mathcal{A}}^* \tilde{\mathcal{A}}-  \mathcal{A}^* \mathcal{A}  \right) \left(  \utplushalfaux  \vtaux^\top   \right) } \ustar  \\
& + \froinnerproduct{  \ustar  \vtaux^\top , \left(   \tilde{\mathcal{A}}^* \tilde{\mathcal{A}}-  \mathcal{A}^* \mathcal{A}  \right) \left(  \ustar  \vstar^\top  \right) } \ustar. 
\end{align*}	
It follows from the triangle inequality, the restricted isometry property, and the Cauchy-Schwarz inequality that
\begin{align*}
\twonorm{  	\utplushalf^{\parallel} - 	\utplushalfaux^{\parallel}} \le& \twonorm{ \vt^{\parallel} - \vtaux^{\parallel}    } + \delta  \twonorm{\vt -\vtaux} \cdot \fronorm{\utplushalf\vt^\top - \ustar  \vstar ^\top }  \\
&+ \delta \twonorm{\vtaux} \cdot \fronorm{  \utplushalf  \vt^\top  - \utplushalfaux  \vtaux^\top   } +  \Big\Vert \left[  \left(   \tilde{\mathcal{A}}^* \tilde{\mathcal{A}}-  \mathcal{A}^* \mathcal{A}  \right) \left(  \utplushalfaux  \vtaux^\top  \right) \right]  \vtaux \Big\Vert \\
&+ \twonorm{ \left[  \left(    \tilde{\mathcal{A}}^* \tilde{\mathcal{A}}-  \mathcal{A}^* \mathcal{A}  \right) \left(  \ustar  \vstar ^\top \right)  \right] \vtaux   }\\
\le& \twonorm{ \vt^{\parallel} - \vtaux^{\parallel}    } + \delta  \twonorm{\vt -\vtaux} \left( \twonorm{\utplushalf} \cdot \twonorm{\vt} + \twonorm{\ustar } \cdot \twonorm{\vstar }  \right)  \\
&+ \delta \twonorm{\vtaux} \left(   \twonorm{\utplushalf - \utplushalfaux} \cdot \twonorm{\vt} + \twonorm{\utplushalfaux} \cdot \twonorm{ \vt - \vtaux }       \right)  \\
&+  \Big\Vert \left[  \left(   \tilde{\mathcal{A}}^* \tilde{\mathcal{A}}-  \mathcal{A}^* \mathcal{A}  \right) \left(  \utplushalfaux  \vtaux^\top  \right) \right]  \vtaux \Big\Vert \\
&+ \Big\Vert \left[  \left(    \tilde{\mathcal{A}}^* \tilde{\mathcal{A}}-  \mathcal{A}^* \mathcal{A}  \right) \left(  \ustar  \vstar ^\top \right)  \right] \vtaux \Big\Vert \\
\le & \twonorm{ \vt^{\parallel} - \vtaux^{\parallel}    }  + 5\delta \twonorm{\vt -\vtaux} + 2 \delta    \twonorm{\utplushalf - \utplushalfaux} \\
&+  \Big\Vert \left[  \left(   \tilde{\mathcal{A}}^* \tilde{\mathcal{A}}-  \mathcal{A}^* \mathcal{A}  \right) \left(  \utplushalfaux  \vtaux^\top  \right) \right]  \vtaux \Big\Vert \\
&+ \Big\Vert \left[  \left(    \tilde{\mathcal{A}}^* \tilde{\mathcal{A}}-  \mathcal{A}^* \mathcal{A}  \right) \left(  \ustar  \vstar ^\top \right)  \right] \vtaux \Big\Vert,
\end{align*}
where in the last inequality we have used the assumptions $\twonorm{\utplushalf}  \le 2$ and $\twonorm{\utplushalfaux}  \le 2$.
Recall from Lemma \ref{lemma:help1} that
\begin{equation}\label{ineq:aux4}
\begin{split}
&\Big\Vert \left[  \left(    \mathcal{A}^* \mathcal{A}  - \tilde{ \mathcal{A}}^* \tilde{\mathcal{A}}   \right) \left(  \utplushalfaux \vtaux^\top   \right)   \right] \vtaux  \Big\Vert\\
\lesssim &  \sqrt{\frac{\log T}{m}}  + \delta \Vert  \vt^{\parallel} \Vert  + \left( \delta + \sqrt{\frac{n_1}{m}} \right) \Vert  \tilde{\vt}^{\parallel} \Vert  + \delta \Vert \vtaux -  \vt  \Vert  + \delta  \Vert \utplushalfaux - \utplushalf  \Vert. 
\end{split}
\end{equation}	
From Lemma \ref{lemma:stochasticbound} it follows that
\begin{equation}\label{ineq:auxiliarybound102}
\Big\Vert \left[  \left(    \tilde{\mathcal{A}}^* \tilde{\mathcal{A}}-  \mathcal{A}^* \mathcal{A}  \right) \left(  \ustar  \vstar^\top  \right)  \right] \vtaux \Big\Vert \le C \left(  \sqrt{\frac{\log T}{m}}+  \sqrt{  \frac{n_1}{m}}  \Vert \vtaux^{\parallel} \Vert \right) + \delta \Vert \vt -\vtaux \Vert.
\end{equation}
This implies that there is an absolute constant $\tilde{C}_1>0$ such that 
\begin{align*}
& \twonorm{  	\utplushalf^{\parallel} - 	\utplushalfaux^{\parallel}}\\
 \le & \twonorm{ \vt^{\parallel} - \vtaux^{\parallel}    }  + \tilde{C}_1 \left(   \sqrt{\frac{\log T}{m}}+ \left(  \sqrt{  \frac{n_1}{m}} + \delta  \right)  \Vert \vtaux^{\parallel} \Vert + \delta \twonorm{ \vt^\parallel } + \delta \Vert \vt -\vtaux \Vert  + \delta    \twonorm{\utplushalf - \utplushalfaux}   \right).
\end{align*}
By using Assumption \eqref{ineq:induction1} and Condition ii) of Proposition~\ref{prop:stage1}, we obtain that
\begin{align*}
	&\twonorm{  	\utplushalf^{\parallel} - 	\utplushalfaux^{\parallel}}\\
	\le & \twonorm{ \vt^{\parallel} - \vtaux^{\parallel}    }  
	+ \tilde{C}_2 \left(   \delta \Vert \vtaux^{\parallel} \Vert +   \delta \Vert \vt^{\parallel} \Vert  + \delta \Vert \vt -\vtaux \Vert  + \delta    \twonorm{\utplushalf - \utplushalfaux}   \right)
\end{align*}
with an absolute constant $\tilde{C}_2 >0$ chosen large enough.
By using the triangle inequality and Assumption \eqref{ineq:induction2} we obtain that
\begin{align*}
	&\twonorm{  	\utplushalf^{\parallel} - 	\utplushalfaux^{\parallel}} 
	\le \twonorm{ \vt^{\parallel} - \vtaux^{\parallel}    }  
	+\tilde{C}_3\delta (1+\ct) \twonorm{\vt^\parallel} 
	+\tilde{C}_3\delta  \twonorm{\utplushalf - \utplushalfaux},
\end{align*}
where $\tilde{C}_3>0$ is an absolute constant chosen large enough.
By using the triangle inequality, by rearranging terms, and using the elementary inequality $ 1/(1-x) \le 1+2x $ for $ 0 < x <1/2 $ it follows that
\begin{align*}
\twonorm{  	\utplushalf^{\parallel} - 	\utplushalfaux^{\parallel}} 
\le &  \left(1+\tilde{C}_4\delta\right) \twonorm{ \vt^{\parallel} - \vtaux^{\parallel}    }  
+\tilde{C}_4\delta (1+\ct)  \twonorm{\vt^\parallel} 
+\tilde{C}_4\delta  \twonorm{\utplushalf^\perp - \utplushalfaux^\perp},
\end{align*}
where $\tilde{C}_4>0$ is an absolute constant chosen large enough and we have used that $\delta>0$ is chosen small enough. Using Assumption \eqref{ineq:induction2} we obtain that
\begin{align*}
	\twonorm{  	\utplushalf^{\parallel} - 	\utplushalfaux^{\parallel}}
	&\le \ct \twonorm{ \vt^{\parallel} }  
	+\tilde{C}_4\delta (1+\ct)  \twonorm{\vt^\parallel} 
	+\tilde{C}_4\delta  \twonorm{\utplushalf^\perp - \utplushalfaux^\perp} \\
	&\le \left( \left(1+ 2\tilde{C}_4\delta \right) \ct + \tilde{C}_4 \delta \right) \twonorm{\vt^\parallel}
	+\tilde{C}_4\delta  \twonorm{\utplushalf^\perp - \utplushalfaux^\perp} 
\end{align*}
This shows the auxiliary inequality \eqref{ineq:closeness3}.\\

\noindent \textbf{Proof of inequality \eqref{ineq:closeness5}:}
Having established the auxiliary inequality \eqref{ineq:closeness3}, we can in the next step prove inequality \eqref{ineq:closeness5}. 
It follows from the normal equations that
	\begin{align*}
	\utplushalf^\perp  &= P_{\ustar ^\perp}  [(\mathrm{Id} - \mathcal{A}^* \mathcal{A}) (\utplushalf\vt^\top  - \ustar  \vstar^\top )] \vt, \\
	\utplushalfaux^\perp&=  P_{\ustar ^\perp}  [(\mathrm{Id} - \tilde{\mathcal{A}}^* \tilde{\mathcal{A}}) (\utplushalfaux \vtaux^\top  - \ustar  \vstar^\top )] \vtaux.
\end{align*}
Hence, we obtain that
\begin{align*}
&\twonorm{	\utplushalf^\perp - 	\utplushalfaux^\perp   }\\
 \le & \twonorm{[(\mathrm{Id} - \mathcal{A}^* \mathcal{A}) (\utplushalf\vt^\top  - \ustar  \vstar^\top )] \vt -[(\mathrm{Id} - \tilde{\mathcal{A}}^* \tilde{\mathcal{A}}) (\utplushalfaux \vtaux^\top - \ustar  \vstar^\top )] \vtaux }\\
\le  &  \bracing{=:(I)}{ \twonorm{[(\mathrm{Id} - \mathcal{A}^* \mathcal{A}) (\utplushalf\vt^\top  )] \vt -[(\mathrm{Id} - \tilde{\mathcal{A}}^* \tilde{\mathcal{A}}) (\utplushalfaux \vtaux^\top  )] \vtaux } } \\
&+ \bracing{=:(II)}{ \twonorm{[(\mathrm{Id} - \mathcal{A}^* \mathcal{A}) ( \ustar  \vstar^\top )] \vt -[(\mathrm{Id} - \tilde{\mathcal{A}}^* \tilde{\mathcal{A}}) (  \ustar  \vstar^\top )] \vtaux }}.
\end{align*}
We estimate the first term by
\begin{align*}
\twonorm{(I)}  \overleq{(a)} & \twonorm{  (\mathrm{Id} - \mathcal{A}^* \mathcal{A}) (\utplushalf\vt^\top  - \utplushalfaux \vtaux^\top   ) \vt    } + \Big\Vert \left[ (\mathrm{Id} - \mathcal{A}^* \mathcal{A}) (\utplushalfaux \vtaux^\top  ) \right] \left( \vt - \vtaux \right) \Big\Vert\\
&+  \Big\Vert \left[  \left(   \tilde{\mathcal{A}}^* \tilde{\mathcal{A}}-  \mathcal{A}^* \mathcal{A}  \right) \left(  \utplushalfaux  \vtaux^\top   \right) \right]  \vtaux  \Big\Vert\\
\overleq{(b)} & \delta \fronorm{\utplushalf\vt^\top  - \utplushalfaux \vtaux^\top    } + \delta  \fronorm{ \utplushalfaux \vtaux^\top  }   \twonorm{  \vt - \vtaux } \\
& + \Big\Vert \left[  \left(   \tilde{\mathcal{A}}^* \tilde{\mathcal{A}}-  \mathcal{A}^* \mathcal{A}  \right) \left(  \utplushalfaux  \vtaux^\top   \right) \right]  \vtaux \Big\Vert \\
\overleq{(c)} &  \delta  \twonorm{ \utplushalf - \utplushalfaux  }      + 3\delta  \twonorm{  \vt - \vtaux } + \Big\Vert \left[  \left(   \tilde{\mathcal{A}}^* \tilde{\mathcal{A}}-  \mathcal{A}^* \mathcal{A}  \right) \left(  \utplushalfaux  \vtaux^\top   \right) \right]  \vtaux \Big\Vert \\
\overset{(d)}{\lesssim} &  \sqrt{\frac{\log T}{m}}  + \delta \Vert  \vt^{\parallel} \Vert  + \left( \delta + \sqrt{\frac{n_1}{m}} \right) \Vert  \tilde{\vt}^{\parallel} \Vert  + \delta \Vert \vtaux -  \vt  \Vert  + \delta  \Vert \utplushalfaux - \utplushalf  \Vert .
\end{align*}
In inequality (a) we used the triangle inequality and in inequality (b) we used the Restricted Isometry Property. In inequality (c) we used the triangle inequality as well as $\twonorm{\utplushalf} \le 2$. Inequality (d) follows from inserting inequality \eqref{ineq:aux4}.
In the next step, we are going to estimate summand $(II)$. For that, we observe
\begin{align*}
\twonorm{(II)} \le & \twonorm{ [(\mathrm{Id} - \mathcal{A}^* \mathcal{A}) ( \ustar  \vstar^\top )] \left( \vt -\vtaux \right)  } + \Big\Vert \left[  \left(    \tilde{\mathcal{A}}^* \tilde{\mathcal{A}}-  \mathcal{A}^* \mathcal{A}  \right) \left(  \ustar  \vstar^\top  \right)  \right] \vtaux \Big\Vert \\
\le & 2 \delta \twonorm{ \vt -\vtaux } + C \left(  \sqrt{\frac{\log T}{m}}+  \sqrt{  \frac{n_1}{m}}  \Vert \vtaux^{\parallel} \Vert \right),
\end{align*}
where in the second inequality we have used inequality \eqref{ineq:auxiliarybound102} and that $\mathcal{A}$ satisfies
the Restricted Isometry Property.	
Hence, we have shown that
\begin{align*}
\twonorm{	\utplushalf^\perp - 	\utplushalfaux^\perp   } \le & \Vert (I) \Vert + \Vert (II) \Vert \\
 \lesssim & \sqrt{\frac{\log T}{m}}  + \delta \Vert  \vt^{\parallel} \Vert  + \left( \delta + \sqrt{\frac{n_1}{m}} \right) \Vert  \tilde{\vt}^{\parallel} \Vert  + \delta \Vert \vtaux -  \vt  \Vert  + \delta  \Vert \utplushalfaux - \utplushalf  \Vert \\
 \le &  \sqrt{\frac{\log T}{m}}  + \delta \Vert  \vt^{\parallel} \Vert  + \left( \delta + \sqrt{\frac{n_1}{m}} \right) \Vert  \tilde{\vt}^{\parallel} \Vert  + \delta \Vert \vtaux -  \vt  \Vert  + \delta  \Vert \utplushalfaux^\perp - \utplushalf^\perp  \Vert  \\
 &  + \delta  \Vert \utplushalfaux^\parallel - \utplushalf^\parallel  \Vert,
\end{align*}
where for the last line we used the triangle inequality.
Next, we obtain that
\begin{align*}
&\twonorm{	\utplushalf^\perp - \utplushalfaux^\perp} \\
\lesssim  & \delta \Vert  \vt^{\parallel} \Vert  
+ \delta \Vert  \tilde{\vt}^{\parallel} \Vert  + \delta \Vert \vtaux -  \vt  \Vert  
+ \delta  \Vert \utplushalfaux^\perp - \utplushalf^\perp  \Vert 
  + \delta  \Vert \utplushalfaux^\parallel - \utplushalf^\parallel  \Vert,
\end{align*}
where we have used Assumption \eqref{ineq:induction1} and Condition ii) of Proposition~\ref{prop:stage1}. By rearranging terms and using our assumption $\delta < 1/2$ we obtain 
that
\begin{align*}
	&\twonorm{	\utplushalf^\perp - \utplushalfaux^\perp} \lesssim   \delta \Vert  \vt^{\parallel} \Vert  
	+ \delta \Vert  \tilde{\vt}^{\parallel} \Vert  + \delta \Vert \vtaux -  \vt  \Vert  
	  + \delta  \Vert \utplushalfaux^\parallel - \utplushalf^\parallel  \Vert .	
\end{align*}
By using the triangle inequality and Assumption~\eqref{ineq:induction2} we obtain that
\begin{align*}
	&\twonorm{	\utplushalf^\perp - \utplushalfaux^\perp} \lesssim  
	\delta (1+\ct) \Vert \vt^\parallel  \Vert   + \delta  \Vert \utplushalfaux^\parallel - \utplushalf^\parallel  \Vert .		
\end{align*}
By inserting the auxiliary inequality \eqref{ineq:closeness3} we obtain that
\begin{align*}
	\twonorm{	\utplushalf^\perp - \utplushalfaux^\perp} \lesssim  
	&\delta \left(1+\ct\right) \Vert \vt^\parallel  \Vert 
	  +\delta^2 \twonorm{ \utplushalf^\perp -\utplushalfaux^\perp}.
\end{align*}
By rearranging terms we obtain that
\begin{align*}
	\twonorm{	\utplushalf^\perp - \utplushalfaux^\perp} \lesssim  \delta \left(1+\ct\right) \Vert \vt^\parallel  \Vert. 	
\end{align*}
This shows the claimed inequality \eqref{ineq:closeness5}.

In order to finish the proof, it remains to prove inequality \eqref{ineq:closeness4}.
For that, it suffices to note that this inequality follows from inserting inequality \eqref{ineq:closeness5},
which we have just shown,
into the auxiliary inequality \eqref{ineq:closeness3}.

\subsection{Proof of Lemma~\ref{lemma:uparallelbound}}
\label{app:uparallelbound}
For convenience, we set $ \lambda_t= \innerproduct{v_t,\vstar }$.
We compute that
\begin{align*}
\Vert \utplushalf^{\parallel} \Vert &= \vert \innerproduct{\utplushalf, \ustar } \vert \\
 &= \vert  \innerproduct{ \utplushalf-\lambda_t \ustar , \ustar  } + \lambda_t \innerproduct{\ustar ,\ustar }  \vert\\
 &= \vert  \innerproduct{ \utplushalf-\lambda_t \ustar , \ustar  } + \innerproduct{\vt, \vstar }  \vert.
\end{align*}
It follows from the triangle inequality and $\vert \innerproduct{\vt, \vstar }  \vert = \twonorm{\vt^\parallel}$  that
\begin{equation}\label{ineq:intern2}
	\twonorm{\utplushalf^\parallel} - \vert \innerproduct{ \utplushalf-\lambda_t \ustar , \ustar  } \vert  \le  \Vert \vt^{\parallel} \Vert  \le \twonorm{\utplushalf^\parallel} +\vert \innerproduct{ \utplushalf-\lambda_t \ustar , \ustar  } \vert.
\end{equation}
Hence, we need to bound $ \vert \innerproduct{ \utplushalf-\lambda_t \ustar , \ustar  } \vert $ from above. For that purpose we compute that
\begin{align*}
&\utplushalf-\lambda_t \ustar \\
=& \left[ \left(\Id - \mathcal{A}^* \mathcal{A} \right) \left( \utplushalf v_t^\top  - \ustar  \vstar^\top     \right)   \right] v_t\\
=& \lambda_t  \left[ \left(\Id - \mathcal{A}^* \mathcal{A} \right) \left( \utplushalf v_t^\top - \ustar  \vstar^\top     \right)   \right] \vstar  + \left[ \left(\Id - \mathcal{A}^* \mathcal{A} \right) \left( \utplushalf v_t^\top - \ustar  \vstar^\top     \right)   \right] \vt^\perp\\
=& \lambda_t  \left[ \left(\Id - \mathcal{A}^* \mathcal{A} \right) \left( \utplushalf v_t^\top - \ustar  \vstar^\top     \right)   \right] \vstar  + \left[ \left(\Id - \mathcal{A}^* \mathcal{A} \right) \left( \utplushalf v_t^\top   \right)   \right] \vt^\perp - \left[ \left(\Id - \mathcal{A}^* \mathcal{A} \right) \left(  \ustar  \vstar^\top    \right)   \right] \vt^\perp \\
=&  \lambda_t  \left[ \left(\Id - \mathcal{A}^* \mathcal{A} \right) \left( \utplushalf v_t^\top - \ustar  \vstar^\top    \right)   \right] \vstar  + \left[ \left(\Id - \mathcal{A}^* \mathcal{A} \right) \left( \utplushalf v_t^\top   \right)   \right] \vt^\perp + \left[ \left( \mathcal{A}^* \mathcal{A} \right) \left(  \ustar  \vstar^\top     \right)   \right] \vt^\perp \\
=&\lambda_t  \left[ \left(\Id - \mathcal{A}^* \mathcal{A} \right) \left( \utplushalf v_t^\top - \ustar  \vstar^\top     \right)   \right] \vstar  + \lambda_t  \left[ \left(\Id - \mathcal{A}^* \mathcal{A} \right) \left( \utplushalf \vstar^\top    \right)   \right] \vt^\perp \\
&+ \left[ \left(\Id - \mathcal{A}^* \mathcal{A} \right) \left( \utplushalf \left( v^{\perp}_t \right)^\top    \right)   \right] \vt^\perp + \left[ \left( \mathcal{A}^* \mathcal{A} \right) \left(  \ustar  \vstar^\top     \right)   \right] \vt^\perp \\
 =&\lambda_t  \left[ \left(\Id - \mathcal{A}^* \mathcal{A} \right) \left( \utplushalf v_t^\top- \ustar  \vstar^\top     \right)   \right] \vstar  + \lambda_t  \left[ \left(\Id - \mathcal{A}^* \mathcal{A} \right) \left( \utplushalf \vstar^\top     \right)   \right] \vt^\perp \\ 
 &+  \innerproduct{\utplushalf, \ustar } \left[ \left(\Id - \mathcal{A}^* \mathcal{A} \right) \left( \ustar  \left( v^{\perp}_t \right)^\top    \right)   \right] \vt^\perp + \left[ \left(\Id - \mathcal{A}^* \mathcal{A} \right) \left( \utplushalf^{\perp} \left( v^{\perp}_t \right)^\top    \right)   \right] \vt^\perp \\
 &+ \left[ \left( \mathcal{A}^* \mathcal{A} \right) \left(  \ustar  \vstar^\top     \right)   \right] \vt^\perp.
\end{align*}
It follows that
\begin{align*}
& \vert \innerproduct{\utplushalf-\lambda_t \ustar ,\ustar } \vert \\
\le 
& \vert \lambda_t \vert \cdot \Big\Vert  \left[ \left(\Id - \mathcal{A}^* \mathcal{A} \right) \left( \utplushalf v_t^\top - \ustar  \vstar^\top   \right)   \right] \vstar   \Big\Vert
+ \vert \lambda_t \vert \cdot \Big\Vert \left[ \left(\Id - \mathcal{A}^* \mathcal{A} \right) \left( \utplushalf \vstar^\top    \right)   \right] \vt^\perp \Big\Vert\\
& + \vert \innerproduct{\utplushalf, \ustar } \vert \cdot \Big\Vert \left[ \left(\Id - \mathcal{A}^* \mathcal{A} \right) \left( \ustar  \left( v^{\perp}_t \right)^\top   \right)   \right] \vt^\perp \Big\Vert
+ \Big\vert \innerproduct{\ustar ,   \left[ \left(\Id - \mathcal{A}^* \mathcal{A} \right) \left( \utplushalf^{\perp} \left( v^{\perp}_t \right)^\top   \right)   \right] \vt^\perp } \Big\vert \\
&+ \Big\vert \innerproduct{  \left[ \left( \mathcal{A}^* \mathcal{A} \right) \left(  \ustar  \vstar^\top    \right)   \right] \vt^\perp , \ustar  } \Big\vert \\
= &  \twonorm{ \vt^\parallel }  \cdot \Big\Vert  \left[ \left(\Id - \mathcal{A}^* \mathcal{A} \right) \left( \utplushalf v_t^\top - \ustar  \vstar^\top    \right)   \right] \vstar   \Big\Vert
+ \twonorm{ \vt^\parallel } \cdot \Big\Vert \left[ \left(\Id - \mathcal{A}^* \mathcal{A} \right) \left( \utplushalf \vstar^\top    \right)   \right] \vt^\perp \Big\Vert\\
& +  \twonorm{\utplushalf^\parallel} \cdot \Big\Vert \left[ \left(\Id - \mathcal{A}^* \mathcal{A} \right) \left( \ustar  \left( v^{\perp}_t \right)^\top   \right)   \right] \vt^\perp \Big\Vert
+ \Big\vert \innerproduct{  \mathcal{A} \left( \utplushalf^{\perp} \left( v^{\perp}_t \right)^\top \right), \mathcal{A} \left(  \ustar  (\vt^\perp)^\top \right) } \Big\vert \\
&+ \Big\vert \innerproduct{  \mathcal{A} \left(  \ustar  \vstar^\top   \right)    ,  \mathcal{A} \left( \ustar  \left(\vt^\perp\right)^\top \right)} \Big\vert .
\end{align*}
By the RIP of $\mathcal{A}$ and the assumption $ \twonorm{\utplushalf} \le 2$ we obtain that
\begin{align*}
\Big\Vert  \left[ \left(\Id - \mathcal{A}^* \mathcal{A} \right) \left( \utplushalf v_t^\top - \ustar  \vstar^\top    \right)   \right] \vstar    \Big\Vert &\le \delta \Vert    \utplushalf v_t^\top - \ustar  \vstar^\top   \Vert_F\\
&\le \delta \left( \Vert \utplushalf \Vert \cdot \Vert v_t \Vert  + \Vert \ustar  \vstar^\top  \Vert_F \right)  \\
&= \delta  \left( \Vert \utplushalf \Vert + 1    \right)  \\ 
&\le 3\delta.
\end{align*} 
Furthermore, it follows from the Restricted Isometry Property and the assumption $ \twonorm{\utplushalf} \le 2$ that
\begin{align*}
\Big\Vert  \left[ \left(\Id - \mathcal{A}^* \mathcal{A} \right) \left( \utplushalf \vstar^\top    \right)   \right] \vt^\perp \Big\Vert 
&\le \delta \Vert \utplushalf \Vert \cdot \Vert \vstar \Vert \cdot \Vert \vt^\perp \Vert \le 2 \delta
\end{align*}
and
\begin{align*}
\Big\Vert \left[ \left(\Id - \mathcal{A}^* \mathcal{A} \right) \left( \ustar  \left( v^{\perp}_t \right)^\top   \right)   \right] \vt^\perp  \Big\Vert &\le \delta  \Vert \ustar  \Vert \cdot \Vert \vt^{\perp} \Vert^2 \le \delta.
\end{align*}
We obtain that
\begin{align*}
& \vert \innerproduct{\utplushalf-\lambda_t \ustar ,\ustar } \vert \\
& \le 5\delta	\twonorm{\vt^\parallel } + \delta \twonorm{ \utplushalf^\parallel }
 + \Big\vert \innerproduct{  \mathcal{A} \left( \utplushalf^{\perp} \left( v^{\perp}_t \right)^\top \right), \mathcal{A} \left(  \ustar  (\vt^\perp)^\top \right) } \Big\vert 
+ \Big\vert \innerproduct{  \mathcal{A} \left(  \ustar  \vstar^\top    \right)    ,  \mathcal{A} \left( \ustar  \left(\vt^\perp\right)^\top \right)} \Big\vert .
\end{align*}
Recall from Lemma \ref{lemma:nearindependencebounds} that
\begin{align*}
	\Big\vert \innerproduct{ \mathcal{A}  \left(  \ustar  \left( v^{\perp}_t  \right)^\top \right), \mathcal{A} \left(  \ustar  \vstar^\top    \right)     } \Big\vert &\le \delta \Vert  \vt^\perp -\vtaux^\perp \Vert + C \sqrt{\frac{\log T}{m}}
\end{align*}
and
\begin{align*}
   \Big\vert \innerproduct{\mathcal{A} \left( \ustar  \left( v^{\perp}_t \right)^\top \right)  , \mathcal{A} \left(   \utplushalf^{\perp} \left( v^{\perp}_t \right)^\top  \right) } \Big\vert 
   &\le \delta \Vert \utplushalf^\perp - \utplushalfaux^\perp \Vert + 2\delta \Vert \vt^\perp -\vtaux^\perp \Vert +C \sqrt{\frac{\log T}{m}}.
\end{align*}
Inserting these estimates into the above inequality we obtain that
\begin{align*}
&\vert \innerproduct{\utplushalf-\lambda_t \ustar ,\ustar } \vert	\\
& \le 5\delta	\twonorm{\vt^\parallel } 
+ \delta \twonorm{ \utplushalf^\parallel } 
+  \delta \Vert \utplushalf^\perp - \utplushalfaux^\perp \Vert 
+ 3\delta \Vert \vt^\perp -\vtaux^\perp \Vert + 2C \sqrt{\frac{\log T}{m}} \\
& \lesssim \delta  \left(1+\ct\right)	\twonorm{\vt^\parallel } 
+ \delta \twonorm{ \utplushalf^\parallel },
\end{align*}
where in the last line we used Assumptions \eqref{ineq:induction1}, \eqref{ineq:induction2}, Condition ii) of Proposition~\ref{prop:stage1}, and \eqref{ineq:closeness5}.
By inserting this estimate into \eqref{ineq:intern2} and by rearranging terms we obtain inequality \eqref{ineq:uparallelbound}.
This finishes the proof.

\subsection{Proof of Lemma~\ref{lemma:closeness4}}
\label{app:proof_normalization}
\textbf{Part 1 (Estimating $\twonorm{ \utplus^\parallel - \utplusaux^\parallel }$): }
First, we are going to estimate $\twonorm{ \utplus^\parallel - \utplusaux^\parallel }$. We compute that
\begin{align*}
\twonorm{ \utplus^\parallel - \utplusaux^\parallel }  &=\twonorm{ \frac{\utplushalf^\parallel}{\twonorm{\utplushalf}} - \frac{\utplushalfaux^\parallel }{\twonorm{\utplushalfaux}}  }\\
&= \frac{ \Big\Vert  \twonorm{\utplushalfaux } \utplushalf^\parallel - \twonorm{\utplushalf } \utplushalfaux^\parallel      \Big\Vert}{ \twonorm{\utplushalf} \cdot \twonorm{\utplushalfaux } }\\
&\le \frac{\twonorm{\utplushalf^\parallel  - \utplushalfaux^\parallel   }}{\twonorm{\utplushalf}} + \frac{\Big\vert \twonorm{\utplushalfaux} - \twonorm{\utplushalf } \Big\vert}{\twonorm{\utplushalf}} \cdot \frac{\twonorm{\utplushalfaux^\parallel }}{  \twonorm{\utplushalfaux}} \\
&= \frac{\twonorm{\utplushalf^\parallel  - \utplushalfaux^\parallel   }}{\twonorm{\utplushalf}} + \frac{\Big\vert \twonorm{\utplushalfaux} - \twonorm{\utplushalf } \Big\vert}{\twonorm{\utplushalf}} \cdot \twonorm{\utplusaux^\parallel }\\
&\le \bracing{=:(\S)}{ \frac{\twonorm{\utplushalf^\parallel  - \utplushalfaux^\parallel   }}{\twonorm{\utplushalf}} } + \bracing{=:(\S\S)} { \frac{ \twonorm{\utplushalfaux - \utplushalf }}{\twonorm{\utplushalf}} \cdot \twonorm{\utplusaux^\parallel } }.
\end{align*}	
We estimate the two summands separately.\\

\noindent \textbf{Estimation of $(\S)$:} We obtain that
\begin{equation}\label{ineq:aux35}
\begin{split}
\frac{\twonorm{\utplushalf^\parallel  - \utplushalfaux^\parallel   }}{\twonorm{\utplushalf}} 
\overleq{(a)}& \left( \ct + C_1\delta(1+\ct) \right) \frac{ \twonorm{ \vt^\parallel }}{\twonorm{\utplushalf}} \\
\overleq{(b)}& \left( \ct + C_1\delta(1+\ct) \right) \left( 1 + C_2\delta(1+\ct) \right) \frac{ \twonorm{ \utplushalf^\parallel }}{\twonorm{\utplushalf}} \\
=& \left( \ct + C_1\delta(1+\ct) \right) \left( 1 + C_2\delta(1+\ct) \right) \twonorm{ \utplus^\parallel }
\end{split}
\end{equation}
where in inequality $(a)$ we have used Assumption \eqref{ineq:closeness4} and in inequality $(b)$ we have used Assumption \eqref{ineq:uparallelbound}.\\

\noindent \textbf{Estimation of $(\S\S)$:} 
By the triangle inequality we have
\begin{equation}\label{ineq:auxiliary31}
\begin{split}
\frac{ \twonorm{\utplushalfaux - \utplushalf }}{\twonorm{\utplushalf}} \cdot \twonorm{\utplusaux^\parallel }
\le 
\frac{ \twonorm{\utplushalfaux^\parallel - \utplushalf^\parallel }}{\twonorm{\utplushalf}} \cdot \twonorm{\utplusaux^\parallel } +  \frac{ \twonorm{\utplushalfaux^\perp - \utplushalf^\perp }}{\twonorm{\utplushalf}} \cdot \twonorm{\utplusaux^\parallel }.
\end{split}
\end{equation}
Then we estimate the two summands in the right-hand side of \eqref{ineq:auxiliary31} individually. 
It follows from \eqref{ineq:aux35} that the first summand is upper-bounded by
\begin{align*}
\frac{ \twonorm{\utplushalfaux^\parallel - \utplushalf^\parallel }}{\twonorm{\utplushalf}} \cdot \twonorm{\utplusaux^\parallel }
\le & \left(\ct + C_1 \delta (1 + \ct) \right) \left(1+ C_2 \delta (1+\ct) \right)   \twonorm{\utplus^\parallel} \cdot \twonorm{\utplusaux^\parallel }. 
\end{align*}
Moreover by Assumptions \eqref{ineq:closeness5} and \eqref{ineq:uparallelbound} the second summand is upper-bounded by
\begin{align*}
\frac{ \twonorm{\utplushalfaux^\perp - \utplushalf^\perp }}{\twonorm{\utplushalf}} \cdot \twonorm{\utplusaux^\parallel } &\overleq{(a)} C_1 \delta \left(1+\ct\right)  \frac{  \Vert \vt^{\parallel} \Vert }{\twonorm{\utplushalf}}   \cdot \twonorm{\utplusaux^\parallel } \\  
&\overleq{(b)} C_1 \delta \left(1+\ct\right) \left(1+ C_2 \delta (1+\ct) \right) \frac{  \Vert \utplushalf^{\parallel} \Vert }{\twonorm{\utplushalf}} \cdot \twonorm{\utplusaux^\parallel }  \\ 
&= C_1 \delta \left( 1+\ct \right) \left(1+ C_2 \delta (1+\ct) \right) \twonorm{\utplus^\parallel}  \cdot \twonorm{\utplusaux^\parallel }.
\end{align*}
\noindent By combining the two estimates and inserting them into \eqref{ineq:auxiliary31}, we obtain that
\begin{equation*}
\frac{ \twonorm{\utplushalfaux - \utplushalf }}{\twonorm{\utplushalf}} \cdot \twonorm{\utplusaux^\parallel } \le 
\left(\ct + 2 C_1 \delta (1 + \ct) \right) \left(1+ C_2 \delta (1+\ct) \right)   \twonorm{\utplus^\parallel}
\cdot \twonorm{\utplusaux^\parallel }.
\end{equation*}

\noindent \textbf{Combining the estimates: } By combining the estimates for $(\S)$ and $(\S\S)$ it follows that
\begin{align*}
&\twonorm{ \utplus^\parallel - \utplusaux^\parallel }\\
& \le (\S) + (\S\S) \\
& \le \left( \ct + C_1\delta(1+\ct) \right) \left( 1 + C_2\delta(1+\ct) \right) \twonorm{ \utplus^\parallel } \\
& \quad + \left(\ct + 2 C_1 \delta (1 + \ct) \right) \left(1+ C_2 \delta (1+\ct) \right)   \twonorm{\utplus^\parallel} \cdot \twonorm{\utplusaux^\parallel } \\ 
& \le \left( \ct + 2 C_1\delta(1+\ct) \right) \left( 1 + C_2\delta(1+\ct) \right) \twonorm{ \utplus^\parallel } \\
& \quad + \left(\ct + 2 C_1 \delta (1 + \ct) \right) \left(1+ C_2 \delta (1+\ct) \right)   \twonorm{\utplus^\parallel} \left( \twonorm{\utplusaux^\parallel } + \twonorm{\utplusaux^\parallel - \utplus^\parallel  } \right),
\end{align*}
which is rearranged as
\begin{align}
& \left( 1 - \left(\ct + 2 C_1 \delta (1 + \ct) \right) \left(1+ C_2 \delta (1+\ct) \right)   \twonorm{\utplus^\parallel} \right) \twonorm{ \utplus^\parallel - \utplusaux^\parallel } \nonumber \\
& \le
\left( \ct + 2 C_1\delta(1+\ct) \right) \left( 1 + C_2\delta(1+\ct) \right) \twonorm{ \utplus^\parallel } \left(1 + \twonorm{ \utplus^\parallel } \right). \label{eq:ctbound_eq1mx}
\end{align}
Due to Lemma~\ref{lemma:ctbound} we have $\ct \lesssim 1$. Therefore one can choose $c$ in \eqref{ineq:end_stage1} as a small absolute constant so that $\delta = \frac{c}{4 \log n_2}$ satisfies
\begin{equation}
\label{eq:ctbound_equbx}
\left(\ct + 2 C_1 \delta (1 + \ct) \right) \left(1+ C_2 \delta (1+\ct) \right) \twonorm{\utplus^\parallel} < \frac{1}{2}.
\end{equation}
Then, since $\frac{1}{1-x} \le 1+2x $ for $0<x<1/2$, it follows from \eqref{eq:ctbound_eq1mx} and \eqref{eq:ctbound_equbx} that 
\begin{align}
& \twonorm{ \utplus^\parallel - \utplusaux^\parallel } \nonumber \\
& \le 
\underbrace{ \left( 1 + 2 \left(\ct + 2 C_1 \delta (1 + \ct) \right) \left(1+ C_2 \delta (1+\ct) \right)   \twonorm{\utplus^\parallel} \right)}_{(i)} \nonumber \\
& \quad \cdot \underbrace{\left( \ct + 2 C_1\delta(1+\ct) \right)}_{(ii)} \underbrace{\left( 1 + C_2\delta(1+\ct) \right)}_{(iii)} \twonorm{ \utplus^\parallel } \underbrace{\left(1 + \twonorm{ \utplus^\parallel } \right)}_{(iv)} \nonumber \\
& \le \left( 1 + \frac{C_3 c}{\log n_2} \right)^3 \left( \ct + \frac{C_3 c}{\log n_2} \right) \twonorm{\utplus^\parallel} 
\label{ineq:uparallel_aux_dist}
\end{align}
for some absolute constant $C_3$, where the second inequality follows from the assumptions $\twonorm{\utplus^\parallel} < \frac{c}{\log n_2}$ and $\delta = \frac{c}{4 \log n_2}$, and the fact that $\ct \leq C_0$ due to Lemma~\ref{lemma:ctbound}. 
Indeed, the above conditions imply
\begin{align*}
(i) &= 1 + 2 \left(\ct + 2 C_1 \delta (1 + \ct) \right) \left(1+ C_2 \delta (1+\ct) \right)   \twonorm{\utplus^\parallel} \\
&\leq 1 + \left(C_0 + \frac{2C_1 (C_0+1) c}{4 \log n_2}\right) \cdot \left(1 + \frac{C_2 (C_0+1) c}{4 \log n_2} \right) \cdot \frac{c}{\log n_2}.
\end{align*}
Then we need to choose $C_3$ so that
\[
\left(C_0 + \frac{2C_1 (C_0+1) c}{4 \log n_2}\right) \cdot \left(1 + \frac{C_2 (C_0+1) c}{4 \log n_2} \right) \leq C_3.
\]
The constant $C_3$ also needs to satisfy
\begin{align*}
(ii) &= \ct + 2C_1 \delta (1+\ct) 
\leq \ct + \frac{2C_1(C_0+1)c}{\log n_2} 
\leq \ct + \frac{C_3 c}{\log n_2}, \\
(iii) &= 1 + C_2 \delta(1+\ct) 
\leq 1 + \frac{C_2(C_0+1)c}{\log n_2} 
\leq 1 + \frac{C_3 c}{\log n_2},
\intertext{and}
(iv) & = 1 + \twonorm{\utplus^\parallel} 
\leq 1 + \frac{c}{4 \log n_2} 
\leq 1 + \frac{C_3 c}{\log n_2}.
\end{align*}
This is implied by
\[
\max\left\{ 2C_1(C_0+1), C_2(C_0+1), \frac{1}{4} \right\} \leq C_3.
\]
Thus, there exists an absolute constant $C_3>0$ that satisfies the above conditions. Then one can choose an absolute constant $c>0$ small enough so that the upper bound in \eqref{ineq:uparallel_aux_dist} reduces to 
\begin{equation}\label{ineq:parallelauxcloseness1}
	\twonorm{ \utplus^\parallel - \utplusaux^\parallel } 
	 \le  \bracing{=:c_{2t+1}}{\left[  \left(1+ \frac{1}{\log n_2}\right) \ct +  \frac{1}{\log n_2}    \right]}  \twonorm{\utplus^\parallel}.
\end{equation}
Thus we have shown the claimed bound for $  \twonorm{ \utplus^\parallel - \utplusaux^\parallel }  $.\\

\noindent\textbf{Part 2 (Estimating $\twonorm{ \utplus^\perp - \utplusaux^\perp }  $):}
Analogous as in the beginning of the proof, where we provided an estimate for $ \twonorm{ \utplus^\parallel - \utplusaux^\parallel } $, 
we can show that
\begin{equation*}
\twonorm{ \utplus^\perp - \utplusaux^\perp } 
\le  \frac{\twonorm{\utplushalf^\perp  - \utplushalfaux^\perp   }}{\twonorm{\utplushalf}}  +\frac{ \twonorm{\utplushalfaux - \utplushalf }}{\twonorm{\utplushalf}} \cdot \twonorm{\utplusaux^\perp }.
\end{equation*}	
By using the triangle inequality and $\Vert \utplusaux^\perp \Vert \le 1$ it follows that
\begin{equation}\label{ineq:aux36}
\twonorm{ \utplus^\perp - \utplusaux^\perp } \le  \frac{2 \twonorm{\utplushalf^\perp  - \utplushalfaux^\perp   }}{\twonorm{\utplushalf}}  +  \frac{ \twonorm{\utplushalfaux^\parallel - \utplushalf^\parallel }}{\twonorm{\utplushalf}}.
\end{equation}	
We are going to estimate the two summands individually. 
By Assumptions \eqref{ineq:closeness5} and \eqref{ineq:uparallelbound}, the first summand is upper-bounded by
\begin{align*}
\frac{2 \twonorm{\utplushalf^\perp  - \utplushalfaux^\perp   }}{\twonorm{\utplushalf}} 
&\overleq{(a)} \frac{ 2 C_1 \delta (1+\ct) \Vert \vt^{\parallel} \Vert}{\twonorm{\utplushalf}} \\
&\overleq{(b)} 2 C_1 \delta (1+\ct) \left( 1+ C_2 \delta (1+\ct)  \right) \twonorm{  \utplus  }.
\end{align*}
Moreover, we use the estimate from the inequality chain \eqref{ineq:aux35} to obtain that
\begin{equation*}
\frac{ \twonorm{\utplushalfaux^\parallel - \utplushalf^\parallel }}{\twonorm{\utplushalf}} 
\le \left( \ct + C_1 \delta (1 + \ct) \right)\left(1+ C_2 (1+\ct)\delta\right) \twonorm{ \utplus^\parallel }.
\end{equation*}	
Hence, by inserting these estimates into \eqref{ineq:aux36}, we obtain that
\begin{align*}
\twonorm{ \utplus^\perp - \utplusaux^\perp } 
&\le
\left( \ct + 3 C_1 \delta (1 + \ct) \right)\left(1+ C_2 (1+\ct)\delta\right) \twonorm{ \utplus^\parallel } \\
&\le \left(\ct + \frac{C_4 c}{\log n_2} \right) \left(1 + \frac{C_4 c}{\log n_2} \right) \twonorm{ \utplus^\parallel } \\
&= \left[ \left(1 + \frac{C_4 c}{\log n_2} \right) \ct + \frac{C_4 c}{\log n_2} \left(1 + \frac{C_4 c}{\log n_2} \right) \right] \twonorm{ \utplus^\parallel } 
\end{align*}
for some absolute constant $C_4$, where the second inequality is dervied similarly to that of \eqref{ineq:uparallel_aux_dist}. 
Since $\ct \lesssim 1$, by choosing $c$ as a small enough absolute constant so that 
\begin{equation}\label{ineq:parallelauxcloseness2}
\twonorm{ \utplus^\perp - \utplusaux^\perp } 
\le \left[  \left(1+\frac{1}{\log n_2}\right) c_{2t} + \frac{1}{\log n_2} \right] \twonorm{\utplus^\parallel} 
= c_{2t+1} \twonorm{\utplus^\parallel},
\end{equation}
Then combining \eqref{ineq:parallelauxcloseness1} and \eqref{ineq:parallelauxcloseness2} provides \eqref{ineq:parallelauxcloseness}. This finishes the proof.

\subsection{Proof of Lemma~\ref{lemma:convergence}}
\label{app:convergence}
We observe that
\begin{align*}
\Vert \utplus^{\parallel} \Vert^2 &\overeq{(a)} \frac{\Vert \utplushalf^{\parallel} \Vert^2}{\Vert \utplushalf \Vert^2}	\\
&= \frac{\Vert \utplushalf^{\parallel} \Vert^2}{\Vert \utplushalf^{\parallel} \Vert^2 + \Vert \utplushalf^{\perp} \Vert^2}	\\
&\overset{(b)}{\ge} \frac{ \alpha \Vert \vt^{\parallel} \Vert^2}{ \beta \Vert \vt^{\perp} \Vert^2 + \alpha \Vert \vt^{\parallel} \Vert^2}\\
&\overeq{(c)} \frac{ \alpha \Vert \vt^{\parallel} \Vert^2}{ \beta + \left(\alpha - \beta \right) \Vert \vt^{\parallel} \Vert^2}.
\end{align*}	
In equality $(a)$ we used the definition of $ \utplushalf$. 
Inequality $(b)$ from the inequalities \eqref{ineq:intern99_1} and \eqref{ineq:intern99_2}.
Equality $(c)$ is due to $ \Vert \vt \Vert = 1 $.
This shows the first inequality in \eqref{ineq:convergence1}, from which the second inequality can be deduced immediately.
In a similar manner we obtain that 
\begin{align*}
\Vert \utplus^{\perp} \Vert^2&=  \frac{\Vert \utplushalf^{\perp} \Vert^2}{\Vert \utplushalf \Vert^2}  \\
&= \frac{\Vert \utplushalf^{\perp} \Vert^2}{\Vert \utplushalf^{\perp} \Vert^2 + \Vert \utplushalf^{\parallel} \Vert^2}\\ 
& \le   \frac{ \beta \Vert \vt^{\perp} \Vert^2}{\beta \Vert \vt^{\perp} \Vert^2 + \alpha \Vert \vt^{\parallel} \Vert^2}\\ 
& \le \frac{\beta}{\alpha \twonorm{\vt^\parallel}^2} \cdot  \twonorm{\vt^\perp}^2,
\end{align*}
which finishes the proof.

\end{document}